\documentclass[12pt,a4paper,oneside,openright,english]{book}

\usepackage[table,x11names]{xcolor}
\usepackage[utf8]{inputenc}
\usepackage[T1]{fontenc}
\usepackage[english]{babel}
\usepackage{hyperref}
\usepackage[toc,acronym,nonumberlist]{glossaries}
\usepackage{pdfpages}
\usepackage{natbib}
\usepackage{amssymb}
\usepackage{bm}
\usepackage{svg}
\usepackage{booktabs}
\usepackage{tikz}
\usepackage{colortbl}
\usepackage{algorithm}
\usepackage{algorithmic}
\usepackage{multirow}
\usepackage{pbox}
\usepackage{subfigure}
\usepackage[table,x11names]{xcolor}
\usepackage{colortbl}
\usepackage{float}
\usepackage{hyperref}

\usepackage{extra_style}
\numberwithin{equation}{chapter}

\usetikzlibrary{bayesnet}
\usetikzlibrary{arrows}

\hypersetup {
    colorlinks=true,
    linkcolor=blue,
    urlcolor=blue,
    citecolor=blue
}

\title{Continuous normalizing flows on manifolds}
\author{Luca Falorsi}
\date{\today}

\makeglossaries

\begin{document}


\includepdf[]{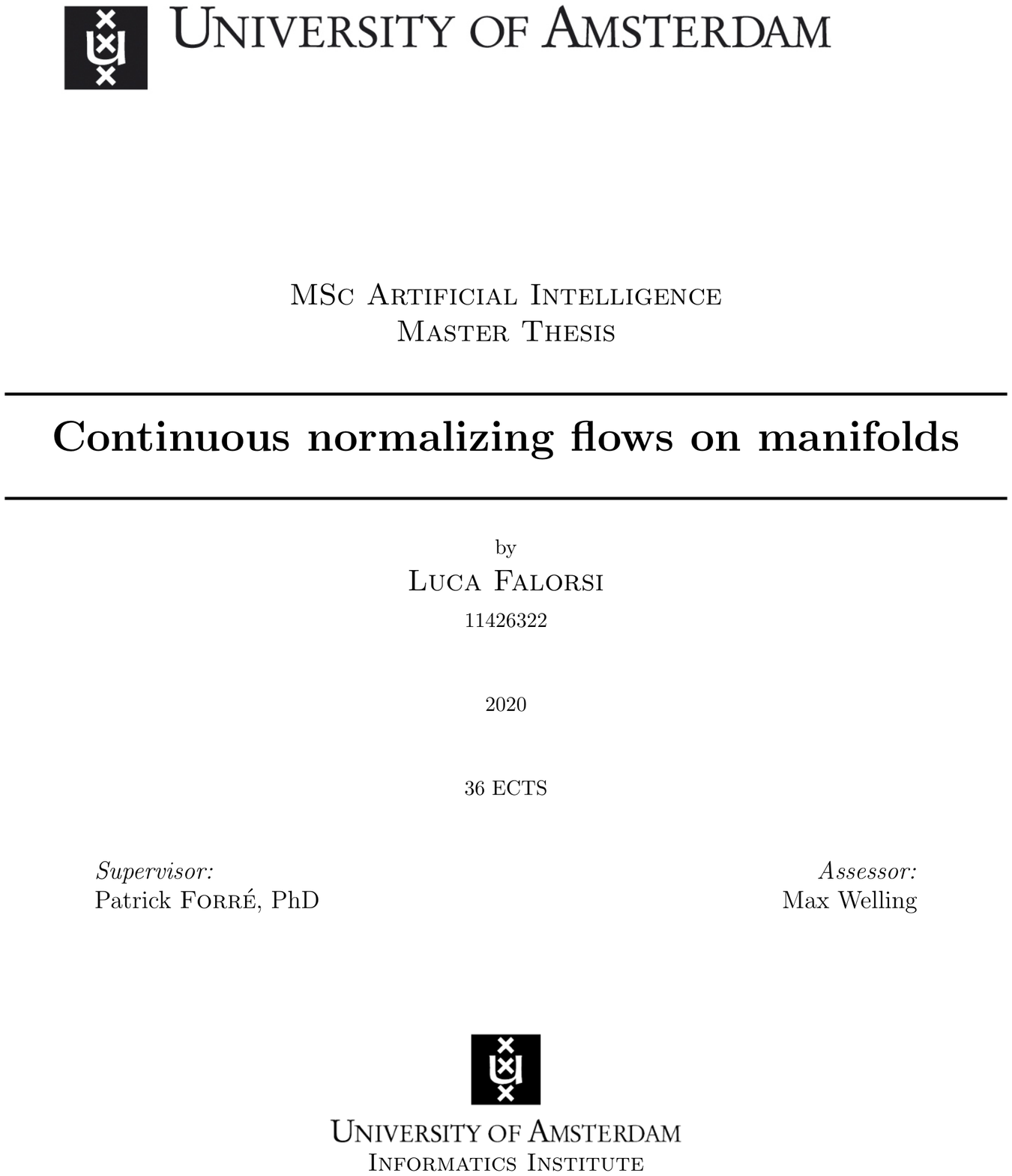}
\cleardoublepage

\pagenumbering{roman}


\chapter*{\centering \begin{normalsize}Acknowledgements\end{normalsize}}

I would first like to thank my thesis supervisor Patrick Forr\'e for helping me not to lose my way during these two years. 

\noindent I would also like to thank Nicola de Cao, Tim R. Davidson, Fabrizio Ambrogi, Pim de Haan, Jonas Köhler, Marco Federici and Luca Simonetto, for the wonderful years spent together in Amsterdam. 

\noindent Finally, I thank my parents and my brother for their love and support.

\phantomsection
\addcontentsline{toc}{chapter}{Abstract}
\chapter*{\centering \begin{normalsize}Abstract\end{normalsize}}
\begin{quotation}
\noindent 
Normalizing flows are a powerful technique for obtaining reparameterizable samples from complex multimodal distributions. Unfortunately, current approaches are only available for the most basic geometries and fall short when the underlying space has a nontrivial topology, limiting their applicability for most real-world data. Using fundamental ideas from differential geometry and geometric control theory, we describe how the recently introduced Neural ODEs and continuous normalizing flows can be extended to arbitrary smooth manifolds. We propose a general methodology for parameterizing vector fields on these spaces and demonstrate how gradient-based learning can be performed. Additionally, we provide a scalable unbiased estimator for the divergence in this generalized setting. Experiments on a diverse selection of spaces empirically showcase the defined framework's ability to obtain reparameterizable samples from complex distributions.
\end{quotation}
\cleardoublepage


\phantomsection
\addcontentsline{toc}{chapter}{Contents}
\tableofcontents
\cleardoublepage



\pagenumbering{arabic}


\chapter{Introduction}
Normalizing flows (NF) are a family of methods for defining flexible reparameterizable probability distributions on high dimensional data. They accomplish this by mapping a sample from a simple base distribution into a complex one, using a series of invertible mappings that are usually parameterized by invertible neural networks. The probability density of the final distribution is then given by the change of variable formula. As such, normalizing flows allow for fast and efficient sampling as well as density evaluation.

Due to their desirable features, flow based model have been subject to an ever growing interest in the machine learning community since their introduction \citep{rezende2015variational}, and have been successfully used in the context of generative modelling, variational inference and density estimation\footnote{See \cite{papamakarios2019normalizing, kobyzev2020normalizing} for a comprehensive review of NF.}, with a growing number of applications in different fields of science.  \citep{noe2019boltzmann, kanwar2020equivariant, wirnsberger2020targeted}. 

As many real word problems in robotics, physics, chemistry, and the earth sciences are naturally defined on spaces with a non-trivial topology, recent work has focused on building probabilistic deep learning frameworks that can work on manifolds different from the Euclidean space \citep{davidson2018hyperspherical, davidson2019increasing, falorsi2018explorations, falorsi2019reparameterizing, diffusionvae, wrappedhyperbolic}. For these type of data the possibility of defining complex reparameterizable densities on manifolds through normalizing flows is of central importance. However, as of today there exist few alternatives, mostly limited to the most basic and simple topologies. 

The main obstacle for defining normalizing flows on manifolds is that currently there is no general methodology for parameterizing maps $F:M\to N$ between two manifolds. Neural networks can only accomplish this for the Euclidean space, $\bbR^n$. In this work we propose to use vector fields on a manifold $M$ as a flexible way to express diffeomorphic maps from the manifold to itself. As vector fields define an infinitesimal displacement on the manifold for every point, they naturally give rise to diffeomorphisms without needing to impose further constraints. Furthermore, vector fields are significantly easier to parameterize using neural architectures, as they form a free module over the ring of functions on the manifold. In doing so we make it possible to define NF on manifolds. Furthermore, there exists decades old research on how to numerically integrate ODEs on manifolds.\footnote{See \citet{hairer2004GeoNumInt} for a review of the main methods.} 

Recently, normalizing flows build using differential equations have proven successful in Euclidean space \citep{neuralode, grathwohl2018scalable} taking advantage of unrestricted neural network architectures. Using ideas from differential geometry and building on the concepts first introduced in \cite{neuralode}, this work continues this line of research by defining a flexible framework for constructing normalizing flows on manifolds that is trivially extendable to any manifold of interest.

\section{Summary}
\paragraph{Chapter 2:} We start by reviewing the fundamental Machine Learning concepts on which we build upon in the rest of the thesis: reparameterizable distributions, normalizing flows and neural ODEs. We additionally review the main works that tried to extend these frameworks to manifold setting. 

\paragraph{Chapter 3 and Appendix A:}Here we take care of carefully describing the mathematical constructs needed for building a coherent theory of reparameterizable distributions on smooth manifolds. We first define volume forms and densities on manifolds, which are the fundamental objects used for measuring volumes and integrating. We then outline how, using the Riesz extension theorem, a smooth nonnegative density uniquely defines a Radon measure on the space, providing a general methodology for instantiating a base measure on a manifold. After this, we illustrate how the concept of reparameterizable distribution can be generalized to abstract measure spaces via measure pushforward. We conclude by describing how a measure pushforward given by a diffeomorphism operates between densities, giving a change of variables formula to use when building normalizing flows on manifolds. 

\paragraph{Chapter 4:}In this chapter, we first delineate how vector fields and ODEs on a manifold $M$ can be defined in the context of differential geometry, and explain how they give rise to diffeomorphisms on $M$ through their associated flow. We then apply the notions developed in Chapter 3 to the specific flow case, demonstrating how flows defined by vector fields allow defining continuous normalizing flows on manifolds. As in the Euclidean setting,  the change of density is given by integrating the divergence on the ODE solutions, where now the divergence is a generalized quantity that depends on the geometry of the space.

\paragraph{Chapter 5:}We describe a general methodology for parameterizing vector fields on smooth manifolds. Using module theory, we show that all vector fields on a manifold can be obtained by linear combinations of a finite set of generating vector fields, with coefficients given by functions on the manifold. We then outline how generating sets of vector fields can be built on embedded submanifolds of $\bbR^m$ and homogeneous spaces. We conclude by describing how the divergence can be computed when employing a generating set, and give an unbiased Monte Carlo estimator in this context.

\paragraph{Chapter 6:}Here we show how the adjoint sensitivity method can be generalized to vector fields on manifolds in the context of geometric control theory \citep{agrachev2013control}, highlighting important connections with symplectic geometry and the Hamiltonian formalism. Similarly, as in the adjoint method in the Euclidean space, to backpropagate through the flow defined by a vector field we have to solve an ODE in an augmented space. In this case, the ODE is given by a vector field on the cotangent bundle $T^*M$, called \emph{cotangent lift}, which is a {\it lift} of the original vector field on $M$. Additionally, we provide expressions for the cotangent lift on local charts, Lie groups, and embedded submanifolds. 

\paragraph{Chapter 7:}We conclude by showcasing the utility of the defined framework building continuous normalizing flows on a wide array of spaces, including the hypersphere $\bbS^n$, the Lie groups $SO(n), SU(n), U(n)$, Stiefel manifolds $\mathcal{V}_{m}(\bbR^n)$, and the positive definite symmetric matrices $\text{Sym}^+(n)$.


\chapter{Preliminaries}
\section{Reparameterization trick} 
In many machine applications, we are interested in finding the parameters of a probability density function that maximizes the expected value of an objective function.

More precisely, suppose we have an objective function $f\in C^1(\bbR^n\times \bbR^k)$ and a family $\cP$ of probability measures:
\begin{align}
    \cP:=\left\{ \rho_\lambda dx = \rho(\cdot, \lambda)\ dx \right\}_{\lambda\in W} \quad \rho(\cdot, \lambda)\in L^1(\bbR^n, dx), \forall \lambda\in W\subseteq \bbR^k
\end{align}
which are absolutely continuous with respect to the Lebesgue measure $dx$, and parameterized by an open set of parameters $W\subseteq \bbR^k$. Our objective is to find the parameters $\lambda\in W$ that maximize the expectation\footnote{Since the purpose of this section is to illustrate how the need to have the reparameterizable densities arises in machine learning, we will assume (as it is commonly done in the machine learning literature) that all the expectations are well defined and that we can always exchange derivatives and integrals.}: \begin{align}
J(\lambda) := \E{\rho(x, \lambda)}{f(x,\lambda)} = \int_{\bbR^n} f(x,\lambda) \rho(x, \lambda)\ dx
\end{align}
The {\bf reparameterization trick} allows to efficiently obtain low variance, unbiased, Monte Carlo estimates of the gradient $\partial_\lambda J(\lambda)$ and therefore makes possible to tackle the above optimization problem using gradient-based optimization techniques, such as Adam \citep{kingma2014adam}.

A {\it reparameterization trick} for the family $\cP$ consists in a base probability density function\footnote{For paractical applications, it is also important that we can easily draw samples $y\sim s(y) dy$} $s \in L^1(\bbR^m, dx)$ independent form the parameters $\lambda$ and a family of maps:
\begin{align}
    \mathcal{T} := \left\{ F(\cdot, \lambda)=F_\lambda(\cdot):\bbR^m\to \bbR^n\right\}_{\lambda\in W} \quad \text{s.t.}\quad F\in  C^1(\bbR^m\times \bbR^k, \bbR^n)
\end{align}
such that:
\begin{align}
y\sim s(y) dy\quad  \Rightarrow\quad  \Phi_\lambda(y) \sim \rho_\lambda(y)dy 
\end{align}
This allows us to rewrite the objective function using the change of variables:
\begin{align}
    J(\lambda) = \E{\rho(x, \lambda)}{f(x,\lambda)} = \E{s(\varepsilon)}{f\lp F_\lambda(\varepsilon),\lambda\rp}
\end{align}
Since now we have an expectation with respect to $s$, which is independent from the parameters $\lambda$, the gradient of the expectation becomes the expectation of the gradient:
\begin{align}
    \partial_\lambda J(\lambda) = \partial_\lambda \E{s(\varepsilon)}{f\lp F_\lambda(\varepsilon),\lambda\rp} = \E{s(\varepsilon)}{\partial_\lambda f\lp F(\varepsilon, \lambda),\lambda\rp}
\end{align}
Where $\partial_\lambda$ indicates the gradient with respect to the parameters $\lambda$. We can now easily compute:
\begin{align}
    \partial_\lambda J(\lambda) \approx \frac{1}{h}\sum_{i=1}^h \partial_\lambda f\lp F(\varepsilon^{(i)}, \lambda)\rp \quad \text{where}\quad \varepsilon^{(i)}\sim s(\varepsilon) d\varepsilon
\end{align}
Where gradients are usually computed using automatic differentiation.

The reparameterization trick was first introduced in the context of variational autoencoders \citep{kingma2013auto} for Gaussian random variables:
\begin{align}
    \mathcal{P}_\mathcal{N} = \left\{\mathcal{N}\lp x| \lambda_0, \lambda_1\rp dx = \frac{1}{\lambda_1\sqrt{2\pi}}e^{\frac{1}{2} \lp\frac{x-\lambda_0}{\lambda_1}\rp^2} dx\right\}_{\lambda\in W},\quad W=\bbR\times \bbR_{>0}
\end{align}
In this case the reparameterization trick is particularly simple and uses a standard normal $s(y) := \mathcal{N}(y| 0, 1)$ as base distribution and family of transformations:
\begin{align}
    T_\mathcal{N} = \left\{F_\lambda(y) = \lambda_0 + \lambda_1\cdot y\right\}_{\lambda\in W}
\end{align}
 The subsequent research focused on broadening the class of reparameterizable distributions\footnote{In this work the expression {\it reparameterizable distributions} indicated a class of probability measures for which a reparameterization trick is defined.} \citep{rezende2015variational, naesseth2017reparameterization, figurnov2018implicit}. The interested reader can consult \cite{mohamed2020monte} for a review of the reparameterization trick in the broader context of Monte Carlo gradient estimation.

In recent years there was an increasing interest in defining reparameterizable distributions with support on non-euclidean spaces. \cite{davidson2018hyperspherical} define a reparameterization trick for the von Mises-Fisher distribution (vMF) on the hypersphere $\bbS^n$, \cite{Cao2020ThePS} propose the Power Spherical distribution as a substitute for the vMF, improving on scalability and numerical stability. In the context of hyperbolic geometry \cite{wrappedhyperbolic, mathieu2019continuous} suggest using a wrapped normal distribution as a reparameterizable density on the hyperbolic space $\bbH^n$. \cite{falorsi2018explorations, falorsi2019reparameterizing} define a general reparameterization trick on Lie groups, allowing to transform any reparameterizable density on the euclidean space to a reparameterizable distribution on a Lie group $G$, using the exponential map from the Lie algebra $\alg{g}$ to the group $\exp: \bbR^n\cong \alg{g} \to G$.

\section{Normalizing flows}
Normalizing flows (NF) are a general methodology for defining complex reparameterizable families of probability distributions, allowing for fast and efficient sampling and density estimation.

Instead of starting from a family of probability measures and then finding a reparameterization trick for it, Normalizing Flows model an expressive class of probability distributions by specifying a base probabiity measure \footnote{Where $s\in L^1(\bbR^n, dx)$ and $\int s(x) dx = 1$} $s\, dx$ and a family $\mathcal{T}$ of transformations: \begin{align}\label{eq:vol-change-rn}
    \mathcal{T} := \left\{ \Phi(\cdot, \lambda)=\Phi_\lambda(\cdot):\bbR^n\to \bbR^n\right\}_{\lambda\in W} \quad \text{s.t.}\quad \Phi\in  C^1(\bbR^n\times \bbR^k, \bbR^n)
\end{align}
such that $\Phi_\lambda:\bbR^n \to \bbR^n$ is a diffeomorphism $\forall \lambda\in W$. This defines a family $\cP:=\left\{ \rho_\lambda dx = \rho(\cdot, \lambda)\ dx \right\}_{\lambda\in W}$ of reparameterizable distributions with probability density:
\begin{align}
    \rho(\Phi_\lambda(x),\lambda) = s\lp x\rp \left|\det D\Phi_\lambda\inv(x) \right|\quad \forall \lambda\in W
\end{align}(
Where the flexibility of the distributions in the family is determined by the transformations in $\mathcal{T}$. To be able to evaluate the probability density at the samples, we need to be capable of efficiently computing the determinant of the inverse Jacobian in Equation \eqref{eq:vol-change-rn}.

Initially introduced in \cite{tabak2010density} and \cite{tabak2013family}, Normalizing Flows were popularized by \cite{rezende2015variational} in the context of variational inference and by \citet{dinh2015nice} for density estimation. Since their introduction, flow-based models have been subject to an increasing interest in the machine learning community, and have been successfully used in the context of generative modelling, variational inference and density estimation, with an increasing number of applications in different fields of science \citep{noe2019boltzmann, kanwar2020equivariant}. See \cite{papamakarios2019normalizing, kobyzev2020normalizing} for a general review of NF.

As many real-world problems are naturally defined on spaces with non-trivial topology, recently there has been a great interest in building normalizing flows that can work on manifolds different from the Euclidean space. However as of today there exist few alternatives, mostly limited to the most basic and simple topologies. 

As already mentioned, the main obstacle for defining normalizing flows on manifolds is that there is no general methodology for parameterizing maps $F:M\to N$ between two manifolds. Neural networks, are by construction designed to operate on the Euclidean space $\bbR^n$.

\citet{gemici2016normalizing} try to sidestep this by first mapping points from the manifold $M$ to $\bbR^n$, applying a normalizing flow in this space and then mapping back to $M$. However, when the manifold $M$ has a non-trivial topology there exist no continuous and continuously invertible mapping, i.e. a \emph{homeomorphism} between $M$ and $\bbR^n$, such that this method is bound to introduce numerical instabilities in the computation and singularities in the density. Similarly, \citet{falorsi2019reparameterizing} create a flexible class of distributions on Lie groups by first applying s normalizing flow on the Lie algebra (which is isomorphic to $\bbR^n$),  and then {\it pushing it} to the group using the exponential map. While the exponential map is not discontinuous, for some particular groups the resulting density can still present singularities when the density in the Lie algebra is not properly constrained. 

\cite{boyda2020sampling} develop gauge equivariant and conjugation equivariant flows on the Lie group $SU(n)$ of special unitary complex matrices, in the context of flow-based sampling for lattice gauge theories.  

\citet{rezende2020normalizing} define normalizing flows for distributions on hyperspheres and tori. This is done by first showing how to define diffeomorphisms from the circle to itself by imposing special constraints. The method is then generalized to products of circles, and extended to the hypersphere $\bbS^n$, by mapping it to $\bbS^1\times[-1,1]^n$ and imposing additional constraints to ensure that the overall map is a well-defined diffeomorphism. \citet{bose2020latent} define normalizing flows on hyperbolic spaces by successfully taking into account the different geometry, however, the definition of diffeomorphisms in hyperbolic space is made easier due to the fact that {\it topologically} the hyperbolic space is homeomorphic to the Euclidean one.

All of the above methods rely on the target manifold containing additional special structure to formulate correct mappings. On the contrary, the framework outlined in this thesis works for any manifold.

\section{Neural ODEs and continuous normalizing flows}\label{sec:node-rn}
Ubiquitous in all fields of science, differential equations are the main modelling tool for many physical processes. Recently \citet{neuralode} showed how to effectively integrate Ordinary Differential Equations (ODEs) with Deep Learning frameworks.
In the context of deep learning, the introduction of ODE based models was initially motivated from the observation that the popular residual network (ResNet) architecture can be interpreted as an Euler discretization step of a differential equation \citep{haber2017stable}.

A {\bf vector field} $Y$ in $\bbR^n$ is a function $Y:\bbR^n \to \bbR^n$. When $Y$ fulfils suitable regularity conditions\footnote{Here we assume that $Y\in C^1(\bbR^n, \bbR^n)$ and that the vector field is {\it complete}. }, the solution of the associated ODE ($\dot{x}(t) = Y(x(t))$), at a fixed time $t_1\in \bbR$, defines a diffeomorphic map from $\bbR^n$ to itself. In practice  a numerical method needs to be employed to obtain the solution of the differential equation.

The key observation of \citet{neuralode} is that we can threat the ODE solution as a black-box. This means that in the backward pass we do not have to differentiate through the operations performed by the numerical solver. Instead \citet{neuralode} propose to use the adjoint sensitivity method \citep{pontryagin1962mathematical}. 

Closely related with the Pontryagin Maximum Principle, one of the most prominent results in control theory, the {\it adjoint sensitivity method} allows to compute Vector-Jacobian Product (VJP) of the ODE solutions with respect to its inputs. This is done by simulating the dynamics given by the initial ODE backwards, augmenting it with a linear differential equation on $\bbR^n$, called the {\it adjoint equation}:
\begin{align}
\dot{a}(t) = a(t)^\top\frac{\partial Y(x)}{\partial x}\biggr|_{x(t)}, \quad\text{where}\quad \dot{x}(t) = - Y(x(t))
\end{align}
which intuitively can be thought of as a continuous version of the usual chain rule.

Since the ODE, trough its solution, defines a diffeomorphism on $\bbR^n$, we can use it to define normalizing flows. Suppose we start from a base probability distribution $\rho_0 dx$, then the ode solution at time $t\in \bbR^n$ maps it to a distribution $\rho_t dx$. The resulting normalizing flow is denoted as Continuous Normalizing Flow (CNF). Differentiating the change of variables formula we obtain that also the volume change can be found by solving a linear ODE on $\bbR$. 
\begin{align}
    \frac{d}{dt}{\rho_t}(x(t)) = -\dive{}{Y}(x(t))\rho_t(x(t))
\end{align}
Where $\dive{}{Y}$ is the {\bf divergence} of the vector field $Y$, which corresponds to the trace of the Jacobian:
\begin{align}
    \dive{}{Y} = \text{tr}\lp \frac{\partial Y(z)}{\partial z}\rp
\end{align}
Since the Jacobian evaluation has a computational cost proportional to $n$ evaluations of $Y$, \cite{grathwohl2018scalable} propose to use Hutchinson’s trace estimator to compute an unbiased Monte Carlo estimation of the divergence: 
\begin{align}
    \dive{}{Y} = \E{\rho(\varepsilon)}{\varepsilon^\top \frac{\partial Y(z)}{\partial z} \varepsilon},\ \ \text{where:} \quad \E{\rho(\varepsilon)}{\varepsilon} = 0, \ \ \E{\rho(\varepsilon)}{\varepsilon^\top\varepsilon} = I
\end{align}
Where $\rho(\varepsilon)$ is a probability distribution on $\bbR^n$. Now the quantity $\varepsilon^\top \frac{\partial Y(z)}{\partial z}$ can be efficiently computed as a Vector Jacobian product (VJP).

Recently, competing work \citep{lou2020neural, mathieu2020riemannian} tried to generalize continuous normalizing flows and neural odes to manifold settings. While based on the same intuition, the present work significantly differs in how the basic idea is developed. 

First of all the above approaches model the manifold $M$ as being embedded in some higher dimensional Euclidean space $\bbR^m$, and then use the adjoint equation on $\bbR^m$ to perform backpropagation. Instead, we argue that the adjoint equation can be generalized to manifold setting, resulting on a differential equation on the cotangent bundle $T^*M$, which is an object intrinsically defined on the manifold, independently from the parameterization chosen. The embedded approach is then derived as a special case. Moreover, we draw insightful connections with Hamiltonian formalism and symplectic geometry. Having an intrinsically defined quantity allows us to chose more freely the practical parameterization method and the numerical integration scheme. Moreover, its manipulation allows us to derive more efficient formulations of the equation for manifolds with additional structure. We show how this can be done in the case of Lie groups. Additionally, it makes easier the definition of regularization methods.

Another substantial difference lies in how vector fields are parameterized. Both \cite{lou2020neural} and \cite{mathieu2020riemannian} start with a vector field on $\bbR^m$, which is later projected on $TM$. While the orthogonal projection on $TM$ is always well defined, and has a simple expression for the hypersphere $\bbS^n$, for more complicated spaces it can be quite complex to compute. For example, if the manifold is described as the $0$ level set of a smooth function $f:\bbR^m \to \bbR^k$, the orthogonal projection on $TM$ is given by $I - Df^+Df$, where $Df^+$ is the pseudo-inverse of the Jacobian, which might be expensive to compute and to further differentiate. Instead, we propose to parameterize vector fields using a generating set for the module of vector fields, which reduces the problem of parameterizing vector fields on $M$ to the easier problem of parameterizing functions on $M$. We argue that using a generating set is a more flexible methodology, which allows an easier parameterization on a wider class of manifolds.

For what concerns the volume change computation, both the above methods compute the divergence using local coordinates in a local chart. While by definition every smooth manifold admits a parameterization by smooth charts, many manifolds are defined implicitly, and finding charts can be quite complicated and lead to difficult expressions. Moreover, the divergence computation on a local chart is in general significantly more complex than the divergence computation on $\bbR^n$, since it involves the determinant of the $n\times n$ symmetric matrix that represents the metric tensor in local coordinates. While for the hypersphere and the hyperbolic space this has a simple expression, in general, this computation has $O(n^3)$ cost. We argue that, when vector fields are parameterized using a generating set formed by vector fields of known divergence, the divergence computation admits an expression with a complexity similar to the euclidean case and does not explicitly involve the Riemannian metric. We then show that this is always possible for all Homogeneous spaces, which always admit a generating set formed by vector fields with zero divergence.

\cite{lou2020neural} Proposes to use the normal coordinates, with Riemann or Lie $\exp$ map as charts. While normal coordinates are defined in a neighbourhood of every point, and the $\exp$ map admits simple formulas in some particular cases, in general, its computational complexity is cubic in the Lie group case, and for the Riemannian case, it involves the solution of an ODE on the tangent bundle $TM$. 

\chapter{Densities and measure pushforwards}\label{ch:vol-forms}
The overall objective of this chapter is to describe how to properly define normalizing flows on smooth manifolds. Before doing it we need to understand what is the correct mathematical formalism to deal with this problem. 

In mathematics probability distributions can properly defined using measure theory. A measure $\mu$ on a set $\mathcal{X}$ gives a way to measure the "volume" of certain subsets of $\mathcal{X}$. If we call the family of subsets that can be measured $\Sigma\subseteq \mathcal{P}(\mathcal{X})$\footnote{in general $\Sigma$ it will be smaller that the collection of all subsets of $\mathcal{X}$}, a measure $\mu$ on $\mathcal{X}$ is then a function $\mu:\Sigma \to [0,+\infty]$. In order for the measure to be well defined, both the collection $\Sigma$ and the function $\mu$ need to satisfy additional properties \footnote{See Definition 1.3 and 1.18 in \cite{Rudin:1987:RCA:26851}}. In particular $\Sigma$ needs to be a $\sigma$-algebra: a collection of subsets that contains $\emptyset$ and is closed under complement and countable union. We call the triple $(\mathcal{X},\Sigma, \mu)$ a measure space, a probability distribution on the set $\mathcal{X}$ is then a measure space for which $\mu(\mathcal{X})=1$.

The notion of measure is deeply tied with the definition of integral. As a matter of fact, given a measure space $(\mathcal{X},\Sigma, \mu)$ we can always define the Lebesgue integral $\int f \dd \mu$. The set of functions for which the integral is defined are called measurable functions. They are all the functions $f$ such that $\{x\in \mathcal{X}: f(x)<t\}\in\Sigma, \quad \forall t\in \bbR$. For probability distributions the integral of a measurable function can be also called expectation: $\E{\mu}{f}:=\int f \dd \mu$.

Since we will work on a manifold $M$, we want to have probability distributions that are compatible with the topological structure of the manifold. This means that we want to be able to measure all the open sets of $M$. In general, if $\mathcal{X}$ is a topological space, the smallest sigma algebra that contains all open sets of $\mathcal{X}$ is called the $\sigma$-algebra of Borel sets  \footnote{Definition 1.11 \cite{Rudin:1987:RCA:26851}} of $\mathcal{X}$, and is denoted as $\mathscr{B}(\mathcal{X})$. Using this definition we want $\Sigma \subseteq \mathscr{B}(M)$. When this happens we call the measure a Borel Measure. 

On a topological space\footnote{Locally compact and Hausdorff.} $\mathcal{X}$, we can define Borel measures by giving a linear functional that integrates continuous functions with compact support\footnote{$\Lambda: C_c(\mathcal{X})\to \bbR$ is a positive linear functional.}:$\Lambda f\in \bbR,$ $ \quad\forall f\in C_c(\mathcal{X})$. This important result is known as Riesz representation theorem\footnote{See Theorem 2.14 in \cite{Rudin:1987:RCA:26851} for the precise statement.}. In a nutshell, the theorem shows that there always exists unique a Radon \footnote{A radon measure is a measure defined on Borel sets, that is finite on all compact sets, outer regular on all Borel sets, and inner regular on open sets.} measure $\mu$ on Borel sets such that the corresponding Lebesgue integral coincides with the original integral defined on $C_c(\mathcal{X})$:
$\Lambda f = \int f\dd \mu,\quad\forall f\in C_c(\mathcal{X})$.

Once we have defined a "base" measure $\mu$ on a space $\mathcal{X}$, we can automatically define new measures on $\mathcal{X}$ using measurable functions: $[f\mu](E) := \int_E f \dd \mu, \quad \forall E\in \Sigma$, where $f$ is measurable positive. The Radon-Nikodym theorem tells us that\footnote{for $\sigma$-finite measures.} this set of measures coincides exactly with the measures that are absolutely continuous with respect\footnote{given two measures $\nu$ and $\mu$ on a measurable space $(\mathcal{X},\Sigma)$, $\nu$ is absolutely continuous w.r.t $\mu$ ($\nu\ll\mu$) if $\forall E \in \Sigma$ we have $\mu(E)=0 \Rightarrow \nu(E)=0$.} to $\mu$. 

Restricting to absolutely continuous measures is very convenient when designing practical algorithms. In fact, fixed a base measure on the space, the problem of specifying a measure is then reduced to giving a function that integrates to 1. We call this function probability density. Moreover, the KL divergence is finite only for a.c. measures: $\KL{f\mu}{g\mu}:=\int f(\ln f-\ln g)\dd \mu$ where $f,g$ are measurable positive. 

We can also redefine the reparameterization trick in measure theoretic terms. This corresponds to the concept of measure pushforward. Given two measurable spaces $(\mathcal{X}_1,\Sigma_1, \mu_1)$ and $(\mathcal{X}_2,\Sigma_2,\mu_2)$ and a measurable\footnote{In this case measurable means that $F^{-1}(E)\in \Sigma_1, \ \forall E \in \Sigma_2$.} map $F:\mathcal{X}\to \mathcal{\mathcal{Y}}$, the push-forward measure is defined as $[F_\#\mu_1](E):=\mu_1(F^{-1}(E)),\ \forall E\in \Sigma_2$. The reparameterizability property is given by the change of variable formula, 
: $\int f \dd[F_\#\mu_1] = \int f\circ F \dd\mu_1$. In general the measure pushforward is completely independent from $\mu_2$. Therefore when designing practical algorithms, in order to  work with densities, one needs to ensure\footnote{In this case, in order to be able to compute KL divergences one needs to compute the density of $F_\#\mu_1$ w.r.t to $\mu_2$.}  that $F_\#\mu_1\ll\mu_2$.

The considerations of the previous paragraphs give us a guideline on what is the correct formalism that we have to use to extend Bayesian machine learning models to a broader class of spaces. However, even if necessary, they provide little guidance on how to proceed in practice when working with smooth manifolds. That is, in how to choose the base measure for the space, and on how to compute the density of the pushforward measures. To do that we need to use the additional smooth manifold structure of our space. Since every manifold locally behaves like $\bbR^n$, we will briefly review how the general approach that was described above is instantiated in the Euclidean case, hoping to generalize it to smooth manifolds.

In $\bbR^n$ the standard choice as a basis density is the Lebesgue measure $dx$. 
The Lebesgue measure is obtained by applying the Riesz Representation theorem to the Riemann integral: $\int f(x)\dd x, \text{ where } f\in C_c(\bbR^n)$. Defined the Lebesgue measure, one can simply work with absolutely continuous measures just by specifying densities. Starting from a probability measure $fdx$, given a diffeomorphism $F:\bbR^n\to \bbR^n$ the resulting pushforward measure is still absolutely continuous, and its density is given by the standard change of variables formula in $\bbR^n$: $F_\#[fdx] = f\circ F^{-1}\left|\det\left(\frac{\partial F(z)}{\partial z}\right)\right|dx$, where $\frac{\partial F(z)}{\partial z}$ is the Jacobian of the inverse transformation. This gives us exactly the change of variable formula used for normalizing flows in $\bbR^n$.

Trying to generalize what is done in $\bbR^n$ is not immediately possible. The main reason is that the integral of continuous functions of compact support is in general not well defined on a smooth manifold $M$. One idea could be to define the integral evaluating the function on charts, integrating the local chart expressions, and then combining the results (using the partition of unity). For example, fix a chart $(U,\varphi)$ and a function $f\in C_c(U)$, we could try to define the integral of $f$ as $\int f(\varphi^{-1}(x)) \dd x$, however this expression is not well defined, as it depends on $\varphi$. 

In order to be able to integrate functions on smooth manifolds, we need to add additional structure. In the mathematical literature, this structure is given by volume forms. For precise definitions and explanations of these and other objects defined in this section, consult Chapters 14,15,16 of \cite{lee2013smooth} and Appendix \ref{app:background-density}.

Intuitively, volume forms give us the oriented volume of infinitesimally small patches in the manifold. Consider $M$ a smooth manifold of dimension $m$ with a volume form $\omega$. Fix a point $q\in M$, and $m$ small independent displacements $(\Delta_i)_{i\in[m]}$. If the displacements are infinitesimally small we can consider them as tangent vectors at $q$: $\Delta_i\in T_q M$. The volume form at $q$, $\omega_q$, is then a multilinear function that gives the oriented volume of the oriented parallelepiped defined by the $\Delta_i$s: $\omega_q(\Delta_1,\cdots,\Delta_m)\in\bbR$. Summing infinitesimally small volume patches gives us the integral of the volume form. Accordingly to this, volume forms are objects that can be naturally integrated on a smooth manifold. However, in order to have a volume form that is nowhere vanishing, we need the manifold to be orientable. A manifold is orientable if we can assign an orientation\footnote{An orientation on a vector space is an equivalence class of ordered bases. Two bases have the same orientation if the change of bases matrix has positive determinant.} to each tangent space $T_q M$ in a continuous way on $M$. 

Despite these considerations, we are interested in nonoriented volumes in $M$. In order to do this, we need to consider the absolute value of volume forms on $M$. The resulting object is called a density. Contrary to volume forms, on every smooth manifold, there exists a nowhere vanishing smooth density $\mu$. 
Moreover, all continuous densities can be expressed as $f\mu$, where $f\in C(M)$ is a continuous function. Since, like volume forms, densities can be integrated on manifolds, assigning $f\mu$ to its integral for every $f\in C_c(M)$ defines a positive linear functional on $C_c(M)$. For the Riesz representation theorem, this defines a measure Borel measure on $M$. \footnote{This justifies using the notation $\mu$ for the densities.} We have then accomplished our objective of defining a Radon measure on a smooth manifold.\footnote{The careful reader may have noticed that the resulting measure will depend on the initial density $\mu$, however since all the other densities can be obtained by pointwise rescaling $\mu$ by a function, all the possible other measures defined by a density will all be absolutely continuous with respect to the first one. Therefore the initial choice does not influence the class of measures that we can express.}
On semi Riemannian manifolds, there is a standard way of defining a nonvanishing density, it is called the semi Riemannian density, and it is defined as the density such that $\mu_q(E_1, \cdots, E_m)=1$ for every point $q\in M$ and every orthonormal frame $(E_1,\cdots, E_m)$ of $T_q M$.  

\section{Reparameterization trick as measure pushforward}
In the previous section, we have outlined how to define a Radon\footnote{A radon measure is a measure defined on Borel sets, that is finite on all compact sets, outer regular on all Borel sets, and inner regular on open sets.} regular measure on any smooth manifold using a smooth positive density. We have established that we will work with probability measures that are absolutely continuous with respect to the base measure. Moreover, we have seen that this class of measures is independent from the initial choice of the initial density. For precise definitions, proofs, and details on this construction we refer the reader to Appendix \ref{app:background-density}. In the rest of the Chapter we will assume that the reader is familiar with all the concepts and the notation defined there.

Established the theoretical framework in which our objects are defined, we are interested to develop a general reparameterization trick for this class of measures. In order to achieve this, we need to understand what happens under transformations. 

In general terms, we could describe the reparameterization trick as a procedure that takes a probability measure on a space $\mathcal{X}$ and defines a new measure on another space $\mathcal{\mathcal{Y}}$ by "pushing" the initial measure to $\mathcal{\mathcal{Y}}$ using a transformation $F:\mathcal{X}\to \mathcal{\mathcal{Y}}$. A correct formalization of this idea is that the transformation defines a {\bf pushforward measure}:
\begin{defn}\label{def:pushforward}\footnote{See Section 7.7 in \citet{teschl1998topics}}
Let $F : \mathcal{X} \to \mathcal{\mathcal{Y}}$ be a measurable function. Where $(\mathcal{X},\Sigma_2)$ and $(\mathcal{\mathcal{\mathcal{Y}}},\Sigma_2)$ are two measurable spaces. Given a measure $\nu$ on $\mathcal{X}$ we define the pushforward measure  $F_\#\nu$ on $\mathcal{\mathcal{Y}}$ as:
\begin{equation*}
    (F_\#\nu)(E) = \nu(F^{-1}(E)), \quad \forall E\in \Sigma_2.
\end{equation*} 
It is straightforward to check that $F_\#\nu$ is indeed a measure. Notice also
that $F_\#\nu$ is supported on the range of $F$.
\end{defn}If $\nu$ is a probability measure, then, since $F_\#\nu(\mathcal{\mathcal{Y}})=\nu(\mathcal{X})=1$, also $F_\#\nu(\mathcal{\mathcal{Y}})$ is a probability measure.

The new measure on $\mathcal{\mathcal{Y}}$ is then {\it parameterized} using a base measure on $\mathcal{X}$ and a transformation $F: \mathcal{X}\to \mathcal{\mathcal{Y}}$. 
We can evaluate expectations with respect to the pushforward measure by taking samples from the initial measure and transforming them through $F$, as stated by the following {\bf change of variables formula}. 

\begin{thm}\label{thm:change-of-variables}\footnote{Theorem 9.15 \citep{teschl1998topics}} Let $\mathcal{X}$, $\mathcal{\mathcal{Y}}$, $\nu$, $F$ as in Definition \ref{def:pushforward} and let $g : \mathcal{\mathcal{Y}} \to \bbR$ be a
Borel function. Then the Borel function $g \circ F : \mathcal{X} \to \bbR$ is
integrable with respect to $\nu$ if and only if $g$ is integrable with respect to the pushforward measure $F_\#\nu$ and in this case the integrals coincide
\begin{equation}
    \int_\mathcal{\mathcal{Y}} g d F_\#\nu = \int_\mathcal{X} g\circ F d\nu
\end{equation}
If $\nu$ is a probability measure, we can rewrite the previous formula using expectations:
\begin{equation}
    \E{F_\#\nu}{g} = \E{\nu}{g\circ F}
\end{equation}
\end{thm}
\subsection{Differential geometric prospective: pullback of inverse}

Let $M,N$ be smooth $n$ dimensional manifolds and $\mu_M, \mu_N$ smooth densities respectively on $M$ and $N$. With $\mu_N$ positive. We then have the regular Radon measures $\tilde\mu_M$ and $\tilde\mu_N$ induced by $\mu_M$ and $\mu_N$. 
If we take a diffeomorphism $F: M\to N$ this induces a measure $F_\#\tilde\mu_M$ on $N$. The following theorem tells us that the defined pushforward measure is absolutely continuous with respect to $\mu_N$ and it is exactly the measure induced from the pullback density $(F^{-1})^*\mu_M$
\begin{thm}\label{thm:push-measure-density}
Let $M,N,\mu_M,\mu_N$ as defined above. If $F:M\to N$ is a diffeomorphism, we have that $F_\#\tilde\mu_M =\widetilde{(F^{-1})^*\mu_M}$. 
This means that pushing forward the measure induced by a density on a manifold, is the same as inducing a measure from the pullback density. In particular this means $F_\#\tilde\mu_M \ll \widetilde\mu_N$ and that there exist $f\in C^\infty(N)$ such that $F_\#\tilde\mu_M =f\tilde\mu_N$
\begin{proof}
We first prove that $F_\#\tilde\mu_M$ is Radon. To achieve this, similarly as done in the proof of Proposition \ref{prop:geometric-rn}, it is sufficient to show that $F_\#\tilde\mu_M$ is finite on compact sets. Let $K\subseteq N$ compact:
$$
F_\#\tilde\mu_M(K) = \int_{F^{-1}(K)}d\tilde\mu_M = \tilde\mu_M(F^{-1}(K))<+\infty
$$
Where the last inequality follows from the fact that $F^{-1}(K)$ is compact and $\tilde\mu_M$ is radon. Then from Corollary 7.6 in \cite{folland1999real} $F_\#\tilde\mu_M$. 
Moreover $F_\#\tilde\mu_M$ and $\widetilde{(F^{-1})^*\mu_M}$ extend the same positive Linear functional in $C_c(N)$. If fact, let $g\in C_c(N)$, we have:
\begin{align}
    &\int_N g d F_\#\tilde\mu_M = \int_M g\circ F d\tilde\mu_M = \\
    =&\int_M (g\circ F)\cdot \mu_M = \int_N g\cdot(F^{-1})^*\mu_M = \int_N gd\widetilde{(F^{-1})^*\mu_M}
\end{align}
Where we have used the fact that since $F^{-1}$is continuous $g\circ F$ has compact support. Then the thesis follows from the Riesz representation theorem. 
\end{proof}
\end{thm}
Let now also $\mu_M$ be smooth positive, and consider absolutely continuous densities with respect to $\tilde\mu_M$. Let $f: M\to \bbR$ measurable, it is then easy to verify using the change of variable formula that:
\begin{equation}\label{eq:push-ac-measure}
F_\#f\tilde\mu_M = (f\circ F)F_\#\tilde\mu_M 
\end{equation}
The pushforward will then be absolutely continuous with respect to $\tilde\mu_N$ and to compute its Radon-Nikodym derivative we need to know the Radon-Nikodym derivative of $F_\#\tilde\mu_M$. 

The last theorem fundamentally tells us  that when working with measures induced by densities on smooth manifolds, and diffeomorphic transformation, {\it the pushforward measure is completely determined by how the underlying density transforms}.
Therefore, when working within this framework, it is always sufficient to work directly with densities, specifying how they are defined and how they change under transformation. The analytic objects (the measures) are completely determined by the underlying geometric objects (densities, which are sections of vector bundles). This allows us from now on to work directly with the geometric objects, using mainly geometric arguments.
\subsection{Computing the local volume change}
Keeping the same notation as before, in the previous subsection we saw that given a density $\mu\! \in\! \smoothsec{M}{\cD M }$ we are interested in computing $(F^{-1})^*\mu$ $\in \smoothsec{N}{\cD N}$, which is a section of the density bundle $\mathcal{D}N$. For some particular $F$ we can try to compute its expression in closed form, however, in many practical Machine Learning applications, a closed-form expression might be impossible or computationally intractable to compute. On the other hand, when working with Monte Carlo estimates, or when performing maximum likelihood, we are only interested in being able to compute the density value at one point. 
For example:
\begin{itemize}
    \item Given a point $q\in N$, we want to compute the density at the point: $[(F^{-1})^*\mu](q)$
    \item Given a sample $q\in M$ with density $\mu(q)$ we want to compute the pullback density at the transformed sample $[(F^{-1})^*\mu]\lp F(q)\rp$
\end{itemize}
To perform these operation we lift the diffeomorphic map $F: M\to N$ to a vector bundle isomorphism $\densitylift{F}: \mathcal{D}M \to \mathcal{D}N$:
\begin{defn}
Let $(\mathcal{D}M, \pi_M, M)$, $(\mathcal{D}N, \pi_N, N)$, density bundles with base spaces respectively $M,N$, smooth $n$ dimensional manifolds. Let $F:M\to N$ diffeomorphism, we call the lift of $F$ to the density bundle the vector bundle isomorphism $\densitylift{F}: \mathcal{D}M\to\mathcal{D}N$, defined in the following way: let $\psi \in \mathcal{D}M$
\begin{align*}
    \densitylift{F}(\psi)(v_1,\cdots,v_n):= \psi(\dd F^{-1}v_1,\cdots,\dd F^{-1}v_n)\quad \forall v_1,\cdots,v_n\in T_{F(\pi_M(\psi))}N
\end{align*}
Which is equivalent to:
\begin{align*}
 \densitylift{F}(\psi) = \lp\lp{F^{-1}}\rp^*(\psi)\rp_{\pi_N(\psi)}   
\end{align*}
\end{defn}
\begin{prop}\label{prop:lift-commute}
The following diagram commutes:
\begin{equation}\label{diag:densitylift-commute}
\adjustbox{scale=1.5}{\begin{tikzcd}
 M \arrow{r}{F} \arrow[swap]{d}{\mu} & N \arrow{d}{(F^{-1})^*\mu} \\
\cD M \arrow{r}{\densitylift{F}} & \cD N
\end{tikzcd}
}
\end{equation}
\end{prop}
\begin{proof}
It immediately follows from the definition. 
\end{proof}
The lift augments the function $F$ with the information on how the density changes.
Being able to compute this map allows to perform the operations described above:
\begin{itemize}
    \item Given $q\in N$, its density is given by: $[(F^{-1})^*\mu](q) = \densitylift{F}(\mu(F^{-1}(q)))$
    \item Given a sample $q\in M$ with density $\mu(q)$, the density of the transformed sample is: $[(F^{-1})^*\mu]\lp F(q)\rp = \densitylift{F}(\mu(q))$ 
\end{itemize}

As we described earlier, to parameterize probability measures on $M$ and $N$ we take two smooth positive densities $\mu_M\in \smoothsec{M}{\cD M }$, $\mu_N\in \smoothsec{N}{\cD N}$. Since $\{\mu_M\}$ and $\{\mu_N\}$ are global frames for the line bundles $\cD M$ and $\cD N$ respectively, they induce the vector bundle isomorphisms: $\cD M \cong M\times \bbR$, and $\cD N\cong N\times \bbR$. With this identification, probability measures are given by $L^1$ functions on $M$ and $N$ respectively and the vector bundle isomorphism $\densitylift{F}$ is completely determined by the map \begin{align}
J_{F^{-1}}:M&\to \bbR\\
 q&\mapsto J_{F^{-1}}(q) = \dfrac{\densitylift{F}(\mu_M(q))}{\mu_N(F(q))}\quad \forall q\in M
\end{align}
Where $J_{F^{-1}}(q)\in \bbR$ is a (positive) real number \footnote{The division is well defined since $\mu_N$ is a positive density. Since $\mu_M$ is also positive and $F$ is a diffeomorphism, the resulting fraction is positive.} that depends on the transformation $F$ and the point $q$\footnote{$J_{F^{-1}}(q)$ depends on the two base positive densities $\mu_M$, $\mu_N$ as well.  But, since we assume them fixed beforehand, we drop the dependency for convenience.}.
This means that globally on the manifold we have a positive function $J_{F^{-1}}: M\to \bbR$ that gives the {\it local volume change}, 
The notation for $J_{F\inv}$ is given by the fact that for $\bbR^n$ it corresponds to the absolute value of the determinant of the Jacobian of the inverse transformation. 

In fact, since $\densitylift{F}$ is a vector bundle isomorphism, fixed a point $q\in M$, $\densitylift{F}$ restricted to the fiber is a linear map $\densitylift{F}|_{\cD M_q}:{\cD M_q}\to \cD N_{F(q)}$. Since $\{\mu_M(q)\}$, and $\{\mu_N(F(q))\}$ form a basis respectively of ${\cD M_q}$ and $\cD N_{F(q)}$ the linear map is completely described by its effect on the basis:
\begin{align*}
    F|_{\cD M_q}:{\cD M_q}&\to \cD N_{F(q)}\\
    a_0 \mu_M(q)&\mapsto \densitylift{F}(a_0\mu_M(q)) = a_0\densitylift{F}(\mu_M(q)) = a_0 J_{F^{-1}}(q)\cdot \mu_N(F(q))
\end{align*}
The previous 
diagram \ref{diag:densitylift-commute}, can then be rewritten as to:
\begin{equation}
\adjustbox{scale=1.5}{
\begin{tikzcd}[row sep=huge, column sep = large]
 M \arrow{r}{F} \arrow[swap]{d}{(id,\, f)} & N \arrow{d}{\lp id,\, f\, \cdot\, \frac{\lp F^{-1}\rp^*\mu_M}{\mu_N}\rp} \\
M\times {\bbR} \arrow{r}{(F,\ \cdot\,  J_{F\inv})} & N\times \bbR
\end{tikzcd}
}
\end{equation}
\subsection{Computing the volume change: practical advice}
How to compute the function $J_{F\inv}$ in practice depends on the precise data structure used to parameterize the manifold and the function $F$. In this section, we provide a theoretical discussion and useful formulas that could be used to tackle this problem. We are going to assume that we are able to compute either the differential or the pullback using an automatic differentiation package.

By definition $J_{F\inv}$ is the fraction of the two densities $\lp F^{-1}\rp^*\mu_M$ and $\mu_N$, to obtain its value at a point $q\in M$ we can evaluate the two densities on the same basis of $T_{F(q)}N$ and take the fraction of the results: 
\begin{align}
    J_{F\inv}(q) = \frac{\lp F^{-1}\rp^*\mu_M(v_1,\cdots,v_n)}{\mu_N(v_1,\cdots,v_n)} = \frac{\mu_M(dF\inv(v_1),\cdots,dF\inv(v_n))}{\mu_N(v_1,\cdots,v_n)}
    \\ \forall (v_1,\cdots,v_n) \in GL(T_{F(q)}N)
\end{align}
If $(M,g_M),(N,g_N)$ are semi riemannian manifolds and $v_1,\cdots,v_n$ is an orthonormal basis of $T_{F(q)}M$ the expression further simplifies to:
\begin{align}
    J_{F\inv}(q) = \mu_M(dF\inv(v_1),\cdots,dF\inv(v_n)) = |\det g_M(dF\inv(v_i), u_j)|\\ \forall (v_1,\cdots,v_n)\in O(T_{F(q)}N, g_N),\
    \forall (u_1,\cdots,u_n)\in O(T_qM, g_M)
\end{align}
The above expressions all depend on the differential of the inverse map, however, the analytical expression of the inverse might not be available. We can circumvent this by observing that, since $F$ is a diffeomorphism, we can define a basis of $T_{F(q)}N$ by pushing forward a basis of $T_qM$:
\begin{align}
     J_{F\inv}(q) = \frac{\lp F^{-1}\rp^*\mu_M(dF(u_1),\cdots,dF(u_n))}{\mu_N(dF(u_1),\cdots,dF(u_n))} = \frac{\mu_M(u_1,\cdots,u_n)}{\mu_N(dF(u_1),\cdots,dF(u_n))}
    \\ \forall (u_1,\cdots,u_n) \in GL(T_pM)
\end{align}
Also in this case we have a simplified expression for the Riemannian case:

    \begin{align}
    J_{F\inv}(q) = \frac{1}{\mu_N(dF(u_1),\cdots,dF(u_n))} = |\det g_N(dF(u_i), v_j)|\inv\\ \forall (u_1,\cdots,u_n)\in O(T_qM, g_M),\
    \forall (v_1,\cdots,v_n)\in O(T_{F(q)}N, g_N)
\end{align}
\subsection{Summary}
To define a reparameterization trick on a smooth manifold $N$ we can proceed as the following
\begin{itemize}
    \item Choose a starting space $M$, where $M$ is a smooth manifold diffeomorphic to $N$.
    \item Define smooth positive densities $\mu_M$, $\mu_N$, on $M$, $N$ respectively. 
    \item Choose a starting distribution by specifying a measurable function $f:M\to\bbR$. The probability measure is then $f\cdot \mu_M$. We need to be able to sample from it. 
    \item Take a diffeomorphism $F: M\to N$, $F$ will generally depend on some parameters. 
    \item Compute the volume change term $J_{F\inv}: M\to\bbR$
    (Alternatively or in addition one can compute the volume change term of the inverse: $J_F:N\to\bbR$)
    \item The Radon-Nikodym derivative of the transformed distribution is given by
    $f(F\inv(q))\cdot J_{F\inv}(F\inv(q)), \forall q\in N$. 
    \item Given a sample $q\in M$ the value of the Radon-Nikodym derivative of the transformed distribution computed at the mapped point $F(q)$ is given by 
    $f(q)\cdot J_{F\inv}(q), \forall q\in M$.

\end{itemize}

\chapter[CNF on Manifolds]{Continuous normalizing flows on Manifolds}

\section{Integral curves and flows}
\subsection{Vector fields}

A vector field\footnote{We refer to \cite[ch.~8-9]{lee2013smooth}} on a smooth manifold $M$, is defined as a section of the tangent bundle $TM$. 
\begin{defn}
A \textbf{vector field} on a smooth manifold $M$ is an element of $\csec{M}{TM}$, i.e. a section of the tangent bundle $TM$. More precisely, denoted by $\pi$ the natural projection $\pi:TM\to M$ a vector field is a continuous map:
\begin{align}
    X: &M\to TM &\text{such that } \\ 
    &q\mapsto X_q \quad \quad  &\pi\circ X = \text{id}_M
\end{align}
A \textbf{smooth vector field} is a vector field for which the map $X$ is smooth 
\end{defn}
\begin{defn}
A \textbf{time dependent vector field} is a continuous map\footnote{More generally, a time dependent vector field should be defined as a map $J\times M\to TM$, where $J\subseteq \bbR$ is a interval. However since we are only interested in complete vector fields, we restrict to the particular case where $J=\bbR$} $X:\bbR\times M \to TM$ where for every time $t\in \bbR,\ X_t:=X(t, \cdot): M\to TM$ is s vector field on M.
\end{defn}
Every vector field $X$ automatically determines a time dependent vector field $X'$ simply by setting $X'_t = X\ \forall t\in\bbR$.
We will often refer to a time dependent vector field as a {\bf nonautonomous vector field} and to a vector field as an {\bf autonomous vector field}.
Vector fields allow to generalize the concept of Ordinary Differential Equation to manifolds. We begin by defining {\bf integral curves}:
\begin{defn}
Let $X:\bbR\times M \to TM$ be a vector field and $J\subseteq \bbR$ an interval. A differentiable curve $\gamma:J\to M$ is an \textbf{integral curve} of $X$ if:
\begin{equation}\label{eq:integral-curve}
    \dot{\gamma}(t)=X(t,{\gamma(t)})\quad \forall t\in J
\end{equation}
If fix a "starting time" $t_0\in J$, the point $\gamma(t_0)\in M$ is called starting point of $\gamma$. We call \textbf{maximal integral curve} an integral curve that cannot be extended to a larger interval $J \subset I$. 
\end{defn}

Locally, in a smooth coordinate chart $(U;x_i)$, we have that $X|_{\bbR\times U} = \sum_{i=1}^n f_i \partial_{x_i}, $ with $f_i\in C(\bbR\times U), \forall i \in [n]$. Equation  \ref{eq:integral-curve} can then be written as

\begin{align*}
    \dot{x}_i(\gamma(t))\partial_{x_i}\Big|_{\gamma(t)}= f_i(t,\gamma(t))\partial_{x_i}\Big|_{\gamma(t)} \quad \forall i\in[n]
\end{align*}
Which is equivalent to locally solving a system of ODE on the Euclidean space. 
Under suitable regularity conditions \footnote{$f_i\in C^1\ \forall i \in \range{n}$ is sufficient.} we can then (locally) apply the existence and uniqueness theorem, ensuring that, for a fixed starting point $q_0\in M$, for a fixed starting time $t_0\in \bbR$, and for a small enough interval there exist unique an integral curve of $X$. 
\begin{thm}\footnote{Theorem 2.15 in \cite{AgrBarBos17}.}
Let $X$ be a smooth time dependent vector field, then for any point $(t_0,p_0)\in \bbR\times M$ there exist a unique maximal integral curve $\gamma:J\to M$ with starting point $q_0$, and starting time $t_0$ denoted by $\gamma(t;t_0,q_0)$. We call $\gamma$ a solution of of the Cauchy problem:
\begin{align}\label{eq:cauchy-problem}
    \left\{
    \begin{array}{l}
     \dot{\gamma}(t) = X(t, \gamma(t)) \\
    \gamma(t_0) = q_0
    \end{array}
    \right.
\end{align}
Moreover the map $(t, q) \to \gamma(t; t_0, q)$ is smooth on a neighborhood of $(t_0, q_0)$. 
\end{thm}
The solution is unique in the sense that, if two solutions $\gamma_1:J_1 \to M$ and $\gamma_2: J_2 \to M $ exist on two different intervals containing $t_0$, then the two solutions agree on the intersection $J_1\cap J_2$. This theorem is equivalent to stating that there exists unique a maximal integral curve with starting point $q_0$ and starting time $t_0$. We we call it {\bf maximal solution} of the Cauchy problem. 
Because of their time invariance, for autonomous vector fields the solutions to the Cauchy problem have the special property that $\gamma(t;t_0,q) = \gamma(t-t_0;0,q)$.
\subsubsection{Reducing to the autonomous case}
 We have already seen that every autonomous vector field can be considered as time dependent vector field. In this section we will see that the converse, in some sense, is also true. This means that we can reduce the study of nonautonomous vector fields to autonomous ones. Let $X: \bbR\times M\to TM$ be a time dependent vector field. We can define an autonomous vector field $\breve X \in \csec{\bbR\times M}{ T\bbR\times TM}$ on the extended space $\bbR\times M$: $$\breve X(s,q):= \lp \frac{\partial}{\partial t}\Bigr|_{t=s}, X(s, q)\rp$$
 It's then straightforward to verify that $\gamma(t;t_0, q_0)$ is a solution of the Cauchy problem for $X$ with initial point $q_0$ an initial time $t_0$ if and only if $\lp t,\gamma(t;t_0, q_0)\rp$ is the solution of the Cauchy problem for $\breve X$ with initial point $(t_0, q_0)$ and initial time $0$. Therefore considering a time dependent vector field on a manifold $M$ is equivalent to consider a particular type of vector fields on $\bbR\times M$.
\subsubsection{Complete vector field}
A time-dependent vector field $X$ is {\bf complete} if for every $(t_0, q_0)\in \bbR\times M$, the maximal solution $\gamma(t;t_0,q_0)$ of the Cauchy problem \ref{eq:cauchy-problem} is defined on all $\bbR$.

In general not all vector fields on a smooth manifolds are complete. A sufficient condition is given by the following theorem: 
\begin{thm}\footnote{Theorem 9.16 in \cite{lee2013smooth}.}
Every compactly supported smooth vector field on a smooth manifold
is complete.
\end{thm}
This means that in a compact manifold all vector fields are complete.  
In the Euclidean space is well known that every sublinear vector field is complete:
\begin{thm}\footnote{Remark 2.6 in \cite{AgrBarBos17}.}
Let $X\in \smoothsec{\bbR^n}{T\bbR^n}$ a smooth vector field such that:
$$\exists\ C_1,C_2>0\ \ s.t. \ \ |X(x)|\le C_1x+C_2\quad \forall x\in \bbR^n$$
Then $X$ is complete
\end{thm}
For complete Riemannian manifolds we can obtain a similar result using Gronwall estimates (on Riemannian manifolds):
\begin{thm}
Let $(M, g)$ a complete Riemannian manifold and $X\in \smoothsec{M}{TM}$ a smooth vector field on $M$.
If \begin{align}
    C:=\sup_{p\in M}\{\|\nabla X(p)\|_g\}<+\infty
\end{align}
Then $X$ is complete, where $\|\nabla X(p)\|_g = \sup_{Y_p\in T_pM}\left\{ \frac{\|\nabla_{Y_p}X(p)\|_g}{\|Y_p\|_g}\right\}$ 
\end{thm}
\begin{proof}
Suppose by contradiction that there esists a maximal integral curve $\gamma:J\to M$ such that the interval $J$ has a finite least upper bound $b:=\sup_{x\in J}\{x\}<+\infty$. Then by the global escape lemma (Lemma 9.19 in \cite{lee2013smooth}) $\forall t_0\in J$ $\gamma([t_0,b))$ is not contained in any compact set of $M$. Since by Hopf Rinov theorem all closed balls  $\overline{B(q, R)},\ \forall q\in M,\forall R>0$ are compact sets, this implies that:
\begin{align}\label{eq:escape-len}
    \sup_{t\in [t_0, b)}\left\{\ell\lp\gamma\lp[t_0,t]\rp\rp\right\} = +\infty
\end{align}
Where $\ell(\cdot)$ denotes the Riemannian length of the curve. 
Now let $t_0\in J$ be such that $t_0 - \Delta t \in J$ where $\Delta t:= b-t_0$. We then have that
\begin{align}
    &\ell\lp\gamma\lp[t_0, t_0 + s]\rp\rp \le  \ell\lp\gamma\lp[t_0-\Delta t + s, t_0+s]\rp\rp \le\\
    &\le \ell\lp\gamma\lp[t_0-\Delta t, t_0]\rp\rp e^{Cs} \le \ell\lp\gamma\lp[t_0-\Delta t, t_0]\rp\rp e^{C\Delta t} \quad 
    \forall s\in [0,\Delta t)
\end{align}
Where the first inequality holds because $s\in [0, \Delta t) \Rightarrow [t_0, t_0+s] \subset [t_0 -\Delta t + s, t_0+s]$; the second follows from Proposition 1 in \cite{kunzinger2006global}. 
The above equation is contradiction with equation \eqref{eq:escape-len}, therefore $J$ has no finite least upper bound. Similarly we can prove that there is not maximal solution with no finite greatest lower bound, this can be done just by considering the field $-X$. We then have proven the field $X$ is complete. 
\end{proof}
\subsection{Flows}
Instead of considering the single trajectory of a point, we are interested in considering the solution of the differential equation globally as a function $M\to M$ defined at every point $q$ by the solution of the Cauchy problem (supposing that the starting time and the final time are fixed). Restricting for now to the autonomous case, this means considering the family of maps $\phi^t: M\to M, \quad \phi^t(q) = \gamma(t;0, q)$ where $t\in\bbR$.
In this case we say that the vector field generates a {\bf flow}.

However in general $\phi^t(q)$ exists only for a subset $\cD\subseteq \bbR\times M$.
When the vector field $X$ is complete, it generates a flow defined in all $\cD=\bbR\times M$, in this case we call it {\bf global flow}. 
A global flow is abstractly defined as following:
\begin{defn}
A \textbf{global flow} is a continuous left $\bbR$-action on a manifold $M$; that is, a continuous map $\phi:M\times \bbR \to M$ satisfying the following properties for all $s,t\in \bbR$ and $p\in M$: 
\begin{equation}
    \phi(t,\phi(s,p))=\phi(t+s,p),\quad \phi(0,p)=p
\end{equation}
Given a global flow $\phi$ on $M$, we define the following functions\footnote{In the rest of the thesis, we will often consider a flow as a family formed by the maps $\{\phi^t\}_{t\in\bbR}$.}:
\begin{align}
    \phi^t:M&\to M &\quad \forall t\in \bbR \\
    p&\mapsto \phi^t(p):=\phi(t,p)\\
    \phi^{(p)}:\bbR&\to M &\quad \forall p\in M \\
    p&\mapsto \phi^{(p)}:=\phi(t,p)
\end{align}
\end{defn}
From the group laws it immediately follows that each $\phi^t$ is invertible:
\begin{equation}
    \left(\phi^t\right)^{-1} = \phi^{-t} \quad \forall t\in\bbR
\end{equation}
This means that $\{\phi^t\}_{t\in\bbR}$ is a family of homeomorphisms. 

Given a smooth global flow $\phi$, we can define a vector field $X\in \smoothsec{M}{TM}$ in the following way:
\begin{equation}\label{eq:inf-gen}
    X_p = \frac{d}{dt}\bigg|_{t=0}\phi^{(p)}(t) = \frac{d}{dt}\bigg|_{t=0}\phi^t(p), \quad \forall p\in M
\end{equation}
This means that $X_p$ is well defined as the tangent vector of the curve $\phi^{(p)}$ at $0$. We call this vector field the {\bf infinitesimal generator of $\phi$}, the name is clarified by the following Proposition:
\begin{prop}\footnote{Proposition 9.7 in \cite{lee2013smooth}}\label{prop:flow2field}
Let $\phi:\bbR\times M \to M$ be a smooth global flow on a smooth manifold
$M$. The infinitesimal generator $X$ of $\phi$ is a smooth complete vector field on $M$; and each curve $\phi^{(p)}$ is an integral curve of $X$.
\end{prop}
We therefore have that every smooth global flow is generated by a complete vector field. The converse is also true, this result is known as the {\it fundamental theorem theorem of flows}, we give here the simplified version for complete vector fields: 
\begin{thm}
Let $X$ be a smooth vector
field on a smooth manifold $M$. There is a unique smooth maximal flow $\phi:\bbR \to M$ whose infinitesimal generator is $X$ . This flow has the following property:
For each $p\in M$, the curve $\phi^{(p)}:\bbR\to M$ is the unique maximal integral
curve of $X$ starting at $p$.
\end{thm}
This results tells us that given a complete vector field on a smooth manifold, this automatically determines a family of diffemeorphisms on the manifold, that can be obtained by finding the integral curves. 
Since in the rest of the work we will only consider complete vector fields and global flows we will refer to them simply, as flows. 
\subsubsection{Time dependent flows}

As we saw, the solutions of a time dependent vector field depend on both the initial and final time $t_0$ and $t_1$, and not merely on the difference $t_1 - t_0$ as in the autonomous case. For this reason, a nonautonomous vector field cannot define a flow. Instead, we have to consider {\bf time dependent flows}. To see how these objects arise consider $X$ a complete time dependent vector field on a manifold $M$ and $\breve X$ its associated vector field on $\bbR\times M$. Then $\forall t_0,t_1\in \bbR$ we can then define $\phi_X^{t_1,t_0}:M\to M$ as the (unique) map that that satisfies:
\begin{align}
    \phi^{t_1-t_0}_{\breve X}\lp\lp t_0,q\rp\rp = \lp t_1, \phi_X^{t_1,t_0}(q)\rp
\end{align}
It's then straightforward to verify that this corresponds to mapping each point $q \in M$ to the solution at time $t_1$ of the Cauchy problem with initial point $q$ and initial time $t_0$:
\begin{align}
     \phi^{t_1,t_0}_X:M&\to M\\
     q&\mapsto \gamma(t_1;t_0,q)
\end{align}
A time dependent flow is then defined as the following:
\begin{defn}
A \textbf{global time dependent flow} on a smooth manifold is a two parameter family of continuous  maps $\{\phi^{t_1,t_0}\}_{t_0,t_1\in \bbR}$
that satisfy the following conditions:
\begin{enumerate}
    \item $\phi^{t,t} = Id \quad \forall t\in \bbR$
    \item $\phi^{t_2,t_1}\circ \phi^{t_1,t_0} = \phi^{t_2,t_0} \quad \forall t_0, t_1, t_2 \in \bbR$
\end{enumerate}
\end{defn} From the definition it immediately follows that 
\begin{align}
    \lp \phi^{t_1,t_0}\rp^{-1} = \phi^{t_0,t_1}\quad \forall t_0,t_1 \in \bbR
\end{align}
And therefore that $\phi^{t_1,t_0}$ is a homeomorphism $\forall t_0, t_1 \in \bbR$.

Consider now the family $\{\phi^{t_1,t_0}_X\}_{t_0,t_1\in \bbR}$ defined above, using the properties of flows is then straightforward to verify that it verifies the axioms for a time dependent flow. Moreover, from our construction and the {\it fundamental theorem of flows} it immediately follows that every complete time dependent vector field generates a time dependent flow. 
Conversely, as in the autonomous case, given a smooth time dependent flow  $\{\phi^{t_1,t_0}\}_{t_0,t_1}$ we can define a time dependent vector field, called the {infinitesimal generator}, in the following way
\begin{align}
    X_t(q) := \frac{d}{ds}\Bigr|_{s=t} \phi^{s,t}(q) \quad \forall q\in M \forall t\in \bbR
\end{align}
\section{Measure pushforward induced by the flow}
\subsection{Divergence of a vector field}
\begin{defn}
Let $M$ a smooth manifold and $X\in \smoothsec{M}{TM}$ a smooth vector field. Given a positive $\mu\in \smoothsec{M}{\cD M}$ we define the \textbf{divergence} of $X$ with respect to $\mu$ as the function $\dive{\mu}{X}\in C^\infty(M)$ such that
\begin{align}
    \mathcal{L}_X\lp\mu\rp = \dive{\mu}{X}\cdot \mu
\end{align}
Similarly, given a non vanishing smooth volume form $\omega \in \smoothsec{M}{\volform{n}{M}}$ we can define the \textbf{divergence} of $X$ with respect to $\omega$ as the function $\dive{\omega}{X}\in C^\infty(M)$ such that:
\begin{align}
    \mathcal{L}_X\lp\omega\rp = \dive{\omega}{X}\cdot \omega
\end{align}
\end{defn}
The following Lemma shows that the divergence with respect to a form and the divergence with respect to a density are essentially the same object
\begin{lem}\label{lem:div-form-density}
Let $M$ a smooth manifold and $X\in \smoothsec{M}{TM}$ a smooth vector field. Given a non vanishing smooth volume form $\omega \in \smoothsec{M}{\volform{n}{M}}$ we have:
\begin{align}
\dive{\omega}{X} = \dive{|\omega|}{X}
\end{align}
\end{lem}
\begin{proof}
Fix $q\in M$ we have that:
\begin{align}
\mathcal{L}_X\lp\omega\rp_q = \lim_{t\to 0} \frac{\lp\lp\phi_X^{-t}\rp^* \omega\rp_q- \omega_q}{t}  
\end{align}
Since $\omega$ is nowhere vanishing, for small enough $t\in \bbR$ we can define $a_t\in \bbR$ such that:
\begin{align}
    \lp\lp\phi_X^{-t}\rp^* \omega\rp_q = a_t\omega_q
\end{align}
Since the field and the volume form are both smooth, the number $a_t$ depends continuously on t. Now plugging it back:
\begin{align}\label{eq:div-form}
\mathcal{L}_X\lp\omega\rp_q = \lim_{t\to 0} \frac{a_t\omega_q - \omega_q}{t} = \lim_{t\to 0}\frac{a_t - 1}{t}\omega_q  
\end{align}
Now we can see that:
\begin{align}
\mathcal{L}_X\lp|\omega|\rp_q & = 
\lim_{t\to 0} \frac{\lp\lp\phi_X^{-t}\rp^* |\omega|\rp_q- |\omega|_q}{t} = \lim_{t\to 0} \frac{\big|\lp\lp\phi_X^{-t}\rp^* \omega \rp_q \big|- |\omega|_q}{t} = \\\label{eq:div-density}
& = \lim_{t\to 0} \frac{|a_t\omega_q| - |\omega|_q}{t} = 
\lim_{t\to 0} \frac{a_t|\omega|_q - |\omega|_q}{t} = 
\lim_{t\to 0}\frac{a_t - 1}{t}|\omega|_q  
\end{align}
Where we have used that since $\lim_{t\to 0} a_t = 1$ we can assume for small enough $t$ that $a_t$ is positive and  therefore $|a_t\omega_q| = a_t|\omega_q| = a_t|\omega|_q$. 
Therefore comparing Equation \eqref{eq:div-density} and Equation \eqref{eq:div-form} we see that we have proven:
\begin{align}
\dive{\omega}{X}\lp q\rp = \dive{|\omega|}{X}\lp q\rp\quad \forall q\in M
\end{align}
\end{proof}
From the properties of the Lie derivative it follows that:
\begin{prop}\label{prop:div-prop}
Let $M$ be a smooth manifold, $X,Y\in \smoothsec{M}{TM}$ smooth vector fields and $f\in C^\infty(M)$ a smooth function, Given a non vanishing density $\mu\in \smoothsec{M}{\mathcal{D}M}$:
\begin{enumerate}
    \item $\dive{\mu}{X+Y} = \dive{\mu}{X}+\dive{\mu}{Y}$
    \item $\dive{\mu}{fX} = X(f) + f\dive{\mu}{X} = df(X) + f\dive{\mu}{X}$
\end{enumerate}
\end{prop}
\subsubsection{Divergence on local coordinates}
Let $(U; x_i),\ U\subseteq M$ be a local chart, $X\in \smoothsec{M}{TM}$ be a smooth vector field and $\omega \in \smoothsec{M}{\Lambda^nT^*M}$ a smooth nonvanishing volume form
\footnote{For a local expression of the divergence with respect to a smooth nonvanishing density $\mu\in\smoothsec{M}{\cD M}$ we can consider, eventually restricting $U$, $\mu|_U = |\omega||_U$ and use Lemma \ref{lem:div-form-density}.}  with local expressions:
\begin{align}
\omega\big|_U &= h\, dx_1\wedge\cdots\wedge dx_n\\
X\big|_U &= \sum_{i=1}^n f_i\partial_{x_i}
\end{align}
Where $f_i\in C^\infty(U),\ \forall i\in \range{n}$ and $h \in C^\infty(U),\ h(x)\neq 0\ \forall x\in U$. 
\begin{align}
    \mathcal{L}_X(\omega)\big|_U &= \mathcal{L}_X(h\, dx_1\wedge\cdots\wedge dx_n) =\\
    &=\mathcal{L}_X(h)\,dx_1\wedge\cdots\wedge dx_n + h\cdot\mathcal{L}_X\left(dx_1\wedge\cdot\wedge dx_n\right) = \\
    &= \lp\sum_{i=1}^nf_i\dfrac{\partial h}{\partial x_i} \rp\cdot\, dx_1\wedge\cdots\wedge dx_n +\\&+ h\cdot \sum_{i=i}^n  dx_1\wedge\cdots\mathcal{L}_X\left(dx_i\right)\wedge\cdots\wedge dx_n = \\
    & = \lp \dfrac{1}{h}\sum_{i=1}^n \frac{\partial\lp hf_i\rp}{\partial x_i}\rp\cdot h\,dx_1\wedge\cdots\wedge dx_n 
    \end{align}
    Where in the last passage we used that
    \begin{align}
        \mathcal{L}_X(dx_i) = d\lp dx_i\lp X\rp \rp = d(f_i) = \sum_{j=1}^n \dfrac{\partial f_i}{\partial x_j} dx_j
    \end{align}
We therefore have proven that the divergence has the following expression in local coordinates:
\begin{align}
\dive{\omega}{X}\big|_U = \dfrac{1}{h}\sum_{i=1}^h \frac{\partial\lp hf_i\rp}{\partial x_i}
\end{align}
For the special case of a semi-Riemannian manifold $(M,g)$, using Equation \eqref{eq:semi-riem-vol}  we have $h = \sqrt{|\det g_{ij}|}$, this gives us the expression for the {\bf semi-Riemannian divergence in local coordinates}:
\begin{align}
\dive{\omega_g}{X}\big|_U = \dfrac{1}{\sqrt{|\det g_{ij}|}}\sum_{i=1}^h \dfrac{\partial (f_i\sqrt{|\det g_{ij}|})}{\partial x_i}
\end{align}

\subsubsection{Divergence on semi-Riemannian manifolds}
Consider $(M, g)$ a semi-Riemannian manifold (with connection $\nabla$), and a smooth vector field $X\in \smoothsec{M}{TM}$. Given a point $q\in M$, and  $E_1,\cdots, E_n$ a orthonormal frame for $T_qM$, the divergence can be expressed using the covariant derivative\footnote{Definition 47 \citep{o1983semi}}:
\begin{align}\label{eq:div-semiriem}
\dive{\mu_{g}}{X} = \sum_{i=1} g\lp E_i, E_i\rp \cdot g\lp\nabla_{E_i}X, E_i\rp
\end{align}

\subsection{Continuity equation on Manifolds}
Let $X$ be a time dependent complete vector field on a smooth manifold $M$, and  $\{\phi_X^{t,s}\}$ its time dependent flow. Let $\mu\in \smoothsec{M}{\cD M}$ be a smooth positive density, we are interested on the change of variables  induced by the diffeomorphisms $\phi_X^{t,s}$.

Consider an initial density $\rho_0\mu$ where $\rho_0\in C^\infty(M)$, 
we define $\rho_t$ as the smooth function such that 
\begin{equation}
    \rho_t\mu = \lp\phi^{0,t}_X\rp^*\lp\rho_0\mu\rp
\end{equation}
\begin{thm}\label{thm:continuity-equation}
Let $M$, $\mu$, $X$, $\rho_t$ as defined above. Then the function $\rho\in C^\infty(M\times\bbR),\ \rho(\cdot, t) = \rho_t$ satisfies the following linear PDE:
\begin{align}
    X_t(\rho_t) + \rho_t\dive{\mu}{X_t} = - \partial_t \rho
\end{align}
\begin{proof}
We begin by fixing $q\in M$ and $t\in\bbR$, we then know that there exists a open neighborhood $U\ni q$ and a volume form $\omega$ such that:
\begin{align}
    \mu\big|_U = |\omega|
\end{align}
Since the flow, as a map $\bbR\times M\times \bbR \to \bbR\times M\times \bbR$ is continuous, there exists a open interval $I\ni t$ and a open neighborhood $U\supseteq V\ni q$ such that:
\begin{align}
\phi_X^{s,t}(V)\subseteq U\quad \forall s\in I
\end{align}
In this set we can work directly withe the volume form $\omega$:
\begin{align}
    \lp\rho_t\omega\rp_q = \lp\lp\phi_X^{t,s}\rp^*\lp\lp\phi_X^{s,t}\rp^*\rho_t\omega \rp\rp_q =
   \lp\lp\phi_X^{t,s}\rp^*\lp\rho_s\omega\rp\rp_q
   \quad \forall s\in I
\end{align}
Applying $\frac{d}{ds}|_{s=t}$ to each side of the equation and using Proposition 22.15 of \cite{lee2013smooth} we obtain
\begin{align}
0\cdot\omega_q &= \lp\mathcal{L}_{X_t}\lp \rho_t\omega \rp + \frac{d}{ds}\Bigr|_{s=t}\rho_t\cdot \omega\rp_q = \\
&= \lp \dive{\omega}{X_t} + X_t(\rho_t) + \frac{d}{ds}\Bigr|_{s=t}\rho_t\rp \cdot\omega_q
\end{align}
From Lemma \ref{lem:div-form-density} and the generality of $q$ the thesis follows. 
\end{proof}
\end{thm}
\subsection{Continuity equation on characteristic curves}
Albeit insightful, directly using Theorem \ref{thm:continuity-equation} for computing the change of density has limited practical usefulness, as using numerical methods for solving PDEs scale poorly with the dimensionality of the manifold. As we already observed in Chapter \ref{ch:vol-forms}, in the context of normalizing flows, we are specifically interested in computing the value of the density $\rho_t\mu$ on the trajectory $\phi_X^{t,0}$.

Following the notation defined in chapter \ref{ch:vol-forms} we are interested in the lifted map $\densitylift{\lp{\phi_{X}^{s,t}}\rp}: \mathcal{D}M \to \mathcal{D}M$: 
such that 
\begin{align}
 \densitylift{\lp\phi_{X}^{t,s}\rp}\lp(\rho_s\mu)_q\rp = \lp\rho_t\mu\rp_{\phi_X^{t,s}(q)} \quad \forall q\in M,\ \forall s,t\in \bbR
\end{align}
The family of maps $\{\densitylift{\lp{\phi_{X}^{s,t}}\rp}\}_{s,t\in\bbR}$  gives us the value of the transformed density, computed at the integral curves of the vector field $X$, which is the quantity that we are interested in. A key observation is that the family $\{\densitylift{\lp{\phi_{X}^{s,t}}\rp}\}_{s,t}$ defines a time dependent flow on the density bundle $\mathcal{D}M$
\begin{prop}
Let $M$ be a smooth manifold and $X$ a complete smooth time-dependent vector field and $\{{\phi_{X}^{s,t}}\}_{s,t\in\bbR}$ its time dependent flow. Then family $\{\densitylift{\lp{\phi_{X}^{s,t}}\rp}\}_{s,t\in\bbR}$ defines time dependent flow on the density bundle $\mathcal{D}M$ and therefore generates a time dependent vector field $\densitylift{X}$ on $\mathcal{D}M$
\end{prop}
\begin{proof}
Since $\phi_{X}^{s,t}$ is a vector bundle ismorphism, it is in particular a diffeomorphism. 
Let $v_q\in \mathcal{D} M$ with base point $q:=\pi_M(v)\in M$.
Using the definition of lift to the density bundle:
\begin{align*}
&\densitylift{\lp{\phi_{X}^{r,s}}\rp}\circ\densitylift{\lp{\phi_{X}^{s,t}}\rp}\lp v_q\rp = 
\densitylift{\lp{\phi_{X}^{r,s}}\rp}\lp \lp\lp\phi_{X}^{t,s}\rp^*v_q  \rp_{\phi_{X}^{s,t}(q)}\rp = \\ &=
\lp  \lp\phi_{X}^{s,r}\rp^* \circ\lp\phi_{X}^{t,s}\rp^*v_q  \rp_{\phi_{X}^{s,t}\circ\phi_{X}^{s,t}(q)} = 
\lp  \lp\phi_{X}^{t,s} \circ \phi_{X}^{s,r}\rp^*v_q  \rp_{\phi_{X}^{s,t}\circ\phi_{X}^{s,t}(q)} = \\
&=\lp  \lp\phi_{X}^{t,r} \rp^*v_q  \rp_{\phi_{X}^{r,t}(q)} = \densitylift{\lp{\phi_{X}^{r,t}}\rp}\lp v_q\rp\quad \forall r,s,t \in \bbR
\end{align*}
\end{proof}

To find an explicit expression for $\densitylift{X}$, we fix a smooth positive density $\mu\in \smoothsec{M}{\cD M}$ and consider the induced vector bundle isomorphism $\bbR\times M \to \mathcal{D}M$. Via this global trivialization, the vector field\footnote{With an abuse of notation we identify the vector field $\densitylift{X}$ on $\mathcal{D}M$ and the vector field on $\bbR\times M$ induced by the vector bundle isomorphism. Since the isomorphism depends explicitly on $\mu$ the expression that we will find will explicitly depend on $\mu$ even if the original vector fields was defined independently.} $\densitylift{X}$ is a time dependent vector field on $\bbR\times M$, linear on the first component (since $\densitylift{\lp{\phi_{X}^{r,s}}\rp}$ is linear on the fiber) and such that the projection to $M$ corresponds with $X$. 
Using the linearity, fixed a time $t\in\bbR$, it is sufficient to evaluate the vector field $\densitylift{X}_t$ at $(1, q)\in \bbR\times M$. In order to compute this value, we first fix an initial density $\rho_t\mu$ where $\rho_t\equiv 1$ and then define $\rho_s\in C^\infty(M)$ such that  $\rho_s\mu:=\lp\phi_X^{s,t}\rp^*\lp\rho_t\mu\rp$. Now let $a(s)\mu_{\phi^{s,t}_X(q)}:= \densitylift{\lp\phi_X^{s,t}\rp}\lp\mu_q\rp = \densitylift{\lp{\phi_{X}^{s,t}}\rp}\lp(\rho_t\mu)_q\rp$ be the solution of the Cauchy problem associated with the field $\densitylift{X}$, with initial point $(1, q)$ and initial time $t$:
\begin{align}
a(s)\mu_{\phi^{s,t}_X(q)} = \densitylift{\lp{\phi_{X}^{s,t}}\rp}\lp\rho_t(q)\mu_q\rp  = 
\rho_s\lp\phi_X^{s,t}(q)\rp\mu_{\phi^{s,t}_X(q)}
\end{align}
We therefore have that $a(s) = \rho_s\lp\phi_X^{s,t}(q)\rp$ gives the value of transformed density along the flow trajectory. We then notice that this value is also given by the solution of the continuity equation computed on the trajectory given by the flow of the vector field $X$. To find the value of $\densitylift X_t(1, q)$ with respect to its first component it is then sufficient to compute $\dot{a}(t)$ using the continuity equation:
\begin{align}
\dot{a}(t) &= \frac{d}{ds}\biggr|_{s=t}\lp \rho_s\lp\phi_X^{s,t}(q)\rp \rp =   \lp\frac{d}{ds}\biggr|_{s=t}\rho_s\rp \lp q\rp + \rho_t\lp \frac{d}{ds}\biggr|_{s=t}  \phi_X^{s,t}(q) \rp = \\
&= - \dive{\mu}{X_t}\lp q \rp - X_t(\rho_t) + X_t(\rho_t) = - \dive{\mu}{X_t}\lp q \rp
\end{align}
The method of solving first order PDEs computing the value of the solution on the flow of an associated vector field is known as {\it method of characteristics}. \footnote{see Chapter 9 of \cite{lee2013smooth}.}

We then have found the following expression for $\densitylift {X}_t$:
\begin{align}
\densitylift{X}_t: \bbR\times M &\to T\bbR\times TM\\
(a, q)&\mapsto \lp - a\ \dive{\mu}{X_t}(q)\cdot\partial_{x_0}, \,X_t(q)\rp
\end{align}
Where $\partial_{x_0}$ is the coordinate field for $T\bbR$.
\subsubsection{Integrating the density on characteristic curves}
Suppose we are given an initial density $\rho_0\mu$. We are interested in computing the value of $\rho_t\lp\phi_X^{t,0}(q)\rp$ for a given $t\in \bbR$ and $q\in M$. This corresponds to finding the solution to the Cauchy problem for $\densitylift{X}$ with initial point $(\rho_0(q), q)$ and initial time $0$. By the linearity of $\densitylift{X}$ in its first component the solution is given by:
\begin{align}
\rho_t\lp\phi_X^{t,0}(q)\rp = - \exp\lp \int_0^t \dive{\mu}{X_t}\lp\phi_X^{t,0}(q)\rp dt \rp \cdot \rho_0(q)
\end{align}
In many of the applications we are interested in $\log \rho_t$:
\begin{align}\label{eq:cnf-dive}
\log \rho_t\lp\phi_X^{t,0}(q)\rp = - \int_0^t \dive{\mu}{X_t}\lp\phi_X^{t,0}(q)\rp dt + \log\rho_0(q)
\end{align}

\label{ch:fields}
\chapter[Parameterizing vector fields]{Parameterizing vector fields}

Given a manifold $M$, we saw in last chapter that vector fields naturally give rise to diffeomorphisms on $M$, which can then be used to define continuous normalizing flows on $M$. 

We are then left with the problem of parameterizing an expressive enough set of vector fields on the manifold.  When we try to parameterize a large set of functions in a modular way, we look at neural networks as a natural solution, however, they can only parameterize functions $\bbR^n\to\bbR^m$, and thus there is no straightforward way to use them for this task.

Finding the best way of parameterizing vector fields on manifolds is an interesting problem with no unique solution, how to tackle it will largely depend on how the manifold is defined and what data structure is used to parameterize it in practice. Nevertheless, all the objects and methods discussed in the rest of the paper are defined independently from the specific parameterization method chosen. Therefore, if in the future a better way of parameterizing vector fields will emerge, they will still be applicable.

Notwithstanding the above, in this section, we will try to give some guidance on how to approach this problem. In the first part, we will show how, using a generating set, it can be reduced to the much easier task of parameterizing functions on manifolds. We will then give some practical advice in the case where the manifold is a Homogenous space, or it is described using an embedding in $\bbR^m$. For the rest of the thesis, given a function $f:M \to \bbR^m$ we will indicate with $f_i:M\to \bbR$ its $i$-th component, such that $f = (f_1, \cdots, f_m)$.

\subsection{Local frames and global constraints}\label{sec:vector-fields-generators}
We begin by analyzing how we parameterize vector fields in $\bbR^n$, to investigate to what extent we can generalize this procedure. As we saw in Secton \ref{sec:node-rn}, in the Euclidean space vector fields are simply functions $f:\bbR^n\to\bbR^n$. In a more geometrical language, the function $f$ defines the vector field $X$ in the following way:
\begin{align}\label{eq:field-Rn}
  X = f_1\partial{x_1} + \cdots +  f_n\partial{x_n}.
\end{align}
The converse is also true: \emph{for every vector field $X$ there exist a unique continuous function $f:\bbR^n\to\bbR^n$ such that Equation \eqref{eq:field-Rn} holds}. On a generic $n$-dimensional smooth manifold this is only true {\it locally}. This means that there exists a open cover $\{U_i\}_{i\in \mathcal{I}}$ of $M$\footnote{Assuming that the manifold is second countable, there exists  $\mathcal{I}$ that is finite and has cardinality $n+1$, see Lemma 7.1 in \cite{metric-structures-diff}.}, called the {\bf trivialization cover}, such that $TM$ restricted to each $U_i$ is isomorphic to the trivial bundle. This is equivalent to saying that for every set $U_i$ there exist $n$ smooth vector fields $E_1^{(i)}, \cdots, E_n^{(i)}\in \smoothsec{U_i}{TU_i})$ such that for every smooth vector field $X\in \smoothsec{M}{TM}$ there exists a unique smooth function $f: U_i\to \bbR^n$ such that:
\begin{align}
    X \bigr|_{U_i}= f_1E_1^{(i)}+ \cdots+ f_nE_n^{(i)}.
\end{align}
We then can call $E_1^{(i)}, \cdots, E_n^{(i)}$ a {\bf local frame}. A local frame that is defined on an open domain $U=M$ (this means on the entire manifold) is called a {\bf global frame}. On a manifold there exists plenty of local frames, in fact given a smooth local chart $(U;x_i)$ the fields $\partial_{x_1},\cdots, \partial_{x_n} \in \smoothsec{U}{TU}$ form a local frame called {\bf coordinate frame}. In the special case of $\bbR^n$, its coordinate frame is a global frame. Unfortunately, in general, not every manifold has a global frame, the simplest example is the sphere $\bbS^2$. In the sphere case it is well known that there exists no vector field that is everywhere nonzero, this result goes by the {\it hairy ball theorem}. 
It is then clear that no pair of vector fields $E_1, E_2 \in \smoothsec{M}{TM}$ can form a global frame, in fact, there will always be a point $q\in \bbS^2$ such that: 
\begin{align}
    \text{span}\lp \lp E_1\rp_q, \lp E_2\rp_q \rp \le 1 < 2 = \text{dim}\lp T_q M\rp.
\end{align}
The manifolds for which a global frame $E_1, \cdots, E_n\in \smoothsec{M}{TM}$ exists are called {\bf parallelizable manifolds}, for this class we can parameterize all smooth vector fields on $M$ in the same way as we did on $\bbR^n$. This means choosing a smooth function $f: M\to \bbR$ and defining a vector field $X$:
\begin{align}
    X = f_1E_1+ \cdots+ f_nE_n.
\end{align}
A manifold is parallelizable iff its tangent bundle is isomorphic to the trivial bundle: $\bbR^n\times M\cong TM$. A global frame gives an explicit isomorphism:
\begin{align*}
\bbR^n\times M&\to TM,\\
(\beta,q)&\mapsto \beta_1\lp E_1\rp_q +\cdots +  \beta_n\lp E_n\rp_q.
\end{align*}
An important and large class class of parallelizable manifolds is given by Lie Groups, which are smooth manifold which additionally posses a group structure compatible with the manifold structure. 
\subsection{Lie groups}\label{app:lie-groups}
A {\bf Lie group} $G$ is a smooth manifold with the additional structure of a group, where the group multiplication and inversion are smooth maps. Lie groups are an important instrument in physics where they are used to model continuous symmetries. 
Many relevant Lie groups arise as subgroups of the matrix groups $\text{GL}_n(\bbR)$ and $\text{GL}_n(\bbC)$ of real and complex invertible matrices with matrix multiplications as a group product. 

The {\bf Lie algebra} $\mathfrak{g}$ of a Lie group $G$ is the tangent space of the group at the identity element $\mathfrak{g}:=T_eG$. 
The Lie algebra $\mathfrak{g}$ can be identified with the space of (right) left invariant vector fields on $G$. In fact any vector $v\in \mathfrak{g}$ defines a left invariant vector field $v^L$ and a right invariant vecotor field $v^R$ in the following way:\begin{equation}
    v^L_a := \dd L_a(v),\quad  v^R_a := \dd R_a(v), \quad \forall a\in G.
\end{equation}
Where $L_a,R_a:G\to G$ are respectively the left and the right group multiplication. Conversely any left (right) invariant vector field $V\in \smoothsec{G}{TG}$ gives a Lie algebra element $V_e\in \mathfrak{g}$. With this identification we can define the Lie bracket in $\mathfrak{g}$ using the Lie bracket between the associated left invariant vector fields:
\begin{equation}
    [v,w] := [v^L, w^L], \qquad \forall v,w \in \mathfrak{g}.
\end{equation}
A fundamental property of Lie groups is that they are parallelizable manifolds. A basis $\{v_1,\cdots,v_n\} \subset T_eG=\alg{g}$ 
defines a {\bf global frame} $\{E_i\}_{i=1}^n$ for $TG$, where $E_i:=v_i^L$, or $E_i:=v_i^R$. 

Any scalar product $\langle\cdot,\cdot\rangle_{\mathfrak{g}}$ on $\mathfrak{g}$ defines a {\bf left invariant Riemannian metric} $g$ on $G$:
\begin{equation}
    g(v_a, w_a) := \langle \dd L_{a^{-1}}(v_a), \dd L_{a^{-1}}(w_a)\rangle_\mathfrak{g},
    \quad \forall a\in G,\ \ \forall v_a,w_a\in T_aG.
\end{equation}
This, in turn, induces a {\bf left-invariant Riemannian density} $\mu_g$, which is unique up to a normalizing constant (which depends on the initial scalar product choice). The associated left-invariant Borel measure is known as (left) {\bf Haar meausure}.
A similar construction can be done to define a {\bf right invariant volume form} on $G$. A Lie group is {\bf unimodular} if its left and right invariant volume forms coincide. Examples of unimodular groups are {\it compact} Lie groups and {\it semisimple} Lie groups.
For proofs, additional details and background we refer to {\citet{lee2013smooth}} Chapers 7,16 and \citet{falorsi2019reparameterizing} Appendix D.

Given $v\in \mathfrak{g}$, the exponential map is defined as $\exp(v) := \gamma(1)$ where $\gamma:\bbR\to G$ is the only 1-parameter subgroup such that $\dot\gamma(0) = v$. The exponential map $\exp: \mathfrak{g} \to G$ describes the flow of left and right invariant vector fields: 
\begin{align}\label{eq:flow-left-lie}
\phi_{v^L}^t(a) = a\exp(vt), \quad \phi_{v^R}^t(a) = \exp(vt)a,\quad \forall a\in G,\forall v\in \mathfrak{g}.
\end{align}
Notice that left invariant vector fields act by right translation and vice-versa. Therefore,
given any right invariant vector field, $v^R\in \smoothsec{G}{TG}$ its flow $\Phi_{v^R}^t: G\to G$ is an {\bf isometry} with respect to the left invariant Riemannian metric $g$. This means that $v^R$ has $0$ divergence, in fact:
\begin{align}
    \dive{\mu_g}{v^R}\mu_g = \mathcal{L}(\mu_g) &= \dfrac{d}{dt}\biggr|_{t=0}\lp\lp\phi^t_{v^R}\rp^*\mu_g\rp = \\
    &=\dfrac{d}{dt}\biggr|_{t=0}\lp \lp L_{\exp(tv)}\rp^*\mu_g\rp =
    \dfrac{d}{dt}\bigr|_{t=0}\lp\mu_g\rp = 0\cdot \mu_g,
\end{align}
which implies $\dive{\mu_g}{v^R}=0$. We therefore we can obtain a {\bf global frame} $\{v^R_i\}_{i=1}^n$ {\bf formed by zero divergence vector fields}. When the Lie group is unimodular we can use left and right invariant vector fields interchangeably. 
\section{Generators of vector fields}\label{app:gen-vec-field}
We have seen that for parallelizable manifolds, once we have defined a global frame, we have a bijective correspondence between functions $C^\infty(M, \bbR^n)$ and smooth vector fields:
\begin{align}
    C^\infty(M, \bbR^n) &\to \smoothsec{M}{TM},\\
    f&\mapsto f_1E_1+ \cdots+ f_nE_n.
\end{align}
For non parallelizable manifolds, we fail to find a global frame because, given any $n$ vector fields $\{E_i\}_{i=1}^n$,
there always exist points $q$ where $\{\lp E_i\rp_q\}_{i=1}^n$ fail to span all $T_qM$:
\begin{align*}\label{eq:func-field-frame}
    \text{span}\lp\{\lp E_i\rp_q\}_{i=1}^n\rp \subsetneq T_qM.
\end{align*}
The idea is then to add vector fields to the set $\{\lp E_i\rp_q\}_{i=1}^n$, giving up on the injectivity constraint, until they "generate" all $\smoothsec{M}{TM}$. To make this statement more precise we have to use the language of {\bf modules}. In fact in general the space of smooth sections of a vector bundle $(E, \pi, M)$ forms a module over the ring $C^\infty(M)$ of the smooth functions on $M$.
\begin{defn}
A finite set of vector fields $\{X_i\}_{i=1}^m\subset \smoothsec{M}{TM}, m\in \bbN_{>0}$ is a generating set for the $C^\infty(M)$-module of the smooth vector fields on $M$ if for every vector field $X\in \smoothsec{M}{TM}$ there exist 
$\{f_i\}_{i=1}^m\subset C^\infty(M)$ such that:
\begin{align}
    X = f_1X_1+ \cdots+ f_mX_m.
\end{align}
If there exist a generating set for  $\smoothsec{M}{TM}$ we then say that $\smoothsec{M}{TM}$ is finitely generated.
\begin{lem}\label{lem:span-gen}
Let $M$ be a smooth manifold and let $\{X_i\}_{i=1}^m\subset \smoothsec{M}{TM},$ $m\in \bbN_{>0}$ a set of smooth vector fields such that $\text{span}\lp\{\lp X_i\rp_q\}_{i=1}^n\rp = T_qM,$ $ \forall q\in M$. Then $\{X_i\}_{i=1}^m$ is a \textbf{generating set} for $\smoothsec{M}{TM}$.
\end{lem} 
\begin{proof}

Consider the open sets 
$$U_I:=\{q\in M | \{X_i(q)\}_{i\in I}\quad \text{are linearly independent}\}$$ where $I\subset \{1,\cdots, m\}$ is any subset of indices of cardinality $n$. let $\mathcal{I}:= \{I | I\subset \range{m},\ \#(I)=n,\ U_I  \text{ is not empty}\}$.  To see that these sets are open, observe that in a local coordinate chart $(U; x_i)$ we can write $X_{i}|_U = \sum_{j=1}^n a^U_{ij}\partial_{x_j}  \forall i\in \range{m}, \ a_{ij}^U\in C^\infty(U)$. In local coordinates the linear independence of $\{X_i(q)\}_{i\in I}$, is equivalent to $\text{det}(A(x))\neq 0$. Where if $I={i_1,\cdots,i_n}$ $A(x)$ is defined as $A(x)_{jk} = a^U_{i_k,j}$. From the definition of $U_I$ it descends that the family $\{U_I |\ I\subset \{1,\cdots,m\},\ \#(I)=n,\ U_I  \text{ is not empty}\}$ forms a open trivialization cover. 
 Fixed $I\in \mathcal{I}$ there exist smooth functions $\{f_i^I\}_{i\in I}\subset C^\infty(U_I)$ such that 
 \begin{align}
     X|_{U_i} = \sum_{i\in I} f_i^I E_i|_{U_I}
 \end{align}
Now let $\{\psi_I\}_{i\in \mathcal{I}}$ be a smooth partition of unity subordinate to $\{U_I\}_{I\in \mathcal{I}}^n$. Defining 
\begin{align}
\grave{f}^{I}_{i} := \begin{cases} \psi_I\cdot f_i^I, & \mbox{on}\ U_I, \\ 0, & \mbox{on} M\setminus\mbox{supp}(\psi_I). \end{cases}
\quad \forall I\in \mathcal{I},\ \forall i\in I.
\end{align}
We have that:
\begin{align}
    \psi_I \cdot X = \sum_{i\in I} \grave{f}_i^I X_i
\end{align}
And therefore
\begin{align}
    X = \sum_{j=1}^m \lp\sum_{I\in \mathcal{I}\ s.t.\ j\in I}\grave{f}_i^I\rp X_j
\end{align}
\end{proof}
\end{defn}
\begin{thm}\label{thm:exists-generator}
Let M be a (second countable) smooth manifold $M$. Then the module of smooth vector fields $\smoothsec{M}{TM}$ is finitely generated.
\begin{proof}
Since $M$ is second countable we can apply Lemma 7.1 in \cite{metric-structures-diff} and say that there exist an open trivialization cover $\{U_i\}_{i=0}^n$, where $n$ is the dimension of $M$. We denote with $E_1^{(i)},\cdots, E_n^{(i)}$ the local frame relative to the domain $U_i\subseteq M$. Now let $\{\psi_i\}_{i=0}^n$ be a smooth partition of unity subordinate to $\{U_i\}_{i=0}^n$. We define the global vector fields on $M$:
\begin{align}
\grave{E}^{(i)}_{j} := \begin{cases} \psi_i\cdot E^{(i)}_{j}, & \mbox{on  } U_i, \\ 0, & \mbox{on  } M\setminus\mbox{supp}(\psi_i). \end{cases}
\quad \forall i\in\range{n},\ \forall j\in\range{n}.
\end{align}
Using Lemma \ref{lem:span-gen} we have that $\{\grave{E}^{(i)}_{j}\}_{\substack{i=\in\range{n}\\j=\in\range{n}}}$ is a generating set of $\smoothsec{M}{TM}$.

\end{proof}
\end{thm}
From this Theorem and the definition of generating set we can extract a general methodology to parameterize all vector fields on smooth manifolds:
\begin{enumerate}
    \item choose a suitable generating set $\{ X_i\}_{i=1}^m$,
    \item find a way of parameterizing functions $f_i:M\to \bbR$,
    \item model a generic vector field $X$ as a linear combination:
    \begin{align}\label{eq:vec-field-gen}
        X = f_1X_1+ \cdots+ f_mX_m.
    \end{align}
\end{enumerate}
The above proof also tells us that a simple and general recipe to obtain a generating set is to take a collection of local frames and multiply them by a smooth partition of unity. 
The efficiency of this framework is determined by the cardinality of the generating set: lower cardinality requires parameterization of fewer functions. The proof gives us an initial upper bound on the lowest cardinality of the generating set we can achieve for a generic manifold: $n^2+n$ where $n$ is the dimension of the manifold. We will see that using the Whitney embedding theorem, and the fact that any smooth manifold admits a Riemannian metric, this number can be further reduced to $2n + 1$. 
\subsection{Time dependent vector fields}
When parameterizing time-dependent vector fields, we need to model a vector field $X_t$ for every time $t\in\bbR$. Using a generating set we can easily accomplish this by parametrizing a function $f:\bbR\times M\to\bbR^m$ and defining:
\begin{align}
    X_t:= f_1(t,\cdot)X_1 + \cdots + f_m(t,\cdot)X_m.
\end{align}

\section{Divergence using a generating set}

Suppose now we are using a generating set $\{X_i\}_{i=1}^m$ to parameterize a vector field $X = \sum_{i=1}^m f_iX_i, \ f\in C^\infty(M, \bbR^m)$. From the properties of the Lie derivative it is straightforward to see that its divergence decomposes in:
\begin{align}\label{eq:dive-gen}
    \dive{\mu}{X} &= \sum_{i=1}^m X_i(f_i) + \sum_{i=1}^m f_i\dive{\mu}{X_i} = \\
     &= \sum_{i=1}^m df_i(X_i) + \sum_{i=1}^m f_i\dive{\mu}{X_i}.
\end{align}
Let's analyze the above expression: the first term involves the derivative of the parameterizing functions $f_i$ on the vector field directions $X_i$, and it is the same as the divergence in the Euclidean space (with $X_i = \partial_{x_i}$). 
The second term involves the divergence of the vector fields in the generating set. Therefore if we are using a generating set formed by vector fields with known, or easy divergence expression we can use Equation \eqref{eq:dive-gen} for divergence computation without relying on local charts. We will show later that this is always the case for homogeneous spaces (see Theorem \ref{thm:dive-hom-gen}) 

As already noticed by \cite{grathwohl2018scalable}, the first term in Equation \ref{eq:dive-gen} has complexity proportional to $m$ evaluations of the function $f$. Therefore, when integrating the divergence along on flow curves in Equation \eqref{eq:cnf-dive}, it can be expensive to compute at every step of the numerical solver. To mitigate this issue we show that the Monte Carlo unbiased estimator of the divergence presented in  \cite{grathwohl2018scalable} can be generalized to manifold setting:
\begin{thm}\label{thm:dive-estimator}
Consider a finite set of vector fields $\{X_i\}_{i=1}^m\subset \smoothsec{M}{TM}$ and a function $f\in C^\infty(M, \bbR^m)$ on a smooth manifold $M$. Then for every probability distribution $\rho(\varepsilon)$ on $\bbR^m$  with zero mean $\E{\rho(\varepsilon)}{\varepsilon} = 0$ and identity covariance $\textup{Cov}_{\rho(\varepsilon)}[\varepsilon] = I$ it holds:
\begin{align}\label{eq:div-gen-est}
    \sum_{i=1}^m df_i(X_i) = \E{\rho(\varepsilon)}{f^*\varepsilon \lp \sum_{i=1}^m\varepsilon_i X_i\rp}  
\end{align}
Moreover if the function $f$ can can be expressed as the composition of two smooth functions $f = f^{(2)}\circ f^{(1)}$, the above term can be rewritten as:
\begin{align}\label{eq:div-gen-est-comp}
    = \E{\rho(\varepsilon)}{{f^{(2)}}^*\varepsilon \lp \dd f^{(1)} \lp \sum_{i=1}^m\varepsilon_i X_i\rp\rp}.
\end{align}
\begin{proof}
First of all we notice that the LHS of equation \eqref{eq:div-gen-est} can be rewritten as:
\begin{align}
\sum_{i=1}^m df_i(X_i) = \text{tr}\lp A \rp
\end{align}
Where $A$ is the matrix function that contains all the partial derivatives:
\begin{align}
    A_{ij} = X_i(f_j) = df_j(X_i)
\end{align}
We can then use Hutchinson's trace estimator \citep{hutchinson1989} to obtain
\begin{align}
\text{tr}\lp A \rp& = \E{\rho(\varepsilon)}{\varepsilon^\top A \varepsilon} = \E{\rho(\varepsilon)}{\sum_{i,j = 1}^m \varepsilon_idf_i\lp\varepsilon_j X_j\rp}= \\&= 
\E{\rho(\varepsilon)}{\sum_{i = 1}^m f_i^*\varepsilon_i\lp\sum_{j=1}^m\varepsilon_j X_j\rp}  
= \E{\rho(\varepsilon)}{f^*\varepsilon \lp \sum_{i=1}^m\varepsilon_i X_i\rp}.
\end{align}

Where we used the linearity of the differential and the pullback. Equation \eqref{eq:div-gen-est-comp} can be obtained observing that the pullback is the transpose of the differential.
\end{proof}
\end{thm}

\section{Homogeneous spaces}\label{sec:hom-space}
\begin{defn}
Let $N$, $M$ smooth manifolds and $F:M\to N$ a smooth function between them. Given $X\in \smoothsec{M}{TM}$ and $Y\in \smoothsec{N}{TN}$ smooth vector fields respectively on $M$ and $N$, we say that they are {\bf F-related} if
\begin{align}
    \dd F_p\lp X_p\rp = Y_{F(p)}, \quad \forall p\in M.
\end{align}
\end{defn}

A {\bf homogeneous space} is a manifold equipped with a transitive Lie group action:
\begin{defn}
A smooth manifold $M$ is homogeneous if a Lie group G acts transitively on $M$, i.e.:
    \item There exists a smooth map $G\times M \to M, (a,x)\mapsto a.x$ such that:
    \begin{itemize}
        \item $(ab).x = a.(b.x)$,
        \item $e.x = x, \quad \forall x \in M$,
    \item for any $x,y \in M$ there exists an element $a\in G$ such that $a.x = y$.
    \end{itemize}
\end{defn} 
We can construct homogeneous spaces taking a Lie group $G$ and $H<G$ a closed subgroup. Then the quotient manifold $G/H$ is a homogeneous space with action $a.(bH):= abH$ and the projection map $\pi:G\to G/H$ is a smooth submersion\footnote{Theorem 21.17 \citet{lee2013smooth}.}.
In particular, any Lie group $G=G/\{e\}$ can be considered a homogeneous space with the action given by left multiplication. The following theorem tells us that the above construction completely characterizes homogeneous spaces:
\begin{thm}\footnote{Theorem 21.18 \citet{lee2013smooth}.}
Let $G$ be a Lie group, let $M$ be a homogeneous $G$-space, and let $q$ be any point of $M$. The
isotropy group $G_q:=\{a\in G : a.q = q\}$ is a closed subgroup of $G$, and the map:
\begin{align}
F : G/G_q &\to M \\
   aG_q&\mapsto a.q  
\end{align}
is an equivariant diffeomorphism.
\end{thm} 



Given $v\in\alg{g}$ an element of the Lie algebra, we can define a vector field $\widetilde v \in \smoothsec{G/H}{TG/H}$ on $G/H$ associated to it:

\begin{align}
    \widetilde v_{aH} := \dfrac{d}{dt}\biggr|_{t=0} \exp(tv).aH, \quad \forall a\in G
\end{align}
Since $\exp(tv).aH = (\exp(tv)a).H$ we have that $\phi^t_{\widetilde v}\circ \pi = \pi \circ \phi^t_{v^R}$. This is equivalent to say that $\widetilde v$ and $v^R$ are $\pi$-related\footnote{Proposition 9.6 \citet{lee2013smooth}.}. Since $\pi: G\to G/H$ is a smooth submersion and we know that there exist a generating set of $\smoothsec{G}{TG}$ made by right-invariant vector fields $\{v_i^R\}_{i=1}^m$, then there exist a set $\{v_i\}_{i=1}^m\subset\alg{g}$ of elements of the Lie algebra such that 
$\{\widetilde{v}_i\}_{i=1}^m\subset\smoothsec{G/H}{TG/H}$
is a generating set for $\smoothsec{G/H}{TG/H}$.
We therefore have proven the following result:
\begin{prop}\label{prop:gen-hom}
Let $M$ a be homogeneous $G$-space. Every basis\footnote{Depending on the space and the group, taking only a subset of the basis might be sufficient for inducing a generating set. An example is the hyperbolic space $\bbH^n$ with action by the {\it special indefinite orthogonal group.} $ SO^+(n,1)$. In fact in this case, a basis for the vector subspace of Lorentz boosts in $\alg{so(n,1)}$ gives rise to a generating set for the vector fields on $\bbH^n$.} $\{v_i\}_{i=1}^m\subset \alg{g}$ of the Lie algebra $\alg{g}$ induces a generating set $\{\widetilde{v}_i\}_{i=1}^m\subset \smoothsec{M}{TM}$ for the smooth vector fields on $M$.
\end{prop}

Let now $\mu$ be a group action invariant\footnote{See sections 2.3.2 and 2.3.3 in \citet{Howard94analysison} for construction of invariant volume forms and Riemannian metrics on homogeneous spaces.} density on $G/H$. 
Then, given $v^R$ any right-invariant vector field on $G$, we have that the flow of $\widetilde v$ is given by the group action of $\exp(tv)$, and therefore $\dive{\mu}{\widetilde v} = 0$:
\begin{align}
 \dive{\mu}{\widetilde v}\mu_g = \mathcal{L}(\mu_g) = \dfrac{d}{dt}\biggr|_{t=0}\lp\lp\phi^t_{\widetilde v}\rp^*\mu_g\rp =
    \dfrac{d}{dt}\bigr|_{t=0}\lp\mu_g\rp = 0\cdot \mu_g.   
\end{align}
We thus have proven the following theorem:
\begin{thm}\label{thm:dive-hom-gen}
Let $M$ a homogeneous $G$-space and let $\mu$ be a $G$ invariant smooth positive density on $M$. For every element $v\in \alg{g}$ of the Lie algebra we have that the associated vector field $\widetilde v$ on $M$ has zero divergence:
\begin{align}
\dive{\mu}{\widetilde v} = 0.    
\end{align}
\end{thm}

We conclude this section by giving a sufficient condition for defining a group action invariant Riemannian metric on a homogeneous space:

\begin{prop}\footnote{Proposition 2.3.14 in \cite{Howard94analysison}.}
Let $G$ be a Lie group and $K$ a closed subgroup of
$G$. Let $g$ be a Riemannian metric on $G$ which is left invariant under
elements of $G$ and also right invariant under elements of $K$. Then there
is a unique Riemannian metric $h$ on $G/K$ so that the natural map $\pi :
G \to G/K$ is a Riemannian submersion. This metric is invariant under the
action of $G$ on $G/K$.

Conversely if $K$ is compact and $h$ is an invariant Riemannian on metric on $G/K$ then there is a Riemannian metric $g$ on $G$ which is left invariant
under all elements of $G$ and right invariant under elements of $K$. We will
say that the metric $g$ is adapted to the metric $h$.
\end{prop}
\section{Embedded submanifolds of $\bbR^m$}
A general way to work in practice with manifolds is using embedded submanifolds of $\bbR^m$. An embedding for a manifold $M$ is a continuous injective function $\iota: M\hookrightarrow \bbR^m$ such that $\iota: M\to \iota(M)$ is a homeomorphism. The embedding is smooth if $\iota$ is smooth and $M$ is diffeomorphic to its image. In this case $\iota(M)$ is a smooth submanifold of $\bbR^m$. For all practical purposes we can directly identify $M$ as a submanifold of $\bbR^m$, the function $\iota: M\hookrightarrow \bbR^m$ then simply denotes the inclusion. Through this identification, we can then consider the tangent space $T_qM,\ q\in M$ as a vector subspace of $T_q\bbR^m$. An embedding is said {\bf proper} if $\iota(M)$ is a closed set in $\bbR^m$.\footnote{Requiring that the embedding is proper excludes embeddings of the form $U\hookrightarrow M$ where $U$ is an open subset of $M$.}
\begin{thm}[Whitney Embedding Theorem, 6.15 in \cite{lee2013smooth}]
Every smooth $n$-dimensional manifold admits a proper smooth embedding in $\bbR^{2n+1}$
\end{thm}
The Whitney embedding theorem tells us that parameterizing manifolds as submanifolds of the Euclidean space gives us a general methodology to work with manifolds. Developing algorithms that assume that the manifold is given as an embedded submanifold of $\bbR^m$ is therefore of outstanding importance. 

For embedded submanifolds parameterizing functions is extremely easy, and can be simply done via restriction: given a smooth function $f:\bbR^m\to \bbR$, $f\circ\iota$ then defines a smooth function from $M$ to $\bbR$. 

Unfortunately, for vector fields, it is not as easy. In fact, in general, given a vector field $X\in \csec{\bbR^m}{ T\bbR^m}$ this does not restrict to a vector field on a submanifold $M\subseteq \bbR^m$, as, in general, given $q\in M$ we have $X_q\not\in T_qM\subseteq T\bbR_q^m$. In order for $X$ to restrict to a vector field on a submanifold $M$, we need for $X$ to be {\bf tangent to the submanifold}:
\begin{align}
    X_q\in T_qM\subseteq T_q\bbR^m,\quad \forall q\in M.
\end{align}
A tangent vector field then defines a vector field on the submanifold:

\begin{lem}
Let $M$ be a smoothly embedded submanifold of $\bbR^m$, and let $\iota: M\hookrightarrow\bbR^m$ denote the inclusion map. If a smooth vector field $Y\in \smoothsec{\bbR^m}{T\bbR^m}$ is tangent to $M$
there is a unique smooth vector field on $M$, denoted by $Y|_M$ , that is $\iota$-related to $Y$. Conversely a vector field $\overline{Y}\in \smoothsec{\bbR^m}{T\bbR^m}$ that is $\iota$-related to $Y$ is tangent to $M$ 
\begin{proof}
See proof of Proposition 8.23 in \cite{lee2013smooth}
\end{proof}
\end{lem}
More importantly, we can parameterize all vector fields on an embedded submanifold using tangent vector fields:
\begin{prop}
Let $M$ be a properly embedded submanifold of $\bbR^m$, and let $\iota: M\hookrightarrow\bbR^m$ denote the inclusion map. For any smooth vector field $X\in \smoothsec{M}{TM}$ there exist a smooth vector field $\overline{X}\in \smoothsec{\bbR^m}{T\bbR^m}$ tangent to $M$ such that:
\begin{align}
     \overline{X}|_M = X.
\end{align}
We call $\overline{X}$ an extension of $X$.
\begin{proof}
Let $U\subseteq\bbR^m$ be a tubular neighborhood of $M$, then by Proposition 6.25 of \cite{lee2013smooth}, there exist a smooth map $r:U\to M$ that is both a retraction and a smooth submersion. Then, since $r$ is a submersion there exist a vector field $\overline{X}\in \smoothsec{U}{TU}$ that is $r$-related to $X$ \footnote{See for example Exercise 8-18 of \cite{lee2013smooth}}. This means $ \dd r_z \overline{X}_z = X_{r(z)} \forall z\in \bbR^m$. Since $r$ is a retraction \begin{align}
\dd\iota_q\circ \dd r_q = Id_{T_q\bbR^m}\   
\Rightarrow\ \lp \dd\iota_q\circ \dd r_q\rp\overline{X}_q = \overline{X}_q
\Rightarrow\ \dd\iota_q X_q = \overline{X}_q\ \forall q\in M,
\end{align}
$\overline{X}$ is $\iota$-related to $X$ and therefore tangent to $M$ and such that $\overline{X}|_M = X$. Then $\overline{X}$ can be used to define a tangent vector field on all $\bbR^m$ using a smooth partition of unity subordinate to the open cover $\{\bbR^m\setminus M, U\}$.
\end{proof}
\end{prop}
From the proof of the theorem, it's clear that the extension of $X$ is not unique. Our objective is then finding a way to parameterize all vector fields tangent to a submanifold. We first observe that given smooth vector fields $X,Y\in \smoothsec{\bbR^m}{T\bbR^m}$ tangent to $M$ and smooth functions $f,g\in C^\infty(\bbR^m)$ then $fX+gY$ is tangent to $M$. This means that the set of all smooth vector fields tangent to $M$ is a {\bf submodule} of the module of smooth vector fields on $\bbR^m$. Following the framework outlined in Section \ref{sec:vector-fields-generators} we then need to find $l$ tangent vector fields $\overline{X}_1,\cdots,\overline{X}_l$ such that $\overline{X}_1|_M,\cdots,\overline{X}_l|_M$ generates all $\smoothsec{M}{TM}$\footnote{Since the extension of a vector field is not unique this is different from finding a generating set for the submodule of vector fields tangent to $M$}.
\subsection{Embedded Riemannian Submanifolds}
If our embedded submanifold manifold is equipped with a Riemannian metric\footnote{We do not assume the Riemannian metric to be inherited from the ambient space.}, the gradient of the embedding gives us a generating set for the tangent bundle. We first prove the following Lemma
\begin{lem}\label{lem:pullback-emb-surj}
Let $M$ be a embedded submanifold of $N$, and let $\iota: M\hookrightarrow N$ denote the inclusion map. Then
\begin{align}
    \iota^*: \iota^*T^*N &\to T^*M \\
    \beta_{\iota(q)}&\mapsto \iota^*\beta_{\iota(q)}:\quad v\mapsto \beta_{\iota(q)}(dv_q)\quad \forall q\in M,\ \forall \beta \in T^*_{\iota(q)} N,\ \forall v_q\in T_qM 
\end{align}
the pullback of the inclusion is a surjective vector bundle homomorphism, where by $\iota^*T^*N$ we denote the pullback bundle $\iota^*T^*N = \{(q,\beta)\in M\times T^*N| \pi(\beta)=\iota(q)\}$
\begin{proof}
Fix $q\in M$. Let $n$ be the dimensionality of $M$ and $m$ the dimensionality of $N$. We need to prove that $\iota^*:T^*_{\iota(q)}N\to T^*_qM$ is surjective.  Let $e_1,\cdots,e_n$ a basis for $T_qM$ and $\eta_1,\cdots, \eta_n$ its dual basis. By the linearity of $\iota^*$ it's then sufficient to prove that there exists $\beta_1,\cdots,\beta_n\in T^*_{\iota(q)}N$ such that for all $i\in\{1,\cdots,n\}$:
\begin{align}\label{eq:pullback-surjective-basis}
    \iota^*\beta_i = \eta_i.
\end{align}
To see this, consider the set $\{d(\eta_1)_q,\cdots,d(\eta_n)_q\} \subset T_{\iota(q)}N$. Since $\iota$ is an embedding, $d\iota$ is injective. Therefore the vectors are linearly independent. 
We can then complete them to a basis $v_1:=d(\eta_1)_q,\cdots,w_n:=d(\eta_n)_q, w_{n+1},\cdots w_m$ of $T_{\iota(q)}N$. Let $\beta_1,\cdots,\beta_m \in T^*_{\iota(q)}N$ the dual basis. We then have:
\begin{align}
    \iota^*\beta_i(e_j) = \beta_i\lp d\lp\eta_j\rp_q\rp = \beta_i\lp w_j\rp = \delta_{ij} = \eta_i\lp e_j\rp \quad \forall i,j\in\{1,\cdots,n\}.
\end{align}
Thus $\beta_i$ satisfies Equation \eqref{eq:pullback-surjective-basis} $\forall i\in\{1,\cdots,n\}$
\end{proof}
\end{lem}
\begin{thm}\label{thm:gen-gradient}
Let $(M,g)$ be a embedded submanifold of $\bbR^m$, and let $z: M\hookrightarrow \bbR^m$ denote the inclusion map. Then $\{\grad(z_i)\}_{i=1}^m$ is a generating set for the smooth vector fields $\smoothsec{M}{TM}$. Where $\grad$ denotes the Riemannian gradient with respect to the metric $g$. 
\begin{proof}
 Consider the differential forms $\{dz_i\}_{i=1}^m\subset \smoothsec{M}{T^*M}$. Using Lemma \ref{lem:pullback-emb-surj} we have that $\text{span}\lp\{dz_i(q)\}_{i=1}^m\rp = T^*_qM$, which means that at every point they span the cotangent space at the point. Using the musical isomorphism, this implies that the riemannian gradients $\grad(z_i)$ span the tangent space at every point: 
 $\text{span}\lp\{(\grad\ z_i)(q)\}_{i=1}^m\rp = T_qM$. 
Using Lemma \ref{lem:span-gen} can conclude that $\{\grad(z_i)\}$ is a generating set for $\smoothsec{M}{TM}$.
\end{proof}
\end{thm}
\subsection{Generators defined by gradients of laplacian eigenfunctions}\label{app:lap-eigenfunctions}
In general, given a function $f\in C^\infty(M)$ on a Riemannian manifold $(M,g)$, its \textbf{laplacian} is defined as the divergence of its Riemannian gradient:
\begin{align}
    \Delta f := \dive{\mu_g}{\grad f}.
\end{align}
Then the divergence of the fields defined in Theorem \ref{thm:exists-generator} is given by the laplacian of the functions $z_i:M\to\bbR$. 

A {\bf laplacian eigenfunction} $f$ is a smooth function such that
\begin{align}\label{eq:eig-sn}
    \Delta f = \lambda f
\end{align}
For some $\lambda \in \bbR$.

Since any closed Riemannian manifolds has an embedding formed by laplacian eigenfunctions \footnote{See \cite{bates2014embedding}.}, in this case, using Theorem \ref{thm:gen-gradient} we can build a generating set formed by the Riemannian gradient of laplacian eigenfunctions. 

An example of particular importance is given by the hypersphere $\bbS^n = \{x\in \mathbb{R}^{n+1}|\ \|x\|=1 \}$. The coordinate functions $z_i(x)= x_i$ form an embedding of laplacian eigenfunctions\footnote{Section 1.2 \cite{bates2014embedding}.}. This gives us $n+1$ vector fields $\{\overline{\grad(z_i)}\}_{i=1}^{n+1}\subset \smoothsec{\bbR^{n+1}}{T\bbR^{n+1}}$ tangent to $\bbS^n$ such that their restriction to $M$ forms a generating set for $\smoothsec{\bbS^{n}}{T\bbS^n}$.  
\begin{align}
&\overline{\grad(z_i)}(x) = e_i - \langle x, e_i\rangle x \quad \forall i\in \{1,
\cdots, n+1\}\\
&\Delta z_i(x) = -nx_i
\end{align}
Where $e_1, \cdots, e_{n+1}\in \bbR^{n+1}$ are the vectors of the canonical basis of $\bbR^{n+1}$.
\subsection{Isometrically embedded Submanifolds}
Let $(M, g)$ be a Riemannian submanifold of $\bbR^m$. We can consider $\bbR^m$ as a Riemannian manifold with metric $\langle\cdot,\cdot\rangle$ given by the scalar product. We say that an embedding
$z:M\to\bbR^m$ is {\bf isometric}\footnote{For definitions, proofs and additional results on isometrically embedded submanifolds see Chapter 8 in \cite{lee2006riemannian}} if $z^*\langle\cdot,\cdot\rangle = g$. 
If the manifold $M$ is isometrically embedded in $\bbR^m$ we have that the restriction\footnote{The restriction of the tangent bundle $T\bbR^m$ to $M$ is coincides with the pullback bundle $z^*T\bbR^m$} $T\bbR^m|_M:=\{(q, v)\in M\times T\bbR^m: v\in T_q\bbR^m\}$ decomposes in the orthogonal direct sum:
\begin{align}
    T\bbR^m|_M = TM \oplus \normb{M}
\end{align}
Where $\normb{M} := \{(q, v_q)\in T\bbR^m|_M: \langle v_q, w_q\rangle= 0\ \forall w_q\in T_qM\}$ is called the {\bf normal bundle}. We can then define the orthogonal vector bundle projections:
\begin{align}
    (\ \cdot\ )^\top = \pi^\top:T\bbR^m|_M&\to TM\\
    (\ \cdot\ )^\bot =\pi^\bot:T\bbR^m|_M&\to \normb{M}
\end{align}
\begin{prop}\label{prop:ort-emb-frame}
Let $M\subseteq \bbR^m$ be a isometrically embedded submanifold.
Given a point $q\in M$ there always exist a neighborhood $q\in\overline{U}\subseteq \bbR^m, \ U:=\overline{U}\cap M$ and a local smooth orthonormal frame $\{E_i\}_{i=1}^m\subseteq \smoothsec{\overline{U}}{T\overline{U}}$ such that $\{E_i\}_{i=1}^n\subseteq \smoothsec{U}{TU}$ is a local smooth orthonormal frame for $TM$ and $\{E_i\}_{i=n+1}^m\subseteq \smoothsec{U}{TU}$ is a local smooth orthonormal frame for $\normb{M}$
\end{prop}

For isometrically embedded submanifolds, $\grad(z_i)$ is simply given by the orthogonal projection of the constant coordinate field $\partial_{x_i}|_M$ to $TM$.
\begin{lem}
Let $M\subseteq \bbR^m$ be a isometrically embedded submanifold with embedding $z:M\to \bbR^m$. Then the fields $\grad(z_i)\in \smoothsec{M}{TM}$ are given by the orthogonal projection of the constant coordinate fields $\partial_{x_i}|_M$ to $TM$:
\begin{align}
    \grad(z_i) = \pi^\top\lp\partial_{x_i}\big|_M\rp
\end{align}
\begin{proof}
Fix $q\in M$, let $\{E_i\}_{i=1}^m\subseteq \smoothsec{\overline{U}}{T\overline{U}}$ be a local smooth orthonormal frame as described in Proposition \ref{prop:ort-emb-frame}. If we denote with $\{\eta_i\}_{i=1}^m\subseteq \smoothsec{\overline{U}}{T^*\overline{U}}$ the dual coframe, then $dz_i|_U = \sum_{k=1}^n E_k(z_i)\eta_k$ and
\begin{align}
{\grad(z_i)}\big|_U= \sum_{k=1}^n E_k(z_i)E_k = \sum_{k=1}^n E_k(\overline{z}_i)E_k\big|_U =  \lp\sum_{k=1}^m E_k(\overline{z}_i)E_k\big|_U\rp ^\top    
\end{align}
Where $\overline{z}$ is the extension of the embedding function $z_i:M\subseteq \bbR^m \to \bbR$, given by the $i$-th coordinate projection $\overline{z}_i(x) = x_i\quad \forall x\in \bbR^m$. We then have that
\begin{align}
\sum_{k=1}^m E_k(\overline{z}_i)E_k\big|_U = \sum_{k=1}^m \frac{\partial \overline{z}_i}{\partial x_k}\partial_{x_k}\big|_U = \partial_{x_i}\big|_U 
\end{align}
\end{proof}
\end{lem}

\begin{defn}
Let $M\subseteq \bbR^m$ be a isometrically embedded submanifold and let $\nabla$, $\overline{\nabla}$ denote respectively the Levi-civita connection on $M$ and $\bbR^m$. The  second fundamental form $\sff$ is the function
\begin{align}
    \sff: \smoothsec{M}{TM}\times \smoothsec{M}{TM} &\to \smoothsec{M}{\normb{M}}\\
    \lp X, Y\rp&\mapsto \sff\lp X, Y\rp = \lp \overline{\nabla}_XY\rp^\bot = \lp\overline{\nabla}_{\overline{X}} \overline{Y}\rp^\bot
\end{align}
Where $\overline{X}$, $\overline{Y}$ are arbitrary extensions of $X$ and $Y$.
The {\bf mean curvature} is defined as the trace of the second fundamental form:
\begin{align}
    H: M&\to \normb{M}\\
    q&\mapsto H(q) = \frac{1}{n}\sum_{i=1}^n\sff(E_i, E_i)_q
\end{align}
Where $\{E_i\}_{i=1}^n$ is a local frame in a neighborhood of $q$.
\end{defn}
The second fundamental form measures the difference between the connection on $M$ and the connection of the ambient space $\bbR^m$:
\begin{thm}[The Gauss formula]
Let $M\subseteq \bbR^m$ be a isometrically embedded submanifold, $X,Y\in \smoothsec{M}{TM}$ and $\overline{X}, \overline{Y}$ arbitrary extensions to $\bbR^m$. The following holds on $M$:
\begin{align}
    \overline{\nabla}_{\overline{X}} \overline{Y} = \nabla_XY + \sff(X,Y)
\end{align}
\end{thm}

The Laplacian of the embedding functions $z_i:M\to \bbR$ is given by the mean curvature\footnote{Proposition 2.3 \citet{chen2013laplace}}:
\begin{prop}
Let $M\subseteq \bbR^m$ be a isometrically embedded submanifold with $z:M\to\bbR^m$ isometric embedding, then:
\begin{align}
\Delta z = nH     
\end{align}
Where the equality holds elementwise.
\begin{proof}
Fix $q\in M$, let $\{E_i\}_{i=1}^m\subseteq \smoothsec{\overline{U}}{T\overline{U}}$ a local smooth orthonormal frame as described in Proposition \ref{prop:ort-emb-frame}. From Equation \eqref{eq:div-semiriem} we know that:
\begin{align}
    &\Delta z_i\big|_U = \sum_{i=1}^n \langle E_j, \nabla_{E_j}\lp \grad(z_i)\rp\rangle = 
    \sum_{j=1}^n  E_j\lp \langle E_j, \grad(z_i)\rangle\rp-\langle \nabla z_i, \nabla_{E_j}E_j\rangle = \\
    &= \sum_{j=1}^n  E_j\lp \langle E_j, \partial_{x_i}\rangle\rp - \langle \partial_{x_i}, \nabla_{E_j}E_j\rangle = 
    \sum_{j=1}^n \overline{\nabla}_{E_j}\lp \left[E_j\right]_i\rp - \left[\nabla_{E_j}E_j\right]_i = \\
    &= \sum_{j=1}^n \sff\lp E_j, E_j\rp = n H\big|_U
\end{align}
\end{proof}.
\end{prop}

\chapter{Backpropagation through flows}

Given a smooth complete vector field $X\in \smoothsec{M}{TM}$, with flow $\{\phi^t_X\}_t$ we want to compute the pullback $(\phi^t_X)^*\psi$ for a covector $\psi\in T^*M$. The pullback generalizes the Vector Jacobian Product (VJP) to smooth manifolds. 

In general, given a diffeomorphism $\Phi: M\to M$, we can lift it to a diffeomorphism $\comp{\Phi}$ on the cotangent bundle:
\begin{align}
    \comp{\Phi}: T^*M& \to T^*M\\
    p_q&\mapsto \lp(\Phi\inv)^*p_q\rp_{\Phi(q)}
\end{align}
Where the notation $p_q$ indicates that the point $p\in T^*M$ has base point $q\in M$ (i.e. $p_q\in T^*_q M$).

We are therefore interested in computing $\comp{\lp \phi_X^t\rp}$. One key property that we will use is that the lift $\comp{\Phi}$ is a symplectomorphism, this means that it pulls back the canonical symplectic form $\alpha\in \smoothsec{T^*M}{ \Lambda^2T^*T^*M}$ to itself:
\begin{equation}
    \lp\comp{\Phi}\rp^* \alpha = \alpha
\end{equation}
In the next section, we provide a minimal primer on symplectic geometry, defining only the objects that are essential for us.
For a more comprehensive introduction on symplectic geometry, we refer the reader to \cite{da2001lectures}. 
\section{A short primer on symplectic geometry}
\begin{defn}[Symplectic form]\label{def:symplectic-form}
Let $M$ be a smooth manifold. A smooth 2-form $\alpha\in \smoothsec{M}{\Lambda^2T^*M}$ which is closed and nondegenerate is called a {\bf symplectic} form. 

The form $\alpha$ is nondegenerate if for every $p\in M$ the linear map:
\begin{align}
    \widetilde{\alpha}_p: T_p M&\to T_p^*M \\
    v&\mapsto v\lrcorner\,\alpha_p = \alpha_p(v,\cdot)
\end{align}
is an isomorphism.

We say that $\alpha$ is closed if $d\alpha = 0$ where $d$ represents the exterior derivative.
\end{defn}
\begin{defn}[Symplectic manifold]
A {\bf symplectic manifold} is a couple $(M,\alpha)$ where $M$ is a smooth manifold and $\alpha$ is a symplectic form on $M$.
\end{defn}
A direct consequence of the definition is that every symplectic manifold is even dimensional.
Given symplectic $\alpha$ form on a manifold $M$, the map $\widetilde \alpha: TM \to T^*M$ denotes the induced vector bundle isomorphism, whose action on each fiber is given by Definition \ref{def:symplectic-form} . 
\begin{defn}[Symplectomorphism]
Let $(M_2,\alpha_1)$ and $(M_2,\alpha_2)$ symplectic manifolds, and let $\Phi:M_1\to M_2$ be a diffeomorphism. Then $\Phi$ is a {\bf symplectomorphism} if $\Phi^*\alpha_2 = \alpha_1$
\end{defn}
The simplest example of symplectic manifold is the real $2n$-dimensional space $\bbR^{2n}$ with coordinates $x_1,\cdots,x_n,y_1,\cdots,y_n$, endowed with the symplectic form:
\begin{equation}
    \alpha_0 = \sum_{i=1}^n dx_i\wedge dy_i
\end{equation}
The Darboux theorem tell us that any $2n$-dimensional symplectic manifold is locally symplectomorphic to $(\bbR^{2n},\alpha_0)$:
\begin{thm}[Darboux]
Let $(M,\alpha)$ a $2n$-dimensional symplectic manifold. Then given a point $q\in M$, there exists a chart $(U;x_1,\cdots,x_n,y_1,\cdots,y_n$) centered at $q$ such that:
\begin{equation}
    \alpha\big|_U = \sum_{i=1}^n dx_i\wedge dy_i
\end{equation}
Such charts are called {\bf Darboux charts}
\end{thm}
We can give to the cotangent bundle $T^*M$ a natural symplectic structure.

\begin{defn}[Tautological 1-form] 
Let $M$ a smooth manifold and $T^*M$ its cotangent bundle. Consider the natural projection $\pi:T^*M\to M$, $p_q\mapsto \pi(p_q) = q$, the pullback $\theta :=\pi^*:T^*M\to T^*T^*M$ defines a 1-form on $T^*M$ called the {\bf tautological 1-form}. 

For $p\in T^*M$, $\theta_p$ can be defined through its action on a vector $v\in T_pT^*M$:
\begin{equation}
    \theta_p(v) = [\pi^*p]v = p\lp \dd\pi_p\lp v\rp\rp \quad \forall v\in T_pT^*M
\end{equation}
\end{defn}
The tautological 1-form admits a simple expression in coordinates. 
Consider $q\in M$ and a smooth chart $(U;x_i)$ centered at $q$. The differentials $\lp dx_i|_q\rp_{i\in[n]}$ form a basis for $T_q^*M$. This means that if we have $p_q\in T_q^*M$ there exist coefficients $\xi_1,\cdots,\xi_n$ such that $p_q = \sum_{i=1}^n \xi_i(dx_i|_q)$. This gives us a chart $(T^*U;x_i,\xi_i)$ for the cotangent bundle, called {\bf cotangent coordinates}. With respect to these coordinates the tautological 1-form has the expression:
\begin{equation}
    \theta\big|_{T^*U} = \sum_{i=1}^n \xi_i\ dx_i
\end{equation}
\begin{defn}[Canonical symplectic form]
Let $M$ be a smooth manifold. Then the cotangent bundle $T^*M$ is a symplectic manifold with symplectic form $\alpha :=-d\theta$, where $\theta$ is the tautological 1-form. $\alpha$ is called the {\bf canonical symplectic form}.
\end{defn}
The expression of $\alpha$ in cotangent coordinates is:
\begin{equation}
    \alpha\big|_{T^*U} = \sum_{i=1}^n dx_i\wedge d\xi_i.
\end{equation}
 This shows that the cotangent coordinates give a Darboux chart for $T^*M$.
 \begin{prop}\label{prop:lift-symplectic}
 Let $M$ be a smooth manifold and $\Phi:M\to M$ a diffeomorphism. Consider the lift $\comp{\Phi}: T^*M\to T^*M$ to the cotangent bundle, then $\lp\comp{\Phi}\rp^*\theta = \theta$ and $\lp \comp{\Phi}\rp^*\alpha = \alpha$. This means that $ \comp{\Phi}$ preserves the tautological 1-form $\theta$ and the canonical symplectic from $\alpha$, and therefore is a symplectomorphism. 
 \end{prop}
 \begin{proof}
 Let's first show that $\forall p\in T^*M$ it holds $\lp\lp\comp{\Phi}\rp^* \theta\rp_p = \theta_p$.
 Given $v\in T_p T^*M$, using the definition of pullback and of the tautological 1-form
 $$\lp\lp\comp{\Phi}\rp^* \theta\rp_p(v) = 
 \theta_{\comp{\Phi}(p)}\lp \dd\lp\comp{\Phi}\rp_p v\rp = 
 \comp{\Phi}(p)\lp \dd\pi_{\comp{\Phi}(p)}\circ \dd\lp\comp{\Phi}\rp_p v\rp
 $$
 Now, since $\comp{\Phi}$ is a lift of $\Phi$ we have the following commutative diagram:
 \[\begin{tikzcd}
 T^*M \arrow{r}{\comp{ \Phi}} \arrow[swap]{d}{\pi} & T^*M\arrow{d}{\pi} \\
 M \arrow{r}{ \Phi} & M
\end{tikzcd}
\]
Therefore $\dd \Phi\circ \dd\pi =d\pi\circ \dd\lp\comp{\Phi}\rp$ and:
\begin{align}
& \comp{\Phi}(p)\lp \dd\pi_{\comp{\Phi}(p)}\circ \dd\lp\comp{\Phi}\rp_p\, v\rp = 
\comp{\Phi}(p)\lp \dd \Phi_{\pi(p)}\circ \dd\pi_p\, v\rp = \\
& p \lp \dd \lp \Phi^{-1}\rp_{\Phi(\pi(p))} \circ\dd \Phi_{\pi(p)}\circ \dd\pi_p\, v\rp = p\lp \dd\pi_p v\rp = \theta_p(v)   
\end{align}
For the second part of the proof, it is sufficient to use first the fact that the exterior derivative $d$ commutes with pullbacks: $$\lp\comp{\Phi}\rp^*\alpha = \lp\comp{\Phi}\rp^*(-d\theta) = -d(\lp\comp{\Phi}\rp^*\theta)= -d\theta=\alpha$$
 \end{proof}
 
 \begin{defn}[Hamiltonian vector field]
 Let $(M,\alpha)$ a symplectic manifold and $H:M\to\bbR$ a smooth function. Since $\alpha$ is degenerate, there exists unique a vector field $X_H$ such that $X_H \lrcorner\, \alpha = dH$. The vector field $X_H$ is called the {\bf Hamiltonian vector field} with {\bf Hamiltonian} function $H$.
 \end{defn}
 \begin{prop}
 Let $X_H$ a complete Hamiltonian vector field and $\phi_{X_H}^t$ the global flow that it defines. We have that the flow preserves the symplectic form, i.e. $\lp\phi_{X_H}^t\rp^*\alpha = \alpha,\quad \forall t\in\bbR $
 \end{prop}
 A vector field whose flow preserves the symplectic form is called {\bf symplectic vector field}. The previous Proposition then tells us that every Hamiltonian vector field is symplectic. The converse is not always true:
 \begin{prop}
 A vector field $X$ on a symplectic manifold $(M,\alpha)$ is Hamiltonian if and only if $X\lrcorner\, \alpha$ is exact.
 \end{prop}
\section{Cotangent lift}
Let us now consider again a complete vector field $X$ and its flow $\{\phi^t_X\}_t$. We can then lift each diffeomorphism $\phi_X^t$ a to symplectomorphism $\comp{\lp\phi_X^t\rp}: T^*M\to T^*M$ on the cotangent bundle, considered a symplectic manifold with the canonical symplectic form. Using the properties of the pullback and the fact that $\phi$ is a global flow, it can be easily shown that the maps $\{\comp{\lp\phi_X^t\rp}\}_{t\in\bbR}$ define a smooth global flow on the cotangent bundle $T^*M$. 
We then have that the global flow is induced by a complete vector field $\comp{X}$. This field is called the {\bf cotangent lift} of the vector field $X$. We therefore have:
\begin{equation}
    \phi^t_{\comp{X}} = \comp{\lp\phi_X^t\rp} \quad \forall t \in \bbR
\end{equation}
This means that given the cotangent vector $p_q\in T_q^*M$, to compute the pullback of $p_q$ by $\phi_{X}^t$ we can solve the Cauchy problem defined by $-\comp{X}$ with starting point $p_q$.

Since by construction the flow of the cotangent lift is given by a family of symplectomorphisms, then $\comp{X}$ is a symplectic vector field.

From Chapter \ref{ch:fields} we know that to compute the cotangent lift at every point, we need to differentiate its flow: 
    \begin{align}
        \comp{X}_{p_q} = \comp{X}_{p_q} = \frac{d}{dt}\biggr|_{t=0}\lp\comp{\lp\phi^t_{X}\rp}p_q\rp_{\phi^t_{X}\lp q\rp} =
        \frac{d}{dt}\biggr|_{t=0}\lp\lp\phi^t_{-X}\rp^*p_q\rp_{\phi^t_{X}\lp q\rp},\quad \forall p_q\in M
    \end{align}
    However this expression it is of little practical use. A global expression for the cotangent lift can be obtained using the symplectic structure of the cotangent bundle. The next theorem proves that $\comp{X}$ is a Hamiltonian vector field:
    \begin{thm}
    Let $X$ a vector field on a smooth manifold $M$ and $\comp{X}$ its cotangent lift on $T^*M$. The vector field $\comp{X}$ is then a Hamiltonian vector field with respect to the canonical symplectic form on $T^*M$. The corresponding Hamiltonian function $H_X$ is:
    \begin{align}
        H_X: T^*M & \to \bbR \\
        p_q&\mapsto p\lp X_q\rp
    \end{align}
    \end{thm}
    \begin{proof}
    We need to prove that $\comp{X}\lrcorner\, \alpha$ is exact. By the definition of $\alpha$ and by Cartan's magic formula we have 
    $$\comp{X}\lrcorner\, \alpha = - \comp{X}\lrcorner\, d \theta =  d\lp\comp{X}\lrcorner\,  \theta\rp + \mathcal{L}_{\comp{X}}\theta = d\lp\comp{X}\lrcorner\,  \theta \rp$$
    Where in the last passage we have used that since $\phi_{\comp{X}} = \comp{\lp\phi_X\rp}$, then by Proposition \ref{prop:lift-symplectic} $\mathcal{L}_{\comp{X}}\theta = 0$.

    The Hamiltonian $H_X$ is therefore $\comp{X}\lrcorner\,  \theta = \theta\lp\comp{X}\rp$. Evaluating the expression at a point $p_q\in T^*M$ we have $$H_X(p_q) = \theta_p\lp\lp\comp{X}\rp_p\rp = p\lp \dd\pi_p \lp\lp\comp{X}\rp_p\rp\rp = p(X_q)$$. 
    \end{proof}

\subsection{Using local coordinates}
Now that we know that the cotangent lift is a Hamiltonian vector field, we can compute an explicit expression for it. 
Let's begin by calculating an expression for $\comp{X}$ on a local chart $(U; x_i)$. In this chart the vector field will have local expression $X|_U = \sum_{i=1}^nf_i\partial_{x_i}$ where $f_i \in C^{\infty}(U)$.
Since $\comp{X}$ is a vector field on $T^*M$ we can find its components with respect to the frame $\{\partial_{x_i}, \partial_{\xi_i}\}_{i=1}^n$ adapted to its cotangent coordinates $(T^*U; x_i, \xi_i)$. Since the cotangent lift is Hamiltonian,  we can leverage the fact that in Darboux coordinates a Hamiltonian vector field $Y$ with Hamiltonian $H$ has local expression (Hamilton equations):
\begin{align}
    Y_H\big|_{T^*U} = \sum_{i=1}\lp\dfrac{\partial H}{\partial \xi_i} \partial_{x_i} - \dfrac{\partial H}{\partial x_i} \partial_{\xi_i}\rp
\end{align}
In our specific case we have that $H_X$ can be written as:
\begin{align}
    H_X\big|_{T^*U} = \sum_{i=1}^n \xi_i dx_i\lp\sum_{j=1}^nf_j\partial_{x_j}\rp = 
    \sum_{i=1}^n \xi_i f_i
\end{align}
Therefore:
\begin{align}\label{eq:cotangent-lift-coord}
    \comp{X}\big|_{T^*U} = \sum_{i=1}^nf_i\partial_{x_i} -
    \sum_{i=1}^n\lp\sum_{j=1}^n \dfrac{\partial f_i}{\partial x_j}\xi_j\rp\partial_{\xi_i}
\end{align}
Notice that as we expected the expression is linear on the components $\partial_{\xi_i}$ and coincides with $X$ if projected on the components $\partial_{x_i}$. Moreover, this expression is the same as the adjoint equation in \cite{neuralode}. Therefore, for $M=\bbR^n$, and in local charts, the cotangent lift coincides with the adjoint equation.

\subsection{Using a generating set}
Suppose now we are parameterizing vector fields on $M$ using a generating set $\{X_i\}_{i=1}^m \subset \smoothsec{M}{TM}$: 
\begin{align}
    X = \sum_{i=1}^mf_iX_i \quad f_i\in C^\infty(M)
\end{align}
In this case, the cotangent lift can be constructed from the lifts of the elements of the generating set $\{\comp{X_i}\}_{i=1}^m$. First, we rewrite the Hamiltonian of the vector field as a linear combination of the Hamiltonians of the fields in the generating set:
\begin{align*}
    &H_X(p_q) = p(X_q) = p\lp\lp\sum_{i=1}^m f_iX_i\rp_q\rp  
    \sum_{i=1}^mf_i(q)p\lp \lp X_i\rp_q\rp = \\ 
    & = \sum_{i=1}^mf_i(q)H_{X_i}(p_q) \quad \forall p\in T_q^*M,\, \forall q\in M
\end{align*}
therefore:
\begin{align*}
    H_X = \sum_{i=1}^m\pi^*(f_i)H_{X_i}\ \ \text{and} \ \ d H_{X} = \sum_{i=1}^m\pi^*(d f_i)H_{X_i} + \sum_{i=1}^m\pi^*(f_i)d H_{X_i}
\end{align*}

The cotangent lift will then be:
\begin{align}\label{eq:lift-generators}
    \comp{X} = \sum_{i=1}^n\widetilde{\alpha}^{-1}\lp \pi^* df_i \rp H_{X_i} + \sum_{i=1}^m f_i\comp{X_i}
\end{align}
Where the vector fields $\widetilde{\alpha}^{-1}\lp \pi^* d f_i \rp$ are obtained first pulling back the 1-form $d f_i$ on $M$ to a 1-form on $T^*M$ via the pullback of the standard projection,and then mapped to a vector field on $T^*M$ via the bundle isomorphism $\widetilde\alpha$. 
A simple calculation in cotangent coordinates shows that the vector fields $\widetilde{\alpha}^{-1}\lp \pi^* d f_i \rp$ belong to the vertical bundle $VT^*M = \{v\in TT^*M | v\in ker( \dd \pi), \ \pi: T^*M\to M\}$ of the cotangent space:
\begin{align*}
    & \pi^*df_i\big|_{T^*U} = \sum_{j=1}^n \frac{\partial f}{\partial x_j}dx_j \\
     \Rightarrow\ & \widetilde{\alpha}^{-1}\lp \pi^* d f_i \rp\big|_{T^*U} = - \sum_{j=1}^n \frac{\partial f}{\partial x_j}\partial_{\xi_j}
\end{align*}
\subsection{Using a local frame}
Consider a local frame of smooth vector fields $\{E_i\}_{i=1}^n\subset \smoothsec{M}{TM}$ over an open domain $U\subseteq M$, and its dual coframe $\{\eta_i\}_{i=1}^n\subset \smoothsec{M}{T^*M}$. From these frames we can define a local frame for the cotangent bundle over the domain $T^*U$\footnote{See \cite{alekseevsky1994poisson}}:
\begin{prop}
Let $M$ be a smooth $n$-dimensional manifold, and $\{E_i\}_{i=1}^n$ a local frame over the open domain $U\subseteq M$ and $\{\eta_i\}_{i=1}^n$ its dual coframe.  Then $Z_1, \cdots, Z_n, \comp{E_1}, \cdots \comp{E_n}$ is a local frame for $T^*M$ over the open domain $T^*U$.
Where $\comp{E_i} = \widetilde{\alpha}^{-1}(dH_{X_i})\ \forall i$ are the complete lifts of $E_i$ and $Z_i := \widetilde{\alpha}^{-1}(\pi^*\eta_i)\ \forall i$ are vertical vector fields.
\end{prop}
\begin{proof}
Fix $p\in T^*U$ with base point $q:=\pi(p)\in U$, then for the first part is sufficient to prove that the fields $\{Z_i, \comp{E_i}\}_i$ span all $T_pT^*U$. This can be done in local coordinates. Let $(V, x_i)$ a local chart where $q\in V\subseteq U$. In local cotangent coordinates we can write:
\begin{align*}
&E_i\big|_{T^*V} = \sum_{i=1}^n a_{ij}\partial_{x_i}\quad \forall i=1\cdots n\\
& \eta_i\big|_{T^*V} = \sum_{i=1}^n b_{ij}d{x_i} \quad  \forall i=1\cdots n
\end{align*}
Since the $\{E_i\}_{i=1}^n$ give local frame  the for every point $(x_1, \cdots, x_n)$ the matrix $A(x_1, \cdots, x_n)$ and $B(x_1, \cdots, x_n)$ defined as $A(x_1, \cdots, x_n)_{ij} := a_{ij}(x_1, \cdots, x_n)$, $B(x_1, \cdots, x_n)_{ij} := b_{ij}(x_1, \cdots, x_n)$ are invertible one and one is the inverse of the other. 
Using Equation \eqref{eq:cotangent-lift-coord} we have:
\begin{align}
    \comp{E_i}\big|_{T^*V} = \sum_{j=1}^na_{ij}\partial_{x_j} -
    \sum_{j=1}^n\lp\sum_{k=1}^n \dfrac{\partial a_{ij}}{\partial x_k}\xi_k\rp\partial_{\xi_j}
\end{align}
And:
\begin{align}
Z_i = -\sum_{j=1}^nb_{ij}\partial_{\xi_j}
\end{align}
Therefore if we define $A:=A(x_1(q),\cdots,x_n(q))$ and  $B:=B(x_1(q),\cdots,x_n(q))$ we have:
\begin{align}
\begin{pmatrix}
A& * \\
0 & B 
\end{pmatrix}
\begin{pmatrix}
\partial_{x_i}|_p\\
\partial_{\xi_i}|_p\\
\end{pmatrix} 
=
\begin{pmatrix}
\lp\comp{E_i}\rp_p \\
(Z_i)_p \\
\end{pmatrix} 
\end{align}
Since both $A$ and $B$ are invertible, the above bock matrix is invertible, therfore $\{\lp Z_i\rp_p, \lp\comp{E_i}\rp_p\}_i$ span all $T_pT^*U$. 
The fact that the fields $Z_i$ are vertical follows immediately form their expression in local coordinates. 
\end{proof}
Now suppose we have a vector field $X\in \smoothsec{M}{TM}$, with local expression $X|_U= \sum_{i=1}^nf_iE_i$, we can find the local expression for $\comp{X}$ with respect to the local frame $\{ Z_i, \comp{E_i}\}_i$. First, since the fields $\{E_i\}_i$ form locally a generating set, we can apply Equation \eqref{eq:cotangent-lift-coord}:
\begin{align}
\comp{X}\big|_{T^*U} =\sum_{i=1}^n \widetilde\alpha^{-1}\lp\pi^*df_i\rp + \sum_{i=1}^n \pi^*f_i\ \comp{E_i}
\end{align}
We can then locally rewrite $df_i|_U = \sum_{j=1}^n E_j(f_j)\eta_i$, combining this expression with the definition of the fields $Z_i$ we have:
\begin{equation}\label{eq:cotg-lift-local-frame}
\comp{X}\big|_{T^*U}= \sum_{j=1}^n\lp\sum_{i=1}^nE_j(f_i)H_{E_i}\rp Z_j + \sum_{i=1}^n \pi^*f_i\ \comp{E_i}
\end{equation}
Where $H_{E_i}(p_q) = p\lp \lp E_i\rp_q \rp$ is just $i$-th coordinate of $p_q$ with respect to the basis $\{\eta_i|_q\}_{i=1}^m$
\subsection{Parallelizable manifolds}
If a manifold $M$ is parallelizable, its tangent and cotangent bundle can be parameterized using $M\times \bbR^n$. We can find an explicit parameterization using a global frame $\{E_i\}_{i=1}^n\subset \smoothsec{M}{TM}$ and its dual global coframe $\{\eta_i\}_{i=1}^n\subset \smoothsec{M}{T^*M}$:
\begin{align}
    M\times \bbR^n &\to T^*M\\
    (q, \beta) &\mapsto \sum_{i = 1}^n \beta_i\eta_i\big|_q
\end{align}
With this parameterization the vertical vector fields $Z_i$ become the coordinate fields:
\begin{align}
    Z_i = - \partial_{\beta_i}
\end{align}
We can then rewrite the cotangent lifts $\{\comp{E_i}\}_{i=1}^n$ with respect to the global frame $\{E_i, \partial_{\beta_i}\}_{i=1}^n$ as:
\begin{align}\label{eq:cotg-lift-hfun}
    \comp{E_i} = E_i + \sum_{j=1}h_{ij}\partial_{\beta_i}
\end{align}
Where $h_{ij}\in C^\infty(M\times \bbR^n)$ is linear in the second argument. 
Equation \eqref{eq:cotg-lift-local-frame} can be now rewritten in the following form:
\begin{align}
    \comp{X} = - \sum_{j=1}^n\lp\sum_{i=1}^nE_j(f_i)\beta_i\rp \partial_{\beta_j} + \sum_{i=1}^n f_i\ \comp{E_i} = \\
    = \sum_{j=1}^n\lp\sum_{i=1}^n f_ih_{ij}-E_j(f_i)\beta_i\rp \partial_{\beta_j} + \sum_{i=1}^n f_i E_i
\end{align}
From the last expression we observe that for parallelizable manifolds, finding the cotangent lift corresponds with augmenting the initial vector field with a linear vector field on $\bbR^n$
\subsection{Cotangent lift on Lie Groups}
Let $G$ be a Lie group with Lie algebra $\mathfrak{g}$. Since $G$ is parallelizable we can identify $T^*G$ with $G\times \alg{g}^*$ via the following isomorphism:
\begin{align}\label{eq:cotg-lie-alg}
    G\times\alg{g}^* &\to T^*G\\
    (a, \psi)&\mapsto L_{a\inv}^* \psi
\end{align}
\begin{thm}\footnote{See Proposition 1.3 in \cite{optimialityhomogspace}.}\label{thm:cotg-lift-lie}
Let $G$ be a Lie group with Lie algebra $\mathfrak{g}$. If we identify, as described above, $T^*G \cong G\times \alg{g}^*$ then given $v\in \alg{g}$ the cotangent lift of the associated left-invariant vector
field is:\begin{align}
\comp{\lp v^L\rp} = \lp v^L, \ad^*_v \rp
\end{align}
Where $\ad^*_v$ is the coadjoint action, defined as the transpose of the adjoint action $\ad_v: \alg{g}\to \alg{g}, w\mapsto [v, w]$, and since $\alg{g}^*$ is a vector space we identified $T_\psi \alg{g}^*$ with $\alg{g}^*$ for all $\psi\in \alg{g}^*$.
\begin{proof}
Let's first prove that the flow of the cotangent lift is given by:
\begin{align}
    \comp{\lp\phi_{v^L}^t\rp} = \lp R_{\exp(vt)},\ \Ad_{\exp(vt)}^*\rp
\end{align}
We already know from  Equation \eqref{eq:flow-left-lie}  that a left invariant vector field $v^L$ has flow $\phi_{v^L}^t = R_{\exp(vt)}$.
Therefore his cotangent lift is a diffeomorphism on $T^*G$ is given by $R_{\exp(-vt)}^*$. Now identifying $T^*G$ via the map in Equation \eqref{eq:cotg-lie-alg} we have:

\begin{align}
\comp{\lp\phi_{v^L}^t\rp}\lp a, \psi\rp &= \lp R_{\exp(vt)} a,  L^*_{a\exp(vt)}\circ R_{\exp(-vt)}^*\circ L^*_{a\inv}\psi\rp =\\ 
&=\lp R_{\exp(vt)}a,  L^*_{\exp(vt)}\circ R_{\exp(-vt)}^*\psi\rp = \\&= \lp R_{\exp(vt)}a,\ \Ad_{\exp(vt)}^*\psi\rp    
\end{align}
Where we have used that the left and the right multiplications commute. The thesis now follows taking the limit for $t\to 0$.
\end{proof}
\end{thm}
Let now $\{v_i\}_{i=1}^n\subset \alg{g}$ be a basis for the lie algebra, and $\{\eta_i\}_{i=1}^n\subset \alg{g}^*$ its dual basis. The basis defines a global frame of left invariant vector fields $\{v_i^L\}_{i=1}^n$. Let $V = \sum_{i=1}^n f_i v_i^L$ a vector field on $G$ parameterized by the smooth function $f:G\to \bbR^n$.  in order to find an expression for its cotangent lift, we need to compute the functions $\{h_{ij}\}_{i,j=1}^n$ present in Equation \eqref{eq:cotg-lift-hfun}. From Theorem \ref{thm:cotg-lift-lie} we then have that
\begin{align}
   h_{ij}(a, \beta) = \ad^*_{v_i}\lp \sum_{k=1}^n\beta_i \eta_i\rp \lp v_j\rp = \sum_{k=1}^n c_{ij}^k \beta_k
\end{align}
Where $c_{ij}^k$ are the structure constants of for $\{v_i\}_{i=1}^n$ ($[v_i,v_j] = \sum_{k=1}^n c_{ij}^k v_k$).
The cotangent lift of $V$ is then:
\begin{align}
    \comp{V} =
    \sum_{j=1}^n\lp\sum_{i=1}^n \sum_{k=1}^n f_i c_{ij}^k \beta_k -\sum_{i=1}^nE_j(f_i)\beta_i\rp \partial_{\beta_j} + \sum_{i=1}^n f_i E_i
\end{align}
Which can be suggestively rewritten as:
\begin{align}
     \comp{V} =
    \ad^*_{\sum_{i=1}f_iv_i} - \sum_{j=1}^n\lp\sum_{i=1}^n v_j(f_i)\beta_i\rp \partial_{\beta_j} + \sum_{i=1}^n f_i E_i   
\end{align}
\subsection{Cotangent lift on embedded submanifolds} \label{sec:emb-submanifolds}
Let $M$ be a properly embedded smooth submanifold of $\bbR^m$, and let $\iota: M\hookrightarrow\bbR^m$ denote the inclusion map. Consider a smooth vector field $X\in \smoothsec{M}{TM}$ and $\overline{X}\in \smoothsec{\bbR^m}{T\bbR^m}$ a tangent vector field that extends $X$. 
We first observe that since the $\overline{X}$ and $X$ are $\iota$-related their flows commute with the inclusion map. We, therefore, have that the following diagram commutes:

\begin{equation}
\adjustbox{scale=1.5}{
\begin{tikzcd}[column sep=large]
 \bbR^m \arrow{r}{\phi^t_{\overline{X}}} \arrow[d, hookleftarrow,"\iota"]& \bbR^m\arrow[d, hookleftarrow,"\iota"] \\
 M \arrow{r}{\phi^t_{X}} & M
\end{tikzcd}
}
\end{equation}
Suppose now that we are interested in computing the differential of the function $f\circ \iota \circ \phi_X^t = f\circ \phi_X^t: M\to \bbR$ where $f:\bbR^m\to\bbR$. We can then both write:
\begin{align}\label{diag:embedding-commute}
d\lp f\circ \iota \circ \phi_X^t\rp &= \lp\phi_X^t\rp^* \circ \iota^*\circ df = \phi_{\comp{X}}^t \circ \iota^*\circ df \\
&= \iota^*\circ \lp\phi_{\overline{X}}^t\rp^* \circ df = \iota^*\circ \phi_{\comp{\overline{X}}}^t\circ df
\end{align}
This means that we can use the cotangent lift of $\overline{X}$ to compute the pullback of cotangent vectors by the flow of $X$. 

Consider now the pullback bundle $\iota^* T^*\bbR^m := \{(x, \beta)\in M\times T^*\bbR^m | \pi(\beta) = x \}$. Since $M$ is a properly embedded submanifold of $\bbR^m$, we can consider the pullback bundle $\iota^* T^*\bbR^m = \{(x, \beta) \in T^*\bbR^m | \pi(\beta)=x,\ x\in M \}=: T^*\bbR^m|_M$, as a properly embedded submanifold of $T^*\bbR^m$ with inclusion map $(\iota, \text{id})$. Moreover, since the fiber of $T^*\bbR^m$ is canonically
isomorphic to $\bbR^m$, we have that $T^*\bbR^m|_M$ is canonically isomorphic (as a vector bundle) to $M\times \bbR^m$. 

Let's analyze now the cotangent lift $\comp{\overline{X}}$. Since $\overline{X}$ is tangent to $M$, it's easy to show that $\comp{\overline{X}}$ is tangent to $T^*\bbR^m|_M$.
\begin{prop}
Let $M$ be a properly embedded smooth submanifold of $\bbR^m$, and let $\iota: M\hookrightarrow\bbR^m$ denote the inclusion map. Consider a smooth vector field $X\in \smoothsec{M}{TM}$ and $\overline{X}\in \smoothsec{\bbR^m}{T\bbR^m}$ a tangent vector field that extends $X$. Then the cotangent lift is tangent to the restriction bundle $T^*\bbR^m|_M = \iota^* T^*\bbR^m = \{(x, \beta) \in T^*\bbR^m | \pi(\beta)=x,\ x\in M \}$
\begin{proof}
Consider $(x,\beta)\in T^*\bbR^m|_M \subset T^*\bbR^m$. Since $\lp T^*\bbR^m|_M\rp_x = \lp T^*\bbR^m\rp_x$ which means that 
the fiber of $T^*\bbR^m|_M$ and $T^*\bbR^m$ are the same, the the only condition that $\comp{\overline{X}}$ has to satisfy for being tangent to $T^*\bbR^m|_M$ $(x,\beta)$ is:
\begin{align}
\pi_{T\bbR^m}\lp \lp\comp{\overline{X}}\rp_{(x,\beta)}\rp =  d \pi_{\bbR^m}  \lp \lp\comp{\overline{X}}\rp_{(x,\beta)}\rp  \in T_x M\subset T_x\bbR^m
\end{align}
Where $\pi_{\bbR^m}:T^*\bbR^m \to \bbR^m$ is cotangent bundle projection. 
Using the definition of cotangent lift we see that:
\begin{align}
    \pi_{T\bbR^m}\lp \lp\comp{\overline{X}}\rp_{(x,\beta)}\rp =  \overline{X}_x \in T_x M\subset T_x\bbR^m.
\end{align}
Where the equality holds because $\overline{X}$ is tangent to $M$. 
\end{proof}
\end{prop}
The restriction of $\comp{\overline{X}}$ therefore defines the smooth vector field $\comp{\overline{X}}|_{T^*\bbR^m|_M}$ on  $T^*\bbR^m|_M$. Then, by taking the pullback of every map in the commutative diagram \eqref{diag:embedding-commute}, we have that the following diagram commutes:
\begin{equation}
\adjustbox{scale=1.5}{
\begin{tikzcd}[column sep=3cm]
 T^*\bbR^m|_M \arrow{r}{\phi^t_{\comp{\overline{X}}|_{T^*\bbR^m|_M}}} \arrow[d, rightarrow,"\iota^*"]& T^*\bbR^m|_M\arrow[d, rightarrow,"\iota^*"] \\
 T^*M \arrow{r}{\phi^t_{\comp{X}}} & T^*M
\end{tikzcd}
}
\end{equation}
Then by Proposition 9.6 of \cite{lee2013smooth} $\comp{\overline{X}}|_{T^*\bbR^m|_M}$ and $\comp{X}$ are $\iota^*$ related. Moreover by Lemma \ref{lem:pullback-emb-surj}
$\iota^*T^*\bbR^m = T^*\bbR^m|_M$ parameterizes all $T^*M$. 
As previously observed, when working with embedded manifolds functions on $M$ are parameterized by $f\circ \iota$, where $f:\bbR^m\to \bbR$.In this case, the above diagram tells us that in order to compute the pullback through $\phi^t_X$ we can solve a ODE on $\iota^*T^*\bbR^m = T^*\bbR^m|_M \cong M\times\bbR^m$.

\chapter{Experiments}

In this chapter we show that the framework derived in the previous chapters can be effectively used to model complex multimodal probability distributions on non-trivial manifolds. This is achieved by training a manifold continuous normalizing flow (MCNF) to match mixture distributions on a variety of spaces: the hypersphere $\bbS^n$, the matrix Lie groups $SO(n)$, $U(n)$, $SU(n)$, Stiefel manifolds $\mathcal{V}_{m}(\bbR^n)$, and the positive definite symmetric matrices $\text{Sym}^+(n)$. For each manifold we specified a mixture distribution with 4 components, adapted to the specific geometry of the space.

\section{General setup}
\subsection{Density matching on manifolds}
The experiments consist of the following: fixing a smooth manifold $M$ and base smooth positive density $\mu$, we train MCNF to match target density $\mu^\star$: 
    \begin{align*}
        \mu^\star := \dfrac{\rho^\star}{Z}\mu\quad \text{where}\ \ \rho^\star \in C(M, \bbR_{\ge 0}),\  Z:= \int \rho^\star \mu.
    \end{align*}
    Where $\rho^\star$ is a positive continuous function and $Z\in \bbR$ is the normalization constant. Denoted by $\mu_\lambda:= \rho_\lambda\mu$ the density parameterized by the MCNF model\footnote{Where $\rho_\lambda \in C(M, \bbR_{\ge 0})$, $ 1= \int \rho_\lambda\mu$, and $\lambda \in \bbR^k$ denotes the model parameters}, the training objective is to minimize the KL divergence:
    \begin{align}\label{eq:kl-div}
        \KL{\mu_\lambda}{\mu^\star} = \int \lp \ln \rho_\lambda - \ln \rho^\star\rp \rho_\lambda\mu - \ln Z = \\= \E{\mu_\lambda}{\ln \rho_\lambda - \ln \rho^\star} - \ln Z.
    \end{align}
    Where $\ln Z$ is a constant and can therefore be ignored.
    The performance of the model is then assessed using KL divergence\footnote{The  normalization constant $Z$ in the KL divergence is estimated as $Z\approx \sum_{i=1}^S w_i$} and effective sampling size\footnote{\cite{kishess}} (ESS):
    \begin{align}
    ESS \approx \frac{\lp\sum_{i=1}^S w_i \rp^2}{\sum_{i=1}^S w_i^2}   \ \ &\text{where} \ \ w_i = \frac{\rho^\star(q_i)}{\rho_\lambda(q_i)},\quad \text{and  } q_i\sim \rho_\lambda \mu.
    \end{align}

\subsection{Mixture of densities}
Let $\mu$ be a smooth positive density on a smooth manifold. In each of the experiments we fit a $k$ component mixture distribution:
\begin{align}
\mu^\star = \rho^\star(\cdot| \beta, \{W_i\}_{i=1}^k) \mu =
\frac{\sum_{i=1}^k \rho(\cdot|\beta, W_i)}{Z(\beta)}\mu
\end{align}
Where $\rho\in C(M)$ is the component of the distribution, which is given by an unnormalized probability density function with normalization constant $0<Z\in \bbR$. If $Z(\beta)$ does not depends on the parameters $W$ of the components then $\int \rho^\star d\mu = 1$ is a probability density function. The parameter $\beta$ controls the concentration of the component distribution. The log target density is then given by:
\begin{align}
    \log \rho^\star(\cdot| \beta, \{W_i\}_{i=1}^k) = \underset{i\in \range{k}}{\text{logsumexp}} \log\rho(\cdot| \beta, W_i)
\end{align}

In each of the experiments, the function $\rho$ is adapted to the specific geometry of the space.

\subsection{Implementation details}
All the spaces are considered to be embedded in $\bbR^m$ and backpropagation was performed using what is observed in Section \ref{sec:emb-submanifolds}. We parameterize vector fields choosing a suitable generating set of vector fields. Fixed the generating set, the parameterizing function $f\in C^\infty(M; \bbR^{m_{gen}})$ used in Equation \ref{eq:vec-field-gen} is parameterized by a neural network with input in the embedding space $\bbR^m \supset M$, $\textit{tanh}$ nonlinearity, and 2 hidden layers with dimension: $[5\cdot m_{gen}, 5\cdot m_{gen}]$, where $m_{gen}$ is the cardinality of the generating set used. The parameterized flow was numerically integrated with starting time $0$ and final time $1$, using the Dormand-Prince ODE integration algorithm, with adaptive stepsize, as implemented in \citet{jax2018github}.
Each model was optimized using Adam \citep{kingma2014adam} 
with a learning rate of $5\cdot10^{-4}$. Divergence during training was approximated using the estimator described in Theorem \ref{thm:dive-estimator} and sample size of $512$.

Each model was scored on KL divergence and effective sampling size (ESS) using a sample of $200000$ points and the exact divergence expression given by Equation \eqref{eq:dive-gen}.

In the next sections, we will focus individually on each of the spaces we used in the experiments. Describing the type of generating set employed and giving the exact expression of the mixture centres used, as well as reporting the experiments results.

\section{Manifold specific setup and results}

\subsection{Special Orthonormal matrices}
The group $SO(n):=\{A\in \text{GL}_{n}(\bbR) : A^\top A = I,\ \text{det} (A)=1\}$ is the matrix Lie group formed by $n\times n$ orthonormal matrices with unit determinant. 
Its lie algebra $\mathfrak{so}(n):=\{A\in \text{M}_{nn}(\bbR) : A^\top = - A\}$ is given by the $n\times n$ skew-symmetric matrices. 
A basis for $\alg{so}(n)$ is given by the $\frac{n(n-1)}{2}$ matrices :
\begin{align}
    V_{ij} = E_{ij} - E_{ji}\quad i<j\quad  \forall i,j \in \range{n} 
\end{align}
The matrix unit $E_{ij}\in \M_{nm}(\bbR)$ is the matrix with $1$ in the $(i,j)$-entry and $0$ everywhere else.
For example, when $n=3$ a basis $\{v_i\}_{i=1}^3$ for $\mathfrak{so}(3)$ is given by:
\begin{align}\label{eq:basis-so3}
v_{1,2,3} = & 
\begin{bmatrix} 0 & 0 & 0 \\ 0 & 0 & 1 \\ 0 & -1 & 0 \end{bmatrix}, 
\begin{bmatrix} 0 & 0 & 1 \\ 0 & 0 & 0 \\ -1 & 0 & 0 \end{bmatrix}, 
\begin{bmatrix} 0 & 1 & 0 \\ -1 & 0 & 0 \\ 0 & 0 & 0 \end{bmatrix}    
\end{align}

We use the generating set $\{v_i^L\}_{i=1}^3\subset \smoothsec{SO(3)}{TSO(3)}$, when parameterizing vector fields on $SO(3)$ in the experiments. 
\subsection{Stiefel Manifold}
The Stiefel manifold $\cV_m(\bbR^n) = \{Q\in M_{nm}(\bbR)\ |\ Q^\top Q = I_{m}\}$ is the manifold formed by all orthonormal $m$-frames in $\bbR^n$.
As a special case, for $m=1$ we have: $ \mathbb{S}^n = \stief{1}{n+1}$. When $m<n$ the Stiefel manifold can be considered a homogeneous-$SO(n)$ space with action given by the left matrix multiplication: $(A,Q)\mapsto AQ$.

By Proposition, \ref{prop:gen-hom} the basis $\{V_{ij}\}_{i<j} \subset \alg{so(n)}$ induces a
generating set
$\{\widetilde V_{ij}\}_{i<j} \subset \smoothsec{\stief{m}{n}}{T\stief{m}{n}}$  for the smooth vector fields on $\stief{m}{n}$, which we will use in the experiments when parameterizing vector fields on $\stief{m}{n}$:
\begin{align}
    \widetilde V_{kj}(Q) = V_{kj}Q = E_{kj}Q - E_{jk}Q\in T_Q\stief{m}{n}\subset \M_{nm}(\bbR)
    \quad &\forall Q\in \stief{m}{n}
\end{align}
where $1\le k\!<\!j\le n$.

\subsubsection{Experiments:}
The target density used in the experiments for $\stief{m}{n}$ and $SO(n)$ is a $k=4$ mixture of Langevin distribution,\footnote{See Section 3.2 in \cite{CamaoGarcia2006Statistics}} with components given by:
\begin{align}
    \log \rho(Q|\beta, W) = \frac{\beta}{m}\text{tr}\lp W^\top Q\rp
    \quad \forall Q\in \stief{m}{n},SO(n) 
\end{align}
where $W\in \stief{m}{n},SO(n)$. For $SO(n)$ we used $m:= n-1$. 

As a base density $\mu$, and as a initial probability density, we used the $SO(n)$ action invariant density on $\stief{m}{n}$(normalized to 1). Algorithm 1 in Section 4.4.2 of \cite{CamaoGarcia2006Statistics} was employed for sampling. 

In the experiments we fixed $\beta\in\{10\}$ and the centers $W_i\in \stief{m}{n}, SO(n)$ were sampled uniformly at random for each run. Results are reported in Table \ref{tab:experiment-stief}.
\begin{table}[t]
\begin{center}
\caption{Evaluation of Manifold Continuous Normalizing Flow (MCNF) on $\stief{m}{n}$ and $SO(n)$. Error is computed over 5 replicas of each experiment. }
\begin{tabular}{|c|c|c|c|}
    \hline
    {\bf Manifold} & {$\beta$} &{\bf KL[nats]}& {\bf ESS[\%]} \\
    \hline \hline
    $\mathcal{V}_{1}(\bbR^3) \cong \bbS^2$ &10&$0.003_{\pm 0.001}$ & $99.3_{\pm 0.8}$ \\ \hline
    $\mathcal{V}_{1}(\bbR^4) \cong \bbS^3$ &10&$0.006_{\pm 0.002}$ & $98.6_{\pm 0.5}$\\ \hline
    $\mathcal{V}_{2}(\bbR^4)$ &10& $0.02_{\pm 0.006}$ & $96.5_{\pm 1.1}$ \\ \hline
    $\mathcal{V}_{2}(\bbR^5)$ &10&$0.02_{\pm 0.003}$ & $97.1_{\pm 0.6}$\\ \hline
    $SO(3)$ &10&$0.03_{\pm 0.007}$ & $95.3_{\pm 0.8}$\\ \hline
    \end{tabular}
    \label{tab:experiment-stief}
\end{center}
\end{table}

\subsection{Hypersphere}
The hypersphere $\bbR^n = \{x\in \bbR^{n+1}| \|x\|_2 = 1\}$ is the set of points in $\bbR^{n+1}$ with unit norm. 

As we already saw, we can consider the hypersphere $\bbS^{n} = \stief{1}{n+1}$ as a homogeneous $SO(n+1)$-space. And therefore build a generating set of vector fields given by infinitesimal $SO(n+1)$ actions. However, as we saw before,  the cardinality of this generating set is $\frac{n(n+1)}{2}$, number that grows quadratically with the dimension of the manifold, making it inefficient for building normalizing flows on high dimensional hyperspheres. 

Alternatively, we can use the generating set induced by the isometric  embedding $z: \bbS^n \hookrightarrow \bbR^{n+1}$,
given by the  coordinate functions $z_i(x)= x_i$ and described by in Section \ref{app:lap-eigenfunctions}. Fortunately, the embedding components are laplacian eigenfunctions\footnote{Section 1.2 \cite{bates2014embedding}.}. 
The divergence of the gradient fields is then given by a multiple of the embedding functions \footnote{See Section \ref{app:lap-eigenfunctions}}.

This gives us $n+1$ vector fields $\{\overline{\grad( z_i)}\}_{i=1}^{n+1}\subset \smoothsec{\bbR^{n+1}}{T\bbR^{n+1}}$ tangent to $\bbS^n$ such that their restriction to $\bbS^n$ forms a generating set for $\smoothsec{\bbS^{n}}{T\bbS^n}$:
\begin{align}\label{eq:gen-sn_lap}
&\overline{\grad(z_i)}(x) = e_i - \langle x, e_i\rangle x \quad \forall i\in \range{n+1}\\
&\dive{\mu_g}{\grad(z_i)} = \Delta z_i(x) = -nx_i
\end{align}
Where $\mu_g$ is the Riemannian density given by the rotation invariant metric on $\bbS^n$. We have therefore built a generating set with cardinality $n+1$, which scales linearly with the dimension of the space, and with a simple divergence expression. We use this generating set for $\bbS^n$ in the experiments when scaling to higher dimensions. 
\subsubsection{Experiments}
The target density used in the experiments for $\bbS^n$ is a $k=4$ mixture of von Mises Fisher (vMF) distributions, with components given by:
\begin{align}
    \log \rho(Q|\beta, W) = \beta \cdot W^\top Q\quad \forall Q\in \bbS^n
    \qquad \text{where  }W\in \bbS^n
\end{align}

As a base density $\mu$, and as an initial probability density, we used the $SO(n)$ action invariant density (normalized to 1) on $\bbS^n$.

We employed the generating set given by Equation \ref{eq:gen-sn_lap} for parameterizing vector fields on $\bbS^n$. 

In the experiments we fixed $n\in \{10, 20\}$ add $\beta\in \{10, 20, 50\}$ and the centers $W_i\in \bbS^n$ were sampled uniformly at random for each run. Results are reported in Table \ref{tab:experiment-sn}.

\begin{table}[b]
\begin{center}
\caption{Evaluation of Manifold Continuous Normalizing Flow (MCNF) on $\bbS^n$. The experiments on the Highlighted rows were repeated 5 times and reported in the summary Table \ref{tab:experiment-overview} }
\begin{tabular}{|c|c|l|l|}
    \hline
    {\bf Manifold} & {$\beta$} &{\bf KL[nats]}& {\bf ESS[\%]} \\
    \hline \hline
    $\bbS^{10}$ &10&$0.006$ & $98.7$ \\ \hline
    \rowcolor{light-gray}$\bbS^{10}$ &20&$0.01_{\pm 0.003}$ & $97.2_{\pm 0.6}$ \\ \hline
    $\bbS^{10}$ &50&$0.02$ & $95.0$ \\ \hline
    $\bbS^{20}$ &10&$0.01$ & $97.32$ \\ \hline
    \rowcolor{light-gray}$\bbS^{20}$ &20&$0.02_{\pm 0.003}$ & $95.4_{\pm 0.4}$\\  \hline
    \end{tabular}
    \label{tab:experiment-sn}
\end{center}
\end{table}

\subsubsection{Additional experiment: comparing with previous normalizing flows on spheres}
Additionally, we decided to compare our MCNF against the normalizing flow for hyperspheres proposed in \cite{rezende2020normalizing}. We, therefore, trained a MCNF to match the same vMF mixture used in the experiments by \cite{rezende2020normalizing} and compared against the best result provided there.
Results are reported in Table \ref{tab:match-sphere}, while Figure \ref{fig:density-match-vmf} shows the leaned density on $\bbS^2$. We observe that the proposed MCNF model can closely match the target densities, with a considerably lower KL divergence and higher ESS than the model by \citet{rezende2020normalizing}.
\begin{figure}[t]
\centering
\subfigure[Target]{\includegraphics[width=0.45\linewidth]{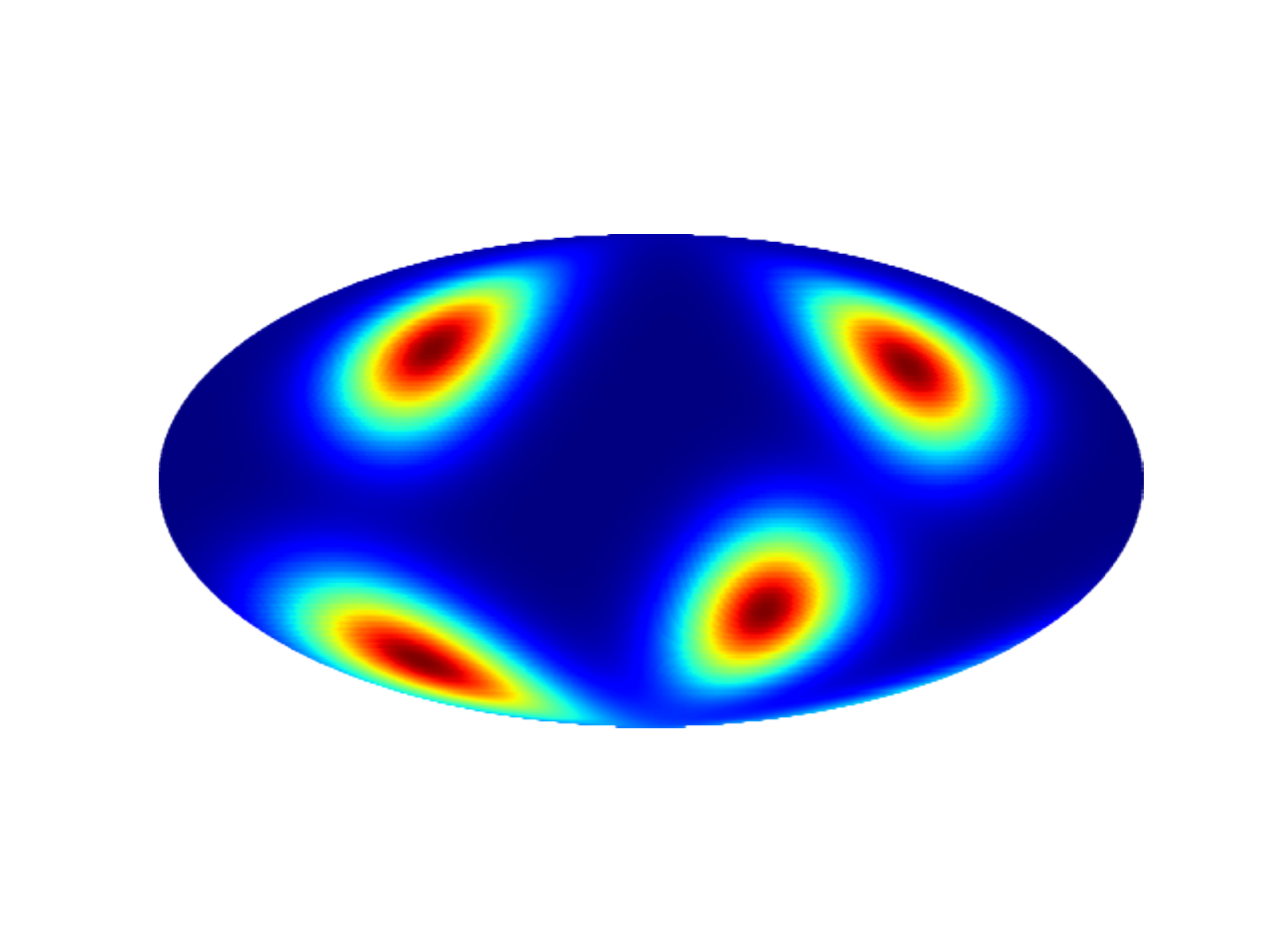}}
\subfigure[Model]{\includegraphics[width=0.45\linewidth]{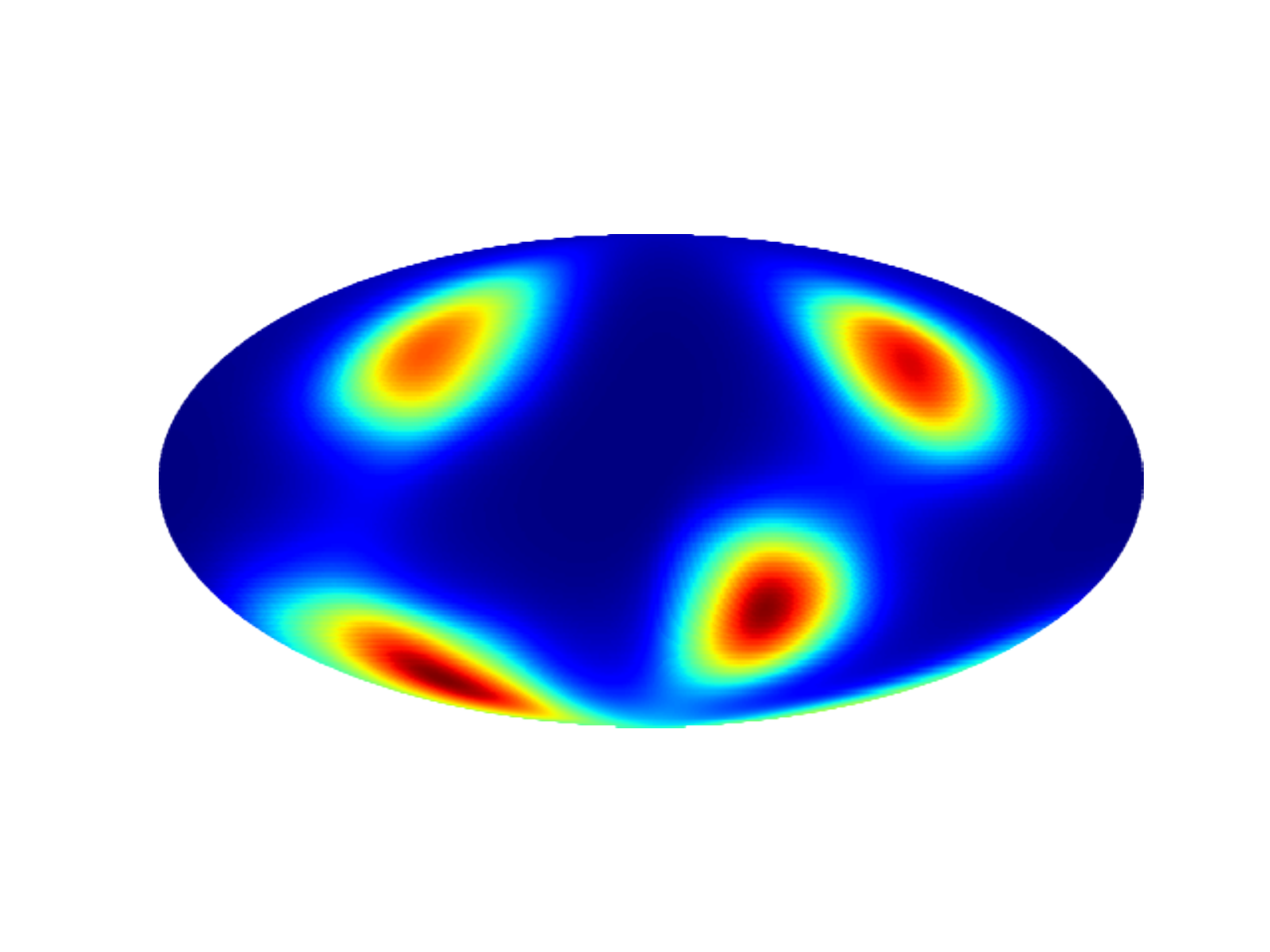}}
\caption{Learned density on $S^2$}
\label{fig:density-match-vmf}
\end{figure}

\begin{table}[h]
\begin{center}
\begin{tabular}{|c|l|l|l|}
    \hline
    {\bf Manifold} & {\bf Model} & {\bf KL[nats]}& {\bf ESS[\%]} \\
    \hline \hline
    \multirow{2}{4em}{$S^2$} & MS& 0.05{\tiny$\pm$.01} & 90\\
                        & MCNF & \textbf{0.003}{\tiny$\pm$.001} & \textbf{99.3}{\tiny$\pm$.2} \\
    \hline \hline
    \multirow{2}{4em}{$S^3$} & MS& 0.14 & 84\\
                        & MCNF & \textbf{0.004}{\tiny$\pm$.001} & \textbf{99.2}{\tiny$\pm$0.2} \\ \hline \hline

\end{tabular}
    \caption{Comparing of MCNF and MS flow proposed in  \cite{rezende2020normalizing} on a vMF density matching task. Error is computed over 3 replicas of each experiment.}
    \label{tab:match-sphere}
\end{center}
\end{table}

\subsection{Unitary matrices} 
The group $U(n):=\{A\in \text{GL}_n(\bbC) : A^\dag A = I\}$ is the matrix Lie group formed by all the complex $n\times n$ unitary matrices. Its Lie algebra $\alg{u(n)} = \{A\in \M_{nn}(\bbC)\ |\ A^\dag = -A\}$ is given by the $n\times n$ skew-Hermitian matrices. A basis for $\alg{u(n)}$ is given by the $n^2$ matrices:
\begin{align}
    V_{kj} := E_{kj} - E_{jk}\quad iV_{kj}^\dag := iE_{kj} + iE_{jk} \quad H_{j} := iE_{jj} \quad \text{where } 1\le k\!<\!j\le n
\end{align}
where  $1\le k\!<\!j\le n$ and $i$ denotes is the imaginary unit. 
For example for $n=2,3$ a basis $\{v_k\}_{k=1}^{n^2}$ is given by:
\begin{align}
    v_{1,2,3,4} = &\begin{bmatrix} 0 & i\\ i & 0 \end{bmatrix}, 
    \begin{bmatrix} 0 & 1 \\ -1 & 0 \end{bmatrix},
    \begin{bmatrix} i & 0 \\ 0 & 0 \end{bmatrix},
    \begin{bmatrix} 0 & 0 \\ 0 & i \end{bmatrix}. \\
    v_{1,2,3,4,5,6,7,8,9} := & 
\begin{bmatrix} 0 & 0 & 0 \\ 0 & 0 & 1 \\ 0 & -1 & 0 \end{bmatrix}, 
\begin{bmatrix} 0 & 0 & 1 \\ 0 & 0 & 0 \\ -1 & 0 & 0 \end{bmatrix}, 
\begin{bmatrix} 0 & 1 & 0 \\ -1 & 0 & 0 \\ 0 & 0 & 0 \end{bmatrix}
\begin{bmatrix} 0 & 0 & 0 \\ 0 & 0 & i \\ 0 & i & 0 \end{bmatrix},\\
&\begin{bmatrix} 0 & 0 & i \\ 0 & 0 & 0 \\ i & 0 & 0 \end{bmatrix}, 
\begin{bmatrix} 0 & i & 0 \\ i & 0 & 0 \\ 0 & 0 & 0 \end{bmatrix},
\begin{bmatrix} i & 0 & 0 \\ 0 & 0 & 0 \\ 0 & 0 & 0 \end{bmatrix},
\begin{bmatrix} 0 & 0 & 0 \\ 0 & i & 0 \\ 0 & 0 & 0 \end{bmatrix},
\begin{bmatrix} 0 & 0 & 0 \\ 0 & 0 & 0 \\ 0 & 0 & i \end{bmatrix},
\end{align}

We use the generating set $\{v_k^L\}_{k=1}^{n^2}\subset \smoothsec{U(n)}{TU(n)}$, when parameterizing vector fields on $U(2)$ and $U(3)$ in the experiments. 
\subsection{Special Unitary matrices}
The group $SU(n):=\{A\in \text{GL}_n(\bbC) : A^\dag A = I,\ \text{det} (U)=1\}$ is the matrix Lie group of complex $n\times n$ unitary matrices with unit determinant. The lie algebra $\mathfrak{su}(n):=\{A\in \text{M}(n, \bbC) : A^\dag = - A,\ \text{tr}(A)= 0\}$ is given by the traceless skew-Hermitian matrices. A basis for $\alg{su(n)}$ is given by the $n^2 - 1$ matrices:
\begin{align}
    V_{kj} = E_{kj} - E_{jk}\quad iV_{kj}^\dag = iE_{kj} + iE_{jk} \quad H_{k} = iE_{kk} - iE_{k+1k+1} 
\end{align}
where $1\le k\!<\!j\le n$.
When $n=2$ a basis $\{v_i\}_{i=1}^3$ for $\mathfrak{su}(2)$ is given by Pauli matrices:
\begin{equation}
 \quad  v_{1,2,3} = \begin{bmatrix} 0 & i\\ i & 0 \end{bmatrix}, 
    \begin{bmatrix} 0 & 1 \\ -1 & 0 \end{bmatrix},
    \begin{bmatrix} -i & 0 \\ 0 & i \end{bmatrix}.
\end{equation}
When $n=3$ a basis $\{v_i\}_{i=1}^8$ for $\mathfrak{su}(3)$ is:
\begin{align}\label{eq:basis-su3}
v_{1,2,3,4,5,6,7,8} = & 
\begin{bmatrix} 0 & 0 & 0 \\ 0 & 0 & 1 \\ 0 & -1 & 0 \end{bmatrix}, 
\begin{bmatrix} 0 & 0 & 1 \\ 0 & 0 & 0 \\ -1 & 0 & 0 \end{bmatrix}, 
\begin{bmatrix} 0 & 1 & 0 \\ -1 & 0 & 0 \\ 0 & 0 & 0 \end{bmatrix}
\begin{bmatrix} 0 & 0 & 0 \\ 0 & 0 & i \\ 0 & i & 0 \end{bmatrix},\\
&\begin{bmatrix} 0 & 0 & i \\ 0 & 0 & 0 \\ i & 0 & 0 \end{bmatrix}, 
\begin{bmatrix} 0 & i & 0 \\ i & 0 & 0 \\ 0 & 0 & 0 \end{bmatrix},
\begin{bmatrix} i & 0 & 0 \\ 0 & -i & 0 \\ 0 & 0 & 0 \end{bmatrix}, 
\begin{bmatrix} 0 & 0 & 0 \\ 0 & i & 0 \\ 0 & 0 & -i \end{bmatrix}.
\end{align}
\subsubsection{Experiments:}
The target density used in the experiments for $U(n)$ and $SU(n)$ is a $k=4$ mixture of distributions with components given by:
\begin{align}
    \log \rho(Q|\beta, W) = \frac{\beta}{n}\text{tr}\lp Real\lp W^\dag Q\rp\rp
    \quad &\forall Q\in U(n),SU(n)\\  &\text{where  }W\in U(n),SU(n)
\end{align}

As a base density $\mu$, and as a initial probability density, we used the bi-invariant density associated to the Haar measure on $SU(n), U(n)$. The algorithm described in \cite{mezzadri2006randomcompact} was employed for sampling. 

In the experiments we fixed $n\in\{2,3\}$,$\beta\in\{5, 10\}$ and the centers $W_i\in U(n), SU(n)$ were sampled uniformly at random for each run. Results are reported in table \ref{tab:experiment-unitary}.
\begin{table}[t]
\begin{center}
\caption{Evaluation of Manifold Continuous Normalizing Flow (MCNF) on $U(n)$ and $SU(n)$. The experiments on the Highlighted rows were repeated 5 times and reported in the summary Table \ref{tab:experiment-overview} }
\begin{tabular}{|c|c|l|l|}
    \hline
    {\bf Manifold} & {$\beta$} &{\bf KL[nats]}& {\bf ESS[\%]} \\
    \hline \hline
    $SU(2)$ &5& $0.001$ & $99.9$ \\ \hline
    \rowcolor{light-gray}$SU(2)$ &10& $0.005_{\pm 0.003}$ & $98.6_{\pm 0.8}$ \\ \hline
    $SU(3)$ &5& $0.003$ & $99.3$ \\ \hline
    \rowcolor{light-gray}$SU(3)$ &10& $0.01_{\pm 0.003}$ & $97.1_{\pm 0.6}$ \\ \hline
    $U(2)$ &5& $0.003$ & $99.3$ \\ \hline
    \rowcolor{light-gray}$U(2)$ &10& $0.02_{\pm 0.008}$ & $95.4_{\pm 1.4}$ \\ \hline
    $U(3)$ &5& $0.006$ & $98.9$ \\ \hline
    \rowcolor{light-gray}$U(3)$ &10& $0.05_{\pm 0.008}$ & $91.5_{\pm 1.3}$ \\ \hline
    \end{tabular}
    \label{tab:experiment-unitary}
\end{center}
\end{table}
\subsubsection{Additional experiment: matching conjugation invariant densities on SU(3)}
To further investigate the ability of MCNF to model multimodal densities, we trained a model to match the following conjugation invariant probability density on $SU(3)$:
\begin{align}
\log \widetilde \rho (A| \beta, c) = \frac{\beta}{3}Real \lp\text{tr}\lp \sum_{j=1}^3 c_jU^j\rp\rp
\end{align}
Where $c = (c_1, c_2, c_3) \in \bbR^3$. Results are reported in Table \ref{tab:experiment-su(3)equiv}. As a reference, we report the scores achieved by the conjugation equivariant flow described in \cite{kanwar2020equivariant}. Notice that a throughout comparison of the models might be inappropriate since the permutation equivariant flow already incorporates the symmetries of the target distribution. Thus reducing to modelling a flow on the two-dimensional simplex $\Delta_2$. Our model instead parameterizes an unconstrained density on $SU(3)$. Notwithstanding the above a MCNF, can match the performance of the conjugation equivariant flow in all settings, achieving in some cases significantly better results. We interpret these results as additional evidence of the high expressive power of our model.
\begin{table}[!t]
\begin{center}
\caption{ESS of Manifold Continuous Normalizing Flow (MCNF) on conjugation invariant density matching task on $SU(3)$. Results are compared with conjugation equivariant flow (CEF) defined by \cite{kanwar2020equivariant}.} 
\begin{tabular}{|c|l|lll|}
    \hline
    {\bf Setup} & Model  & $\beta=1$ & $\beta=5$ & $\beta=9$ \\
    \hline \hline
    \multirow{2}{12em}{$c = (0.17, -0.65, 1.22)$} & CEF& 97 & 80 & 82\\
                        & MCNF & 98.3 & \textbf{91.2} & 83.5\\
    \hline \hline
    \multirow{2}{12em}{$c = (0.98, -0.63, -0.21)$} & CEF& 99 & 91 & 73\\
                        & MCNF & 99.5 & \textbf{98.5} & \textbf{97.4}\\\hline

\end{tabular}
    \label{tab:experiment-su(3)equiv}
\end{center}
\end{table}

\subsection{Positive definite symmetric matrices}
The manifold $\text{Sym}^+(n) = \{Q\in \M_{nn}(\bbR)\ |\ Q^\top = Q,\ x^\top Q x > 0\ \forall x\in \bbR^n\}$ is formed by all $n\times n$ symmetric positive definite matrices. $\text{Sym}^+(n)$ is a convex cone on the manifold of symmetric matrices $\text{Sym}(n) \cong \bbR^{\frac{n(n+1)}{2}}$. 

$\text{Sym}^+(n)$ can be considered a homogeneous $GL(n)$ space with action given by matrix congruence:
\begin{align}
    GL(n)\times \text{Sym}^+(n) &\to \text{Sym}^+(n)\\
    (A, Q)&\mapsto A^\top Q A 
\end{align}
By Proposition \ref{prop:gen-hom} the basis $\{E_{kj}\}_{k,j=1}^n \subset \alg{gl(n)}$ induces the
generating set $\{\widetilde E_{kj}\}_{k,j=1}^n \subset \smoothsec{\text{Sym}^+(n)}{T\text{Sym}^+(n)}$  for the smooth vector fields on $\text{Sym}^+(n)$:
\begin{align}\label{eq:gen-sym-pos}
    \widetilde E_{kj}(Q) &= \dfrac{d}{dt}\biggr|_{t=0}\bigg[\lp I + E_{jk}t\rp Q\lp I + E_{kj}t\rp + o(t)  \bigg] =  \\ 
    &= E_{jk}Q + QE_{jk}\in T_Q\text{Sym}^+(n)\cong \text{Sym}(n)\subset  \M_{nn}(\bbR)
    \quad \forall Q\in \text{Sym}^+(n)
\end{align}
(where $k,j\in \range{n}$), 
which we will use in the experiments when parameterizing vector fields on $\text{Sym}^+(n)$. Notice that we can identify $T_Q \text{Sym}^+(n) \cong \text{Sym}(n) :=  \{Q\in \M_{nn}(\bbR)\ |\ Q^\top = Q\},\ \forall Q\in \text{Sym}^+(n)$.

We can define on $\text{Sym}^+(n)$ the following group action invariant Riemannian metric\footnote{See \cite{posdefmatrices}.}:
\begin{align}
    g_Q(A,B) = \text{tr}\lp Q^{-1} AQ^{-1}B\rp
\end{align}
The corresponding group invariant volume density $\mu_g$ is given by
\begin{align}\label{eq:def-density-sym}
    \mu_g(Q) = \text{det}(Q)^{-\frac{(n+1)}{2}} \left| \bigwedge_{1\le k\le j \le n} dQ_{kj}\right|\quad \forall Q\in \text{Sym}^+(n)
\end{align}
Where $\left| \bigwedge_{1\le k\le j \le n} dQ_{kj}\right|$ represents the uniform density associated with the Lebesgue measure on $\bbR^{\frac{n(n+1)}{2}}$. What discussed in section \ref{sec:hom-space} then assures us that the generating set described in Equation \eqref{eq:gen-sym-pos} is formed by vector fields with zero divergence with respect to $\mu_g$.
\subsubsection{Experiments:}
The target density used in the experiments for $\text{Sym}^+(n)$ is a $k=4$ mixture of Wishart distributions with components given by:
\begin{align*}
    \widetilde \rho(Q|\beta, W) \left| \bigwedge_{1\le k\le j \le n} dQ_{kj}\right| = 
    \frac{\det(Q)^{\frac{\beta - n - 1}{2}}}{\det(W)^{\frac{\beta}{2}}}\exp\lp-\frac{1}{2}\text{tr}\lp  W^{-1} Q\rp\rp \left| \bigwedge_{1\le k\le j \le n} dQ_{kj}\right|
    \\ \forall Q\in \text{Sym}^+(n)
\end{align*}
Where $n+1\le \beta \in \bbN$ corresponds to the degrees of freedom, and $W \in \text{Sym}^+(n)$. 
The above probability density function is given with respect to the density associated with the Lebesgue measure in $\bbR^{\frac{n(n+1)}{2}}$. 
If we use as a base density the congruence invariant density $\mu_g$ defined by Equation \eqref{eq:def-density-sym} we can rewrite:
\begin{align}
    \rho(Q|\beta, W) \mu_g(Q) = 
    \lp\frac{\det(Q)}{\det(W)}\rp^{\frac{\beta}{2}}\exp\lp-\frac{1}{2}\text{tr}\lp  W^{-1} Q\rp\rp \mu_g(Q)
    \\ \forall Q\in \text{Sym}^+(n)
\end{align}
Algorithm 1 in Section 4.4.2 of \cite{CamaoGarcia2006Statistics} was employed for sampling. 
As initial probability density, we used $\rho(\cdot|\beta, I\frac{5}{\beta}) \mu_g$. In the experiments, we fixed $\beta\in\{5, 10, 20\}$, $n\in \{2, 3\}$ with centers:
\begin{align}
    &W_{1,2,3,4} := \frac{1}{\beta}\widehat W_{1,2,3,4} \quad \text{where}\\
    &\widehat W_{1,2,3,4}:= \begin{bmatrix} 1 & 0\\ 0 & 2 \end{bmatrix}, 
    \begin{bmatrix} 2 & 0 \\ 0 & 1 \end{bmatrix},
    \begin{bmatrix} 1 & 1 \\ 1 & 2 \end{bmatrix},
    \begin{bmatrix} 2 & -1 \\ -1 & 1 \end{bmatrix}.    \\
    &W_{1,2,3,4} := \frac{1}{\beta}\widehat W_{1,2,3,4}\quad \text{where}\\
    &\widehat W_{1,2,3,4}:= \begin{bmatrix} 2&0&0\\0&1&0\\0&0&1\end{bmatrix}, 
    \begin{bmatrix} 1&0&0\\0&2&0\\0&0&1\end{bmatrix}, 
    \begin{bmatrix} 1&0&0\\0&1&0\\0&0&2\end{bmatrix}, 
    \begin{bmatrix} 2&0&0\\0&2&0\\0&0&2\end{bmatrix}.    
\end{align}
Results are reported in Table \ref{tab:experiment-sym}.

\begin{table}[!t]
\begin{center}
\caption{Evaluation of Manifold Continuous Normalizing Flow (MCNF) on $\text{Sym}^+(n)$. The experiments on the Highlighted rows were repeated 5 times and reported in the summary Table \ref{tab:experiment-overview} }
\begin{tabular}{|c|c|l|l|}
    \hline
    {\bf Manifold} & {$\beta$} &{\bf KL[nats]}& {\bf ESS[\%]} \\
    \hline \hline
    $\text{Sym}^+(3)$ &5&  $0.006$ & $98.9$\\ \hline
    $\text{Sym}^+(3)$ &10&  $0.007$ & $98.4$\\ \hline
    $\text{Sym}^+(3)$ &20& $0.02$ & $94.9$\\ \hline
    $\text{Sym}^+(2)$ &5&  $0.001$ & $99.7$\\ \hline
    $\text{Sym}^+(2)$ &10& $0.004$ & $99.2$\\ \hline
    \rowcolor{light-gray} $\text{Sym}^+(2)$ &20&$0.009_{\pm 0.002}$ & $98.01_{\pm 0.1}$\\ \hline
    \end{tabular}
    \label{tab:experiment-sym}
\end{center}
\end{table}

\section{Summary of results and comment}

\begin{table}[!t]
\begin{center}
\caption{Summary of manifold montinuous normalizing flow (MCNF) on mixture density matching task. Error is computed over 5 replicas of each experiment.}
    
\begin{tabular}{|c|c|c|c|}
    \hline
    {\bf Manifold} & {\bf dimension} &{\bf KL[nats]}& {\bf ESS[\%]} \\
    \hline \hline
    $\mathcal{V}_{1}(\bbR^3) \cong \bbS^2$ &2&$0.003_{\pm 0.001}$ & $99.3_{\pm 0.8}$ \\ \hline
    $\mathcal{V}_{1}(\bbR^4) \cong \bbS^3$ &3&$0.006_{\pm 0.002}$ & $98.6_{\pm 0.5}$\\ \hline
    $\mathcal{V}_{2}(\bbR^4)$ &5& $0.02_{\pm 0.006}$ & $96.5_{\pm 1.1}$ \\ \hline
    $\mathcal{V}_{2}(\bbR^5)$ &7&$0.02_{\pm 0.003}$ & $97.1_{\pm 0.6}$\\ \hline
    $SO(3)$ &3&$0.03_{\pm 0.007}$ & $95.3_{\pm 0.8}$\\ \hline
    $\bbS^{10}$ &10&$0.01_{\pm 0.003}$ & $97.2_{\pm 0.6}$ \\ \hline
    $\bbS^{20}$ &20&$0.02_{\pm 0.003}$ & $95.4_{\pm 0.4}$\\ \hline
    $SU(2)$ &3& $0.005_{\pm 0.003}$ & $98.6_{\pm 0.8}$  \\ \hline
    $SU(3)$ &8& $0.01_{\pm 0.003}$ & $97.1_{\pm 0.6}$ \\ \hline
    $U(2)$ &4& $0.02_{\pm 0.008}$ & $95.4_{\pm 1.4}$ \\ \hline
    $U(3)$ &9& $0.05_{\pm 0.008}$ & $91.5_{\pm 1.3}$ \\ \hline
    $\text{Sym}^+(2)$ &3&$0.009_{\pm 0.002}$ & $98.01_{\pm 0.1}$\\ \hline
    \end{tabular}
    \label{tab:experiment-overview}
\end{center}
\end{table}
Table \ref{tab:experiment-overview} shows that the proposed MCNF is able to closely match the target densities on all the spaces, with low KL and high ESS ($> 90\%$). The results on $\bbS^{10}$ and $\bbS^{20}$ show that the proposed method can effectively scale to higher dimensional manifolds. Interestingly, the experiments seem to indicate that the model has more difficulty in matching the target distribution when the space is not simply connected ($SO(3),\ U(2),\ U(3)$). 
\chapter{Conclusion}
In this thesis we presented a universal methodology for building free form normalizing flows on manifolds. To achieve this objective we generalized neural ordinary differential equations based architectures and continuous normalizing flows to arbitrary smooth manifolds. As manifolds are nontrivial mathematical objects to work with, we strove to build our theory on strong mathematical foundations, recognizing in differential geometry the correct language to define our framework. 

Working with intrinsically defined objects on manifolds has several advantages over employing quantities that depend on a particular parameterization or defined on local coordinates. First of all, it decouples the mathematical derivation from a specific algorithmic choice, making it easier to use the most convenient parameterization for each specific space. For example, our definition of cotangent lift allows to freely choose the numerical integration method to use in practice. Secondly, our formulation paves the way for future research, such as the design of theoretically sound regularization methods, or the investigation of stability properties and approximation capabilities of neural ODEs on manifolds. 

Notwithstanding the above, we strove to present a framework that can be easily implemented and involves efficient computations. We achieve this by employing a generating set of vector fields in combination with free form neural architectures to parameterize vector fields. We then showed that this formulation leads to a simplified divergence calculation for homogeneous spaces and Lie groups, spaces of outstanding importance in many fields of science and engineering. These claims are supported by our experiments, which prove how the defined framework can be used to successfully train normalizing flows on a wide array of spaces. 
An aspect that we didn't consider in this thesis, and that we leave for future research, is using specific numerical integrators adapted to the manifold structure.

We hope that the methods and algorithms developed in this thesis will allow probabilistic deep learning models to tackle problems in many scientific disciplines outside machine learning.


\appendix
\chapter{Densities and induced measure}\label{app:background-density}

\section{volume forms}
In the rest of the chapter $M$ will be a smooth manifold of dimension $n$
\begin{defn}
A {\bf volume  form} $\omega$ is a section of the line bundle $\Lambda^n T^*M = \Lambda^{top}T^*M$, that is $\omega\in \csec{M}{\Lambda^n T^*M }$. 
\end{defn}
In any smooth chart $\left(U;x_i\right)$ a volume form $\omega$ can be written as:
\begin{equation}
    \omega\big|_U = fdx_1\wedge\cdots\wedge dx_n\quad \quad f\in C(U)
\end{equation}
If the manifold is orientable, then there exists a smooth non vanishing volume form that is positively oriented at each point \footnote{Proposition 15.5 \cite{lee2013smooth}}. Since $\Lambda^n T^*M$ is a vector bundle of rank 1, orientability implies that the line bundle is isomorphic to the trivial bundle $M\times\bbR$, such that any non vanishing (smooth) volume form is a (smooth) global frame. 

This means that fixed a smooth nonvanishing volume form $\omega\in \smoothsec{M}{\Lambda^n T^* M}$ we have a $1$ to $1$ correspondence between continuous (resp. smooth) functions on $M$ and (resp. smooth) volume forms on $M$:
\begin{align}
    C(M)&\to \csec{M}{\Lambda^n T^*M } \qquad \text{and} \quad &C^\infty(M)&\to \smoothsec{M}{\Lambda^n T^*M }\\
    f&\mapsto f\omega & f&\mapsto f\omega
\end{align}
The map depends explicitly on the initial choice for $\omega$. On semi-Riemannian manifolds there is a standard choice for this form:
\begin{prop}\footnote{See for example Proposition 15.29 \cite{lee2013smooth}. Many of the theorems here expressed for semi-Riemannian manifolds, are in the original source expressed and proved only Riemannian manifolds, however, they can effortlessly generalized to semi-Riemannian manifolds. A treatment that includes also semi Riemannian metrics can be found in \cite{o1983semi} }

Suppose $(M, g)$ is an oriented semi Riemannian n-manifold, and $n\ge 1$. There is a unique smooth non vanishing form $\omega_g\in \smoothsec{N}{\Lambda^n T^* M}$ , called the {\bf semi Riemannian volume form}, that satisfies:
\begin{equation}
    \omega_g(E_1,\cdots,E_n) = 1
\end{equation}
for every local oriented orthonormal frame\footnote{The existence of a smooth orthonormal frame around every point in a semi Riemannian manifold is ensured by Corollary 13.8 in \cite{lee2013smooth}.} $\{E_i\}$ on $M$.
\end{prop}
 
In any local coordinates $\left(U;x_i\right)$ the semi Riemannian volume form $\omega_g$ can be written as:
\begin{equation}\label{eq:semi-riem-vol}
    \omega_g\big|_U =  \sqrt{|\det g_{ij}|}\ dx_1\wedge\cdots\wedge dx_n
\end{equation}
Where for each point $q\in U$, $g_{ij}(q)$ is the symmetric matrix  that gives the local coordinate expression for the metric $g$.  

Given a volume form on a smooth manifold we can define its integral. For the rest of the section, we will outline the principal steps of the construction. The interested reader can consult Chapter 16 of \cite{lee2013smooth} for a complete exposition.

We first define the integral of a volume form $\omega$ compactly supported on the domain of a smooth chart $(U,\varphi)$ with coordinates $(x_i)$:
\begin{align}
    \int_M \omega = \int_U f dx_1\wedge\cdots\wedge dx_n :&= \pm \int_{\varphi(U)} f\circ\varphi^{-1} dx_1\wedge\cdots\wedge dx_n =\footnote{This is the Riemann integral for a continuous function with compact support} \\
    &=\int_{\varphi(U)} f(\varphi^{-1}(x)) \dd x 
\end{align}
Where $\omega|_U = f dx_1\wedge\cdots\wedge dx_n$ and the sign depends on the orientation. Notice that $f\circ\varphi^{-1} dx_1\wedge\cdots\wedge dx_n = (\varphi^{-1})^*\omega \in \volform{n}{\varphi(U)}$ is a volume form on a compactly supported open set of $\bbR^n$, and we define its integral as simply "erasing" the wedges and computing the corresponding Riemann integral. \footnote{With a slight abuse of notation we use the $dx_i$s both as a basis for $T^*U\subset T^*M$ and $T^*\varphi(U)\subset T^*\bbR^n$. In the first case $dx_i$ can be interpreted as the differential of the coordinate function $x_i:U\to \bbR$}. Proposition 16.4 of \cite{lee2013smooth} ensures that the integral is well defined, i.e. that it does not depend on the choice of the smooth chart whose domain contains the support of $\omega$.

Once we have defined the integral for a volume form compactly supported in one chart domain, we can define the general integral of a volume form $\omega$.This is done using the partition of unity to express the initial volume form as a sum of volume forms which are compactly supported in one chart domain. To achieve this, we first cover the support of $\omega$ using a finite collection $\{U_i\}$ of open smooth chart domains \footnote{This is possible because $\omega$ has compact support}. We then take a partition of unity ${\alpha_i}$ subordinate to the open collection. We then  define the integral of $\omega$ over $M$ as:
\begin{equation}
    \int_M \omega := \sum_i \int_M \alpha_i \omega
\end{equation}
It can be shown that this definition does not depend on the particular choice of the open cover and partition of unity \footnote{Prop. 16.5 \cite{lee2013smooth}.}.

\section{Densities}
\begin{defn}
Let $V$ be a $n$-dimensional vector space. A {\bf density on a vector space} V is a function:
\begin{equation}
\mu: \underbrace{V\times\cdots\times V}_{\text{n times}}\to \bbR
\end{equation}
Such that for every linear map $T:V\to V$:
\begin{equation}
    \mu(Tv_1,\cdots,Tv_n)= |\det T|\mu(v_1,\cdots,v_n)\quad \forall v_1,\cdots,v_n \in V
\end{equation}
A density is said positive if $\mu(v_1,\cdots,v_n)>0$, for every linear independent tuple $(v_i)_{i\in[n]}$. A negative density is similarly defined. 
\end{defn}
The set of all densities on $V$, $\mathcal{D}(V)$ is a 1 dimensional \footnote{Prop. 16.35 \cite{lee2013smooth}.} vector space.  Given a n-form $\omega \in  \Lambda^n(V^*)$, we can define a density $|\omega|:V\times\cdots\times V \to\bbR$ as 
\begin{equation}
|\omega|(v_1,\cdots,v_n) = |\omega(v_1,\cdots,v_n)|
\end{equation}
On a smooth manifold $M$ we can take the collection of all densities defined on each tangent space $T_qM, \ q\in M$, and give it a vector bundle structure. More formally, we define the {\bf density bundle} of $M$ as the set:
\begin{equation}
\mathcal{D}(M):=\bigsqcup_{q\in M} \mathcal{D}(T_qM)
\end{equation}
Moreover if we define $\pi: \mathcal{D}(M) \to M$ to be the natural projection that sends every element of $\mathcal{D}(T_qM)$ to $q$ we have:
\begin{prop}\footnote{Prop. 16.36 \cite{lee2013smooth}}
$( \mathcal{D}(M),\pi, M)$ with $ \mathcal{D}(M)$ and $\pi$ defined as above, is a smooth line bundle over $M$. 
\end{prop}
A section of $\mathcal{D}M$ is called a density on $M$. Densities are closely connected with volume forms. In fact, a non-vanishing volume form $\omega\in \volform{n}{M}$ determines a positive density\footnote{A positive density is a section that maps every point $p$ in the manifold to a positive density on $T_pM$. A positive density is in particular nonvanishing} 
$|\omega|$, defined as $|\omega|_q := |\omega_q|,\ \forall q \in M$. Moreover in local coordinates $(U;x_i)$ a density $\mu$ can be expressed as:
\begin{equation}
    \mu\big|_U = f |dx_1\wedge\cdots\wedge dx_n|\quad f\in C(U)
\end{equation}
Like differential forms, densities can be pulled back:
\begin{defn}
Let $F:M\to N$ be a smooth map between manifolds of dimension $n$, and $\mu$ a desnity over $N$. We define a density $F^*\mu$ on $M$, called the {\bf pullback density}, as:
\begin{equation}
(F^*\mu)_p(v_1,\cdots,v_n)=\mu_{F(p)}(\dd F_p{v_1},\cdots,\dd F_p{v_n})    
\end{equation}
\end{defn}
\begin{prop}\footnote{Prop. 16.38 \cite{lee2013smooth}}
Let $G:P\to M$ and $F:M\to N$ be smooth maps between $n$ dimensional manifolds, and let $\mu$ be a density on $N$
\begin{enumerate}
    \item[a] $\forall f\in C^\infty(N), F^*(f\mu)=(f\circ F)F^*\mu$
    \item[b] If $\omega$ is an $n$-form on $N$, then $F^*|\omega| = |F^*\omega|$
    \item[c] If $\mu$ is smooth, then $F^*\mu$ is a smooth density on $M$
    \item[d] $(F\circ G)^*\mu=G^*(F^*\mu)$
    \end{enumerate}

\end{prop}
Differently from volume forms, a smooth  positive density exists on every smooth manifold: 
\begin{prop}\footnote {Prop 16.37 \cite{lee2013smooth}}
If $M$ is a smooth manifold, there exists a smooth positive density on $M$
\end{prop}
Since the density bundle is a line bundle, the previous proposition implies that the density bundle is always parallelizable and that any positive density forms a global frame. 
This means that fixed a smooth positive density $\mu\in \smoothsec{M}{\cD M }$ we have a $1$ to $1$ correspondence between continuous (resp. smooth) functions on $M$ and (resp. smooth) densities on $M$:
\begin{align}\label{eq:function-density-corr}
    C(M)&\to \csec{M}{\cD M } \qquad \text{and} \quad &C^\infty(M)&\to \smoothsec{M}{\cD M }\\
    f&\mapsto f\mu & f&\mapsto f\mu
\end{align}
As with volume forms, these maps explicitly depends on the initial choice for $\mu$. Furthermore, in semi-Riemannian manifolds there is a standard choice for this density that generalizes the semi Riemannian volume form:
\begin{prop}\footnote{Proposition 16.45 \cite{lee2013smooth}} 
Let $(M, g)$ be a semi Riemannian manifold
There is a unique smooth positive density 
$\mu_g$ on $M$, called the semi Riemannian density, with the property that
$$\mu_g(E_1,\cdots,E_n)=1$$
for any local orthonormal frame $(E_i)$.
\end{prop}
On an oriented semi Riemannian manifold we have $\mu_g = |\omega_g|$. 

We can define the integral of compactly supported densities on a manifold. The construction closely mirrors  the definition of the integral of volume forms, first defining the integral for densities compactly supported on a chart domain and then generalizing to arbitrary compact support using a partition of unity. For a detailed description of the process, we refer to \cite{lee2013smooth}. The following proposition illustrates the properties of the integral:
\begin{prop}\label{prop:integral-prop}\footnote{Prop. 16.42 \cite{lee2013smooth}}
Suppose M is a n-manifolds and $\mu$,$\eta$ are compactly supported densities on M:
\begin{enumerate}
    \item[(a)] {\sc linearity} If $a,b\in\bbR$, then:
    \begin{equation}
        \int_M a\mu+b\eta = a\int_M \mu + b\int_M \eta.
    \end{equation}
    \item[(b)] {\sc positivity} If $\mu$ is a positive density, then:
    \begin{equation}
        \int_M \mu > 0
    \end{equation}
\end{enumerate}
\end{prop}
This properties fundamentally tell us that this integral is a positive linear functional on the vector space of densities with compact support on a smooth manifold. 
\subsection{Induced measure}
Using the correspondence given by Equation \eqref{eq:function-density-corr} every nonnegative density defines a positive linear functional on $C_c(M)$. In fact, fixed $\mu\in \csec{M}{\cD M }$ nonnegative density we define:
\begin{align}
    L_\mu:C_c(M)&\to \bbR\\
    f&\mapsto \int_Mf\mu    
\end{align}
We can then use the Riesz representation theorem to (uniquely) define a Radon measure on $M$:
\begin{thm}\label{thm:riesz}\footnote{Theorem 7.2 \cite{folland1999real})}
Let $\mathcal{X}$ be a locally compact Hausdorff space, and let $L$ be a
positive linear functional on $C_c(\mathcal{X})$. Then there exists a unique Radon \footnote{A radon measure is a measure defined on Borel sets, that is finite on all compact sets, outer regular on all Borel sets, and inner regular on open sets.} measure $\mu$ such that:
\begin{equation}
     L f=\int_\mathcal{X} f d \mu, \quad \forall f\in C_c(M)
\end{equation}
The Radon measure is said to represent the functional $L$. 
\end{thm}
One direct consequence of the Riesz representation theorem is that if we have two Radon measures $\mu_1$ and $\mu_2$ on a locally compact Hausdorff space such that their Lebesgue integral coincides on continuous and compactly supported functions:
\begin{equation}
    \int_\mathcal{X} g d\mu_1 = \int_\mathcal{X} g d\mu_2 \quad \forall g\in C_c(\mathcal{X})
\end{equation}
then the two measures coincide ($\mu_1 = \mu_2$).

We can use a positive smooth density and the Riesz representation theorem to define a Radon measure on a smooth manifold. 
\begin{thm}\label{}
Let $M$ be a smooth manifold and $\mu$ a non negative density on $M$, then there exists a unique Radon regular measure $\tilde\mu$ such that:
\begin{equation}
    \int_M f \mu = \int_M f d\tilde\mu \quad \forall f\in C_c(M)
\end{equation}
We will say that the measure $\tilde \mu$ is derived from $\mu$.
\end{thm}

\begin{proof}
We first show that we can apply the Riesz representation theorem for the functional $L_\mu$. Since any manifold is a locally compact Hausdorff space, this reduces to proving that $L_\mu$ is a positive linear functional. 
Using Proposition \ref{prop:integral-prop} we see that:
\begin{align*}
    L_\mu(\alpha f+\beta g) = \int_M \alpha f\mu+\beta g\mu =  \alpha  \int_M f\mu+\beta \int_M g\mu = \alpha L_\mu f + \beta L_\mu g\\ \forall f,g\in C_c(M),\ \forall\alpha,\beta \in \bbR
\end{align*}
and 
\begin{align*}
L_\mu f = \int_M f\mu \ge 0 \quad \forall f\in C_c(M),\ f\ge 0    
\end{align*}
This shows that $L_\mu$ is a linear positive functional on $C_c(M)$, therefore a Radon measure $\tilde \mu$ that represents $L_\mu$ exists. 
For Corollary 7.6 in \cite{folland1999real} any Radon measure in a $\sigma$-compact space is regular. Therefore, since any Manifold is $\sigma$-compact, $\tilde\mu$ is also regular.
\end{proof}
This choice of measure on $M$ depends on the initial choice for $\mu$. The next proposition shows how two  measures derived from two different densities relate to each other:
\begin{prop}\label{prop:geometric-rn}
Let $\tilde\eta$ and $\tilde\mu$ two regular Radon measures derived from two densities $\eta$, $\mu$, with $\mu$ positive. Then $\tilde\nu\ll\tilde\mu$, that is there exists a continuous function $f$ such that $f\tilde\mu = \tilde g$. If both $\eta$ and $\mu$ are smooth then $f$ is a smooth function. 
\end{prop}
\begin{proof}
Since every positive density forms a frame for the line bundle $\mathcal{D}M$, we can use Equation \eqref{eq:function-density-corr} to state that there exists $f\in C(M)$ such that $\eta = f\mu$. If $\mu$ and $\eta$ are both smooth then also $f$ is smooth. 

Consider now the measure $f\tilde\mu$, we see that $\tilde\eta$ and $f\tilde\mu$ represent the same positive linear functional on $C_c(M)$:
\begin{equation}
    \int_M gd\tilde\eta= \int_M g\eta = \int_M gf\mu = \int_M gd(f\tilde\mu)
\end{equation}
If $f\tilde \mu$ is a Radon measure then by Riesz representation theorem we can conclude that the two measures coincide.

To show that $f\tilde\mu$ is Radon we can use Theorem 7.8 from \cite{folland1999real}, that tells us that a in a locally compact space in which every open set is $\sigma$ compact, every Borel measure on $\mathcal{X}$ which is finite on compact sets is Radon. Since by definition every manifold satisfies the topological requirements that the space needs to satisfy, we are left to prove that $f\tilde \mu$ is finite on compact sets. Then, fixed $K\subseteq M$ compact, we have:
\begin{align}
    [f\tilde\mu](K) = \int_K f d\tilde\mu\le \sup_K(f) \int d\tilde\mu =  \sup_K(f)\tilde\mu(K)<+\infty
\end{align}
Where in the last equality we have used that, since $\tilde\mu$ is Radon, $\tilde\mu(K)<+\infty$, and that $\sup_K(f)<+\infty$ since $f$ is continuous and $K$ compact. Therefore $f\tilde \mu$ is finite on compact sets and this concludes our proof. 
\end{proof}
In practice, fixed a smooth positive density $\mu$, we will work with absolutely continuous probability measures $g\tilde\mu$, where $g\in L^1_{\tilde\mu}(M),\ g\ge  0,\ \int_M g d\tilde\mu = 1$. Then Proposition \eqref{prop:geometric-rn} tells us that that the set of measures that we can express in this way does not depend on the initial choice for $\mu$.


\printglossaries

\cleardoublepage
\phantomsection
\addcontentsline{toc}{chapter}{Bibliography}
\bibliographystyle{apalike}
\bibliography{thesis}

\begin{thebibliography}{}

\bibitem[Agrachev et~al., 2019]{AgrBarBos17}
Agrachev, A., Barilari, D., and Boscain, U. (2019).
\newblock A comprehensive introduction to sub-riemannian geometry.

\bibitem[Agrachev and Sachkov, 2013]{agrachev2013control}
Agrachev, A.~A. and Sachkov, Y. (2013).
\newblock {\em Control theory from the geometric viewpoint}, volume~87.
\newblock Springer Science \& Business Media.

\bibitem[Alekseevsky et~al., 1994]{alekseevsky1994poisson}
Alekseevsky, D., Grabowski, J., Marmo, G., and Michor, P.~W. (1994).
\newblock Poisson structures on the cotangent bundle of a lie group or a
  principle bundle and their reductions.
\newblock {\em Journal of Mathematical Physics}, 35(9):4909--4927.

\bibitem[Ayala et~al., 2009]{optimialityhomogspace}
Ayala, V., Rodríguez, J., and Sanmartin, L. (2009).
\newblock Optimality on homogeneous spaces, and the angle system associated
  with a bilinear control system.
\newblock {\em SIAM J. Control and Optimization}, 48:2636--2650.

\bibitem[Bates, 2014]{bates2014embedding}
Bates, J. (2014).
\newblock The embedding dimension of laplacian eigenfunction maps.
\newblock {\em Applied and Computational Harmonic Analysis}, 37(3):516--530.

\bibitem[Bose et~al., 2020]{bose2020latent}
Bose, A.~J., Smofsky, A., Liao, R., Panangaden, P., and Hamilton, W.~L. (2020).
\newblock Latent variable modelling with hyperbolic normalizing flows.
\newblock {\em arXiv preprint arXiv:2002.06336}.

\bibitem[Boyda et~al., 2020]{boyda2020sampling}
Boyda, D., Kanwar, G., Racanière, S., Rezende, D.~J., Albergo, M.~S., Cranmer,
  K., Hackett, D.~C., and Shanahan, P.~E. (2020).
\newblock Sampling using $su(n)$ gauge equivariant flows.

\bibitem[Bradbury et~al., 2018]{jax2018github}
Bradbury, J., Frostig, R., Hawkins, P., Johnson, M.~J., Leary, C., Maclaurin,
  D., and Wanderman-Milne, S. (2018).
\newblock {JAX}: composable transformations of {P}ython+{N}um{P}y programs.

\bibitem[Cama{\~n}o-Garcia, 2006]{CamaoGarcia2006Statistics}
Cama{\~n}o-Garcia, G. (2006).
\newblock {\em Statistics on Stiefel manifolds}.
\newblock PhD thesis, Iowa State University.

\bibitem[Cao and Aziz, 2020]{Cao2020ThePS}
Cao, N.~D. and Aziz, W. (2020).
\newblock The power spherical distribution.
\newblock {\em ICML workshop on Invertible Neural Networks, Normalizing Flows,
  and Explicit Likelihood Models}.

\bibitem[Chen and Verstraelen, 2013]{chen2013laplace}
Chen, B.-Y. and Verstraelen, L. (2013).
\newblock Laplace transformations of submanifolds.

\bibitem[Chen et~al., 2018]{neuralode}
Chen, R. T.~Q., Rubanova, Y., Bettencourt, J., and Duvenaud, D.~K. (2018).
\newblock Neural ordinary differential equations.
\newblock In Bengio, S., Wallach, H., Larochelle, H., Grauman, K.,
  Cesa-Bianchi, N., and Garnett, R., editors, {\em Advances in Neural
  Information Processing Systems 31}, pages 6571--6583. Curran Associates, Inc.

\bibitem[Da~Silva, 2001]{da2001lectures}
Da~Silva, A.~C. (2001).
\newblock {\em Lectures on symplectic geometry}, volume 3575.
\newblock Springer.

\bibitem[Davidson et~al., 2018]{davidson2018hyperspherical}
Davidson, T.~R., Falorsi, L., De~Cao, N., Kipf, T., and Tomczak, J.~M. (2018).
\newblock Hyperspherical variational auto-encoders.
\newblock {\em 34th Conference on Uncertainty in Artificial Intelligence
  (UAI-18)}.

\bibitem[Davidson et~al., 2019]{davidson2019increasing}
Davidson, T.~R., Tomczak, J.~M., and Gavves, E. (2019).
\newblock Increasing expressivity of a hyperspherical vae.
\newblock {\em NeurIPS, in Workshop on Bayesian Deep Learning}.

\bibitem[Dinh et~al., 2015]{dinh2015nice}
Dinh, L., Krueger, D., and Bengio, Y. (2015).
\newblock Nice: Non-linear independent components estimation.
\newblock {\em ICLR Workshop}.

\bibitem[Falorsi et~al., 2018]{falorsi2018explorations}
Falorsi, L., de~Haan, P., Davidson, T.~R., De~Cao, N., Weiler, M., Forr{\'e},
  P., and Cohen, T.~S. (2018).
\newblock Explorations in homeomorphic variational auto-encoding.
\newblock {\em ICML workshop on Theoretical Foundations and Applications of
  Deep Generative Models}.

\bibitem[Falorsi et~al., 2019]{falorsi2019reparameterizing}
Falorsi, L., de~Haan, P., Davidson, T.~R., and Forr{\'e}, P. (2019).
\newblock Reparameterizing distributions on lie groups.
\newblock {\em AISTATS}.

\bibitem[Figurnov et~al., 2018]{figurnov2018implicit}
Figurnov, M., Mohamed, S., and Mnih, A. (2018).
\newblock Implicit reparameterization gradients.
\newblock In {\em Advances in Neural Information Processing Systems}, pages
  441--452.

\bibitem[Folland, 1999]{folland1999real}
Folland, G. (1999).
\newblock {\em Real analysis: modern techniques and their applications}.
\newblock Pure and applied mathematics. Wiley.

\bibitem[Gemici et~al., 2016]{gemici2016normalizing}
Gemici, M.~C., Rezende, D., and Mohamed, S. (2016).
\newblock Normalizing flows on riemannian manifolds.
\newblock {\em arXiv preprint arXiv:1611.02304}.

\bibitem[Grathwohl et~al., 2019]{grathwohl2018scalable}
Grathwohl, W., Chen, R. T.~Q., Bettencourt, J., and Duvenaud, D. (2019).
\newblock Scalable reversible generative models with free-form continuous
  dynamics.
\newblock In {\em International Conference on Learning Representations}.

\bibitem[Haber and Ruthotto, 2017]{haber2017stable}
Haber, E. and Ruthotto, L. (2017).
\newblock Stable architectures for deep neural networks.
\newblock {\em Inverse Problems}, 34(1):014004.

\bibitem[Hairer et~al., 2006]{hairer2004GeoNumInt}
Hairer, E., Lubich, C., and Wanner, G. (2006).
\newblock {\em Geometric numerical integration. Structure-preserving algorithms
  for ordinary differential equations. 2nd ed}, volume~31.
\newblock Springer Science \& Business Media.

\bibitem[Howard, 1994]{Howard94analysison}
Howard, R. (1994).
\newblock Analysis on homogeneous spaces.

\bibitem[Hutchinson, 1989]{hutchinson1989}
Hutchinson, M. (1989).
\newblock A stochastic estimator of the trace of the influence matrix for
  laplacian smoothing splines.
\newblock {\em Communication in Statistics- Simulation and Computation},
  18:1059--1076.

\bibitem[Kanwar et~al., 2020]{kanwar2020equivariant}
Kanwar, G., Albergo, M.~S., Boyda, D., Cranmer, K., Hackett, D.~C.,
  Racani{\`e}re, S., Rezende, D.~J., and Shanahan, P.~E. (2020).
\newblock Equivariant flow-based sampling for lattice gauge theory.
\newblock {\em arXiv preprint arXiv:2003.06413}.

\bibitem[Kingma and Ba, 2015]{kingma2014adam}
Kingma, D.~P. and Ba, J. (2015).
\newblock Adam: A method for stochastic optimization.
\newblock {\em International Conference on Learning Representations}.

\bibitem[Kingma and Welling, 2014]{kingma2013auto}
Kingma, D.~P. and Welling, M. (2014).
\newblock Auto-encoding variational bayes.
\newblock {\em Proceedings of the 2nd International Conference on Learning
  Representations (ICLR)}.

\bibitem[Kish, 1965]{kishess}
Kish, L. (1965).
\newblock Survey sampling.
\newblock {\em American Political Science Review}, 59(4):1025–1025.

\bibitem[Kobyzev et~al., 2020]{kobyzev2020normalizing}
Kobyzev, I., Prince, S., and Brubaker, M. (2020).
\newblock Normalizing flows: An introduction and review of current methods.
\newblock {\em IEEE Transactions on Pattern Analysis and Machine Intelligence}.

\bibitem[Kunzinger et~al., 2006]{kunzinger2006global}
Kunzinger, M., Schichl, H., Steinbauer, R., and Vickers, J.~A. (2006).
\newblock Global gronwall estimates for integral curves on riemannian
  manifolds.
\newblock {\em Revista Matem{\'a}tica Complutense}, 19(1):133--137.

\bibitem[Lee, 2006]{lee2006riemannian}
Lee, J.~M. (2006).
\newblock {\em Riemannian manifolds: an introduction to curvature}, volume 176.
\newblock Springer Science \& Business Media.

\bibitem[Lee, 2013]{lee2013smooth}
Lee, J.~M. (2013).
\newblock Smooth manifolds.
\newblock In {\em Introduction to Smooth Manifolds}, pages 1--31. Springer.

\bibitem[Lou et~al., 2020]{lou2020neural}
Lou, A., Lim, D., Katsman, I., Huang, L., Jiang, Q., Lim, S.-N., and Sa, C.~D.
  (2020).
\newblock Neural manifold ordinary differential equations.

\bibitem[Mathieu et~al., 2019]{mathieu2019continuous}
Mathieu, E., Le~Lan, C., Maddison, C.~J., Tomioka, R., and Teh, Y.~W. (2019).
\newblock Continuous hierarchical representations with poincar{\'e} variational
  auto-encoders.
\newblock In {\em Advances in neural information processing systems}, pages
  12544--12555.

\bibitem[Mathieu and Nickel, 2020]{mathieu2020riemannian}
Mathieu, E. and Nickel, M. (2020).
\newblock Riemannian continuous normalizing flows.

\bibitem[Mezzadri, 2006]{mezzadri2006randomcompact}
Mezzadri, F. (2006).
\newblock How to generate random matrices from the classical compact groups.
\newblock {\em Notices of the American Mathematical Society}, 54.

\bibitem[Moakher and Zerai, 2011]{posdefmatrices}
Moakher, M. and Zerai, M. (2011).
\newblock The riemannian geometry of the space of positive-definite matrices
  and its application to the regularization of positive-definite matrix-valued
  data.
\newblock {\em Journal of Mathematical Imaging and Vision}, 40:171--187.

\bibitem[Mohamed et~al., 2020]{mohamed2020monte}
Mohamed, S., Rosca, M., Figurnov, M., and Mnih, A. (2020).
\newblock Monte carlo gradient estimation in machine learning.
\newblock {\em Journal of Machine Learning Research}, 21(132):1--62.

\bibitem[Naesseth et~al., 2017]{naesseth2017reparameterization}
Naesseth, C., Ruiz, F., Linderman, S., and Blei, D. (2017).
\newblock Reparameterization gradients through acceptance-rejection sampling
  algorithms.
\newblock In {\em Artificial Intelligence and Statistics}, pages 489--498.

\bibitem[Nagano et~al., 2019]{wrappedhyperbolic}
Nagano, Y., Yamaguchi, S., Fujita, Y., and Koyama, M. (2019).
\newblock A wrapped normal distribution on hyperbolic space for gradient-based
  learning.
\newblock In {\em International Conference on Machine Learning}, pages
  4693--4702.

\bibitem[No{\'e} et~al., 2019]{noe2019boltzmann}
No{\'e}, F., Olsson, S., K{\"o}hler, J., and Wu, H. (2019).
\newblock Boltzmann generators: Sampling equilibrium states of many-body
  systems with deep learning.
\newblock {\em Science}, 365(6457):eaaw1147.

\bibitem[O'Neill, 1983]{o1983semi}
O'Neill, B. (1983).
\newblock {\em Semi-Riemannian Geometry With Applications to Relativity}.
\newblock Pure and Applied Mathematics. Elsevier Science.

\bibitem[Papamakarios et~al., 2019]{papamakarios2019normalizing}
Papamakarios, G., Nalisnick, E., Rezende, D.~J., Mohamed, S., and
  Lakshminarayanan, B. (2019).
\newblock Normalizing flows for probabilistic modeling and inference.
\newblock {\em stat}, 1050:5.

\bibitem[Pontryagin et~al., 1962]{pontryagin1962mathematical}
Pontryagin, L.~S., Mishchenko, E., Boltyanskii, V., and Gamkrelidze, R. (1962).
\newblock {\em The mathematical theory of optimal processes}.
\newblock Wiley.

\bibitem[Pérez~Rey et~al., 2019]{diffusionvae}
Pérez~Rey, L.~A., Menkovski, V., and Portegies, J.~W. (2019).
\newblock Diffusion variational autoencoders.
\newblock {\em arXiv}.

\bibitem[Rezende and Mohamed, 2015]{rezende2015variational}
Rezende, D. and Mohamed, S. (2015).
\newblock Variational inference with normalizing flows.
\newblock In {\em International Conference on Machine Learning}, pages
  1530--1538.

\bibitem[Rezende et~al., 2020]{rezende2020normalizing}
Rezende, D.~J., Papamakarios, G., Racani{\`e}re, S., Albergo, M.~S., Kanwar,
  G., Shanahan, P.~E., and Cranmer, K. (2020).
\newblock Normalizing flows on tori and spheres.
\newblock {\em arXiv preprint arXiv:2002.02428}.

\bibitem[Rudin, 1987]{Rudin:1987:RCA:26851}
Rudin, W. (1987).
\newblock {\em Real and Complex Analysis, 3rd Ed.}
\newblock McGraw-Hill, Inc., New York, NY, USA.

\bibitem[Tabak and Turner, 2013]{tabak2013family}
Tabak, E.~G. and Turner, C.~V. (2013).
\newblock A family of nonparametric density estimation algorithms.
\newblock {\em Communications on Pure and Applied Mathematics}, 66(2):145--164.

\bibitem[Tabak et~al., 2010]{tabak2010density}
Tabak, E.~G., Vanden-Eijnden, E., et~al. (2010).
\newblock Density estimation by dual ascent of the log-likelihood.
\newblock {\em Communications in Mathematical Sciences}, 8(1):217--233.

\bibitem[Teschl, 1998]{teschl1998topics}
Teschl, G. (1998).
\newblock Topics in real and functional analysis.

\bibitem[Walschap, 2004]{metric-structures-diff}
Walschap, G. (2004).
\newblock {\em Metric Structures in Differential Geometry}.
\newblock Graduate Texts in Mathematics. Springer, softcover reprint of
  hardcover 1st ed. 2004 edition.

\bibitem[Wirnsberger et~al., 2020]{wirnsberger2020targeted}
Wirnsberger, P., Ballard, A.~J., Papamakarios, G., Abercrombie, S.,
  Racani{\`e}re, S., Pritzel, A., Rezende, D.~J., and Blundell, C. (2020).
\newblock Targeted free energy estimation via learned mappings.
\newblock {\em arXiv preprint arXiv:2002.04913}.

\end{thebibliography}

\end{document}